\documentclass{ericsson}

\usepackage{amsmath,amsfonts,bm}

\def\eqref#1{equation~\ref{#1}}
\def\Eqref#1{Equation~\ref{#1}}
\def\1{\bm{1}}

\def\eps{{\epsilon}}

\def\rvz{{\mathbf{z}}}

\def\vone{{\bm{1}}}

\def\vtheta{{\bm{\theta}}}
\def\valpha{{\bm{\alpha}}}
\def\vomega{{\bm{\omega}}}

\def\ve{{\bm{e}}}

\def\vk{{\bm{k}}}

\def\vr{{\bm{r}}}

\def\vv{{\bm{v}}}

\def\vy{{\bm{y}}}

\DeclareMathAlphabet{\mathsfit}{\encodingdefault}{\sfdefault}{m}{sl}
\SetMathAlphabet{\mathsfit}{bold}{\encodingdefault}{\sfdefault}{bx}{n}

\def\gA{{\mathcal{A}}}

\def\gI{{\mathcal{I}}}

\def\gL{{\mathcal{L}}}

\def\gS{{\mathcal{S}}}

\def\gV{{\mathcal{V}}}
\def\gW{{\mathcal{W}}}

\def\sN{{\mathbb{N}}}

\def\sR{{\mathbb{R}}}

\newcommand{\Ls}{\mathcal{L}}
\newcommand{\R}{\mathbb{R}}

\DeclareMathOperator*{\argmax}{arg\,max}

\usepackage[english]{babel}

\usepackage{graphicx}
\usepackage{svg}
\usepackage{float}
\usepackage{siunitx}
\usepackage{tabularray}
\usepackage{nicefrac}

\usepackage{algorithm}
\usepackage{algorithmicx}
\usepackage{algpseudocode}

\usepackage{natbib}

\newtheorem{observation}{Observation}
\newtheorem{example}{Example}
\newtheorem{remark}{Remark}

\UseTblrLibrary{booktabs}

\usepackage{listings}
\usepackage{xcolor}

\definecolor{jsonkey}{RGB}{0,0,0} % {255,191,0}      % orange/yellow
\definecolor{jsonstring}{RGB}{0,200,0}     % green
\definecolor{jsonbool}{RGB}{255,105,180}   % pink
\definecolor{jsonnumber}{RGB}{30,144,255}  % blue

\lstdefinelanguage{json}{
  basicstyle=\ttfamily\scriptsize, %\scriptsize, \tiny
  numbers=left, numberstyle=\tiny\color{gray}, stepnumber=1, numbersep=6pt,
  showstringspaces=false, breaklines=true,
  frame=single, rulecolor=\color{black},
  morestring=[b]",
  stringstyle=\color{jsonstring},
  moredelim=**[s][\bfseries\color{jsonkey}]{"}{":},
  literate=
   *{true}{{{\color{jsonbool}true}}}{4}
    {false}{{{\color{jsonbool}false}}}{5}
    {null}{{{\color{jsonbool}null}}}{4}
    {0}{{{\color{jsonnumber}0}}}{1}
    {1}{{{\color{jsonnumber}1}}}{1}
    {2}{{{\color{jsonnumber}2}}}{1}
    {3}{{{\color{jsonnumber}3}}}{1}
    {4}{{{\color{jsonnumber}4}}}{1}
    {5}{{{\color{jsonnumber}5}}}{1}
    {6}{{{\color{jsonnumber}6}}}{1}
    {7}{{{\color{jsonnumber}7}}}{1}
    {8}{{{\color{jsonnumber}8}}}{1}
    {9}{{{\color{jsonnumber}9}}}{1}
}
\usepackage[most]{tcolorbox}

\newtcolorbox{quoteBox}{
  colback=gray!10,
  colframe=gray!50,
  left=6pt,
  right=6pt,
  top=6pt,
  bottom=6pt,
  sharp corners,
  borderline west={2pt}{0pt}{gray!70}
}

\usepackage[nolist, printonlyused, nohyperlinks]{acronym}

\title{From Intents to Actions: Agentic AI in Autonomous Networks}
\author[1]{Burak Demirel}
\author[1]{Pablo Soldati}
\author[1]{Yu Wang}
\affiliation[1]{Ericsson AB, Kista, Sweden}
\date{\today}
\correspondence{\email{burak.demirel@ericsson.com}}

\begin{document}

% -------------------------
% Abstract
% -------------------------
\begin{abstract}
Telecommunication networks are increasingly expected to operate autonomously while supporting heterogeneous services with diverse and often conflicting \emph{intents}—that is, performance objectives, constraints, and requirements specific to each service. However, transforming high-level intents—such as ultra-low latency, high throughput, or energy efficiency—into concrete control actions (i.e., low-level actuator commands) remains beyond the capability of existing heuristic approaches. This work introduces an Agentic AI system for intent-driven autonomous networks, structured around three specialized agents. A supervisory \emph{interpreter agent}, powered by  language models, performs both lexical parsing of intents into executable optimization templates and cognitive refinement based on feedback, constraint feasibility, and evolving network conditions. An \emph{optimizer agent} converts these templates into tractable optimization problems, analyzes trade-offs, and derives preferences across objectives. Lastly, a preference-driven \emph{controller agent}, based on multi-objective reinforcement learning, leverages these preferences to operate near the Pareto frontier of network performance that best satisfies the original intent. Collectively, these agents enable networks to autonomously interpret, reason over, adapt to, and act upon diverse intents and network conditions in a scalable manner.
\end{abstract}

\maketitle

% -------------------------
% Acronyms — define only once
% -------------------------
%\chapter*{Acronyms}
%\renewcommand{\titulonome}{Acronyms}%
%\renewcommand{\prepbynome}{UFC.33 Team}%

\acrodefplural{MDP}[MDPs]{Markov decision processes}
\acrodefplural{RTG}[RTGs]{returns-to-go}

%\begin{singlespace}
\begin{acronym}%[LTE-Advanced]%\addtolength{\itemsep}{-0.5\baselineskip}
  \acro{AI}{artificial intelligence}
  \acro{ACK}{positive acknowledgment}
  \acro{ARQ}{automatic repeat request}
  \acro{CAN}{cognitive autonomous networks}
  \acro{CCS}{convex coverage set}
  \acro{BE}{best effort}
  \acro{BC}{behavioral cloning}
  \acro{BCQ}{batch-constrained deep Q-learning}
  \acro{BO}{Bayesian optimization}
  \acro{CCTR}{Channel-Conditioned Target Return}
  \acro{CEI}{constrained expected improvement}
  \acro{CNN}{convolutional neural network}
  \acro{CQL}{conservative Q-learning}
  \acro{CT}{continuous time}
  \acro{CV}{computer vision}
  \acro{DAVG}{discounted average}
  \acro{DDQN}{double DQN}
  \acro{D-EQL}{distributed envelope Q-learning}
  \acro{DLLA}{downlink link adaptation}
  \acro{DP}{dynamic programming}
  \acro{DQN}{deep Q-network}
  \acro{DSCP}{differentiated services code point}
  \acro{DST}{deep sea treasure}
  \acro{DT}{decision transformer}
  \acro{eMBB}{enhanced mobile broadband}
  \acro{EF}{expedited forwarding}
  \acro{EI}{expected improvement}
  \acro{EQL}{envelope Q-learning}
  \acro{ES}{entropy search}
  \acro{FB}{full buffer}
  \acro{FTN}{fruit tree navigation}
  \acro{IBN}{intent-based networking}
  \acro{ICL}{in-context learning}
  \acro{ILLA}{inner-loop link adaptation}
  \acro{IMF}{intent management function}
  \acro{GBR}{guaranteed bit rate}
  \acro{gNB}{next generation NodeB}
  \acro{gNB-CU}{gNB centralized unit}
  \acro{gNB-DU}{gNB distributed unit}
  \acro{GP}{Gaussian process}
  \acro{GPI}{generalized policy iteration}
  \acro{KPI}{key performance indicator}
  \acro{LA}{link adaptation}
  \acro{LLM}{large language model}
  \acro{LSTM}{long short-term memory}
  \acro{LT}{learnable time}
  \acro{MAB}{multi-armed bandits}
  \acro{MC}{Monte Carlo}
  \acro{MDP}{Markov decision process}
  \acro{ML}{machine learning}
  \acro{MLP}{multi-layer perceptron}
  \acro{mMIMO}{massive multiple input multiple output}
  \acro{mMTC}{massive machine-type communications}
  \acro{MNO}{mobile network operator}
  \acro{MOMDP}{multi-objective Markov decision process}
  \acro{MORL}{multi-objective reinforcement learning}
  \acro{MORRM}{multi-objective radio resource management}
  \acro{NACK}{negative acknowledgment}
  \acro{NADAM}{Nesterov-accelerated adaptive moment estimation}
  \acro{NLP}{natural language processing}
  \acro{non-GBR}{non-guaranteed bit rate}
  \acro{OLLA}{outer-loop link adaptation}
  \acro{O-RAN}{Open RAN}
  \acro{OTM}{optimization template model}
  \acro{PDF}{probability density function}
  \acro{PE}{positional encoding}
  \acro{PI}{probability of improvement}
  \acro{PEUMO}{pareto efficient uniform model}
  \acro{PHY}{physical layer}
  \acro{QoE}{quality of experience}
  \acro{QoS}{quality of service}
  \acro{5QI}{5G QoS identifier}
  \acro{PDB}{packet delay budget}
  \acro{RBF}{radial basis function}
  \acro{RE}{resource element}
  \acro{RL}{reinforcement learning}
  \acro{RLC}{radio link control}
  \acro{RNN}{recurrent neural network}
  \acro{RPC}{remote procedure call}
  \acro{RRM}{radio resource management}
  \acro{RTGs}{returns-to-go}
  \acro{RTG}{return-to-go}
  \acro{RTT}{round-trip time}
  \acro{RvS}{reinforcement learning via supervised learning}
  \acro{SACo}{state-action coverage}
  \acro{SADCo}{state-action density coverage}
  \acro{SF}{SON functions}
  \acro{SON}{self-organizing networks}
  \acro{SLM}{small language model}
  \acro{SE}{spectral efficiency}
  \acro{TBS}{transport block size}
  \acro{TD}{temporal difference}
  \acro{TQ}{relative trajectory quality}
  \acro{TR}{trust region}
  \acro{TS}{technical specification}
  \acro{UCB}{upper confidence bound}
  \acro{UDP}{user datagram protocol}
  \acro{URLLC}{ultra-reliable low-latency communications}
  \acro{VAE}{variational auto-encoder}
%%%%%%%%%%%%%
  \acro{2G}{Second Generation}
  \acro{3G}{3$^\text{rd}$~Generation}
  \acro{3GPP}{3$^\text{rd}$~Generation Partnership Project}
  \acro{4G}{4$^\text{th}$~Generation}
  \acro{5G}{5$^\text{th}$~Generation}
  \acro{AA}{Antenna Array}
  \acro{AC}{Admission Control}
  \acro{AD}{Attack-Decay}
  \acro{ADSL}{Asymmetric Digital Subscriber Line}
	\acro{AHW}{Alternate Hop-and-Wait}
  \acro{AMC}{Adaptive Modulation and Coding}
  \acro{AoA}{angle of arrival}
	\acro{AP}{Access Point}
  \acro{APA}{Adaptive Power Allocation}
  \acro{AR}{autoregressive}
  \acro{ARMA}{Autoregressive Moving Average}
  \acro{ATES}{Adaptive Throughput-based Efficiency-Satisfaction Trade-Off}
  \acro{AWGN}{additive white Gaussian noise}
  \acro{BB}{Branch and Bound}
  \acro{BD}{Block Diagonalization}
  \acro{BER}{bit error rate}
  \acro{BF}{Best Fit}
  \acro{BLER}{block error rate}
  \acro{BPC}{Binary power control}
  \acro{BPSK}{Binary Phase-Shift Keying}
  \acro{BPA}{Best \ac{PDPR} Algorithm}
  \acro{BRA}{Balanced Random Allocation}
  \acro{BCRB}{Bayesian Cram\'{e}r-Rao Bound}
  \acro{BS}{base station}
  \acro{CAP}{Combinatorial Allocation Problem}
  \acro{CAPEX}{Capital Expenditure}
  \acro{CBF}{Coordinated Beamforming}
  \acro{CBR}{Constant Bit Rate}
  \acro{CBS}{Class Based Scheduling}
  \acro{CC}{Congestion Control}
  \acro{CDF}{cumulative distribution function}
  \acro{CDMA}{Code-Division Multiple Access}
  \acro{CL}{Closed Loop}
  \acro{CLPC}{Closed Loop Power Control}
  \acro{CNR}{Channel-to-Noise Ratio}
  \acro{CPA}{Cellular Protection Algorithm}
  \acro{CPICH}{Common Pilot Channel}
  \acro{CoMP}{Coordinated Multi-Point}
  \acro{CQI}{channel quality indicator}
  \acro{CRB}{Cram\'{e}r-Rao Bound}
  \acro{CRM}{Constrained Rate Maximization}
	\acro{CRN}{Cognitive Radio Network}
  \acro{CS}{Coordinated Scheduling}
  \acro{CSI}{channel state information}
  \acro{CSIR}{channel state information at the receiver}
  \acro{CSIT}{channel state information at the transmitter}
  \acro{CUE}{cellular user equipment}
  \acro{D2D}{device-to-device}
  \acro{DCA}{Dynamic Channel Allocation}
  \acro{DE}{Differential Evolution}
  \acro{DFT}{Discrete Fourier Transform}
%  \acro{DIST}{Distance-based Grouping}
  \acro{DIST}{Distance}
  \acro{DL}{downlink}
  \acro{DMA}{Double Moving Average}
	\acro{DMRS}{demodulation reference signal}
  \acro{D2DM}{D2D Mode}
  \acro{DMS}{D2D Mode Selection}
  \acro{DPC}{Dirty Paper Coding}
  \acro{DRA}{Dynamic Resource Assignment}
  \acro{DSA}{Dynamic Spectrum Access}
  \acro{DSM}{Delay-based Satisfaction Maximization}
  \acro{ECC}{Electronic Communications Committee}
  \acro{EFLC}{Error Feedback Based Load Control}
  % \acro{EI}{Efficiency Indicator}
  \acro{eNB}{Evolved Node B}
  \acro{EPA}{Equal Power Allocation}
  \acro{EPC}{Evolved Packet Core}
  \acro{EPS}{Evolved Packet System}
  \acro{ESPRIT}{estimation of signal parameters via rotational invariance}
  \acro{E-UTRAN}{Evolved Universal Terrestrial Radio Access Network}
  %\acro{ES}{Exhaustive Search}
  \acro{FDD}{frequency division duplexing}
  \acro{FDM}{Frequency Division Multiplexing}
  \acro{FER}{Frame Erasure Rate}
  \acro{FF}{Fast Fading}
  \acro{FIM}{Fisher information matrix}
  \acro{FSB}{Fixed Switched Beamforming}
  \acro{FST}{Fixed SNR Target}
  \acro{FTP}{File Transfer Protocol}
  \acro{GA}{Genetic Algorithm}
  \acro{GLR}{Gain to Leakage Ratio}
  \acro{GOS}{Generated Orthogonal Sequence}
  \acro{GPL}{GNU General Public License}
  \acro{GRP}{Grouping}
  \acro{HARQ}{hybrid automatic repeat request}
  \acro{HMS}{Harmonic Mode Selection}
  \acro{HOL}{Head Of Line}
  \acro{HSDPA}{High-Speed Downlink Packet Access}
  \acro{HSPA}{High Speed Packet Access}
  \acro{HTTP}{HyperText Transfer Protocol}
  \acro{ICMP}{Internet Control Message Protocol}
  \acro{ICI}{Intercell Interference}
  \acro{ID}{Identification}
  \acro{ISAC}{integrated sensing and communication}
  \acro{IEEE}{Institute of Electrical and Electronics Engineers}
  \acro{IETF}{Internet Engineering Task Force}
  \acro{ILP}{Integer Linear Program}
  \acro{JRAPAP}{Joint RB Assignment and Power Allocation Problem}
  \acro{UID}{Unique Identification}
  \acro{HPC}{high-performance computing}
  \acro{IID}{Independent and Identically Distributed}
  \acro{IIR}{Infinite Impulse Response}
  \acro{ILP}{Integer Linear Problem}
  \acro{IMT}{International Mobile Telecommunications}
  \acro{INV}{Inverted Norm-based Grouping}
  \acro{IoT}{Internet of Things}
  \acro{IP}{Internet Protocol}
  \acro{IPv6}{Internet Protocol Version 6}
  \acro{ISD}{Inter-Site Distance}
  \acro{ISI}{Inter Symbol Interference}
  \acro{ITU}{International Telecommunication Union}
  \acro{JOAS}{Joint Opportunistic Assignment and Scheduling}
  \acro{JOS}{Joint Opportunistic Scheduling}
  \acro{JP}{Joint Processing}
  \acro{JS}{Jump-Stay}
  \acro{KKT}{Karush-Kuhn-Tucker}
  \acro{L3}{Layer-3}
  \acro{LAC}{Link Admission Control}
  \acro{LC}{Load Control}
  \acro{LOS}{Line of Sight}
  \acro{LP}{Linear Programming}
  \acro{LS}{least squares}
  \acro{LSF}{load scharing facility}
  \acro{LTE}{Long Term Evolution}
  \acro{LTE-A}{LTE-Advanced}
  \acro{LTE-Advanced}{Long Term Evolution Advanced}
  \acro{M2M}{Machine-to-Machine}
  \acro{MAC}{Medium Access Control}
  \acro{MANET}{Mobile Ad hoc Network}
  %\acro{MC}{Modular Clock}
  \acro{MCS}{modulation and coding scheme}
  \acro{MDB}{Measured Delay Based}
  \acro{MDI}{Minimum D2D Interference}
  \acro{MF}{Matched Filter}
  \acro{MG}{Maximum Gain}
  \acro{MH}{Multi-Hop}
  \acro{MIMO}{multiple input multiple output}
  \acro{MINLP}{Mixed Integer Nonlinear Programming}
  \acro{MIP}{Mixed Integer Programming}
  \acro{MISO}{Multiple Input Single Output}
  \acro{MLE}{maximum likelihood estimator}
  \acro{MLWDF}{Modified Largest Weighted Delay First}
  \acro{MME}{Mobility Management Entity}
  \acro{MMSE}{minimum mean squared error}
  \acro{MOS}{Mean Opinion Score}
  \acro{MPF}{Multicarrier Proportional Fair}
  \acro{MRA}{Maximum Rate Allocation}
  \acro{MR}{Maximum Rate}
  \acro{MRC}{Maximum Ratio Combining}
  \acro{MRT}{Maximum Ratio Transmission}
  \acro{MRUS}{Maximum Rate with User Satisfaction}
  \acro{MS}{mobile station}
  \acro{MSE}{mean squared error}
  \acro{MSI}{Multi-Stream Interference}
  \acro{MTC}{Machine-Type Communication}
  \acro{MTSI}{Multimedia Telephony Services over IMS}
  \acro{MTSM}{Modified Throughput-based Satisfaction Maximization}
  \acro{MU-MIMO}{multiuser multiple input multiple output}
  \acro{MU}{multi-user}
  \acro{MUSIC}{multiple signal classification}
  \acro{NAS}{Non-Access Stratum}
  \acro{NB}{Node B}
  \acro{NE}{Nash equilibrium}
  \acro{NCL}{Neighbor Cell List}
  \acro{NLOS}{Non-Line of Sight}
  \acro{NMSE}{Normalized Mean Square Error}
  \acro{NORM}{Normalized Projection-based Grouping}
  \acro{NP}{Non-Polynomial Time}
  \acro{NR}{New Radio}
  \acro{NRT}{Non-Real Time}
  \acro{NSPS}{National Security and Public Safety Services}
  \acro{O2I}{Outdoor to Indoor}
  \acro{OFDMA}{orthogonal frequency division multiple access}
  \acro{OFDM}{orthogonal frequency division multiplexing}
  \acro{OFPC}{Open Loop with Fractional Path Loss Compensation}
	\acro{O2I}{Outdoor-to-Indoor}
  \acro{OL}{Open Loop}
  \acro{OLPC}{Open-Loop Power Control}
  \acro{OL-PC}{Open-Loop Power Control}
  \acro{OPEX}{Operational Expenditure}
  \acro{ORB}{Orthogonal Random Beamforming}
  \acro{JO-PF}{Joint Opportunistic Proportional Fair}
  \acro{OSI}{Open Systems Interconnection}
  \acro{PAIR}{D2D Pair Gain-based Grouping}
  \acro{PAPR}{Peak-to-Average Power Ratio}
  \acro{P2P}{Peer-to-Peer}
  \acro{PC}{Power Control}
  \acro{PCI}{Physical Cell ID}
  \acro{PDPR}{pilot-to-data power ratio}
  \acro{PER}{packet error rate}
  \acro{PF}{Proportional Fair}
  \acro{P-GW}{Packet Data Network Gateway}
  \acro{PL}{Pathloss}
  \acro{PPR}{pilot power ratio}
  \acro{PRB}{physical resource block}
  \acro{PROJ}{Projection-based Grouping}
  \acro{ProSe}{Proximity Services}
  \acro{PS}{Packet Scheduling}
  \acro{PSAM}{pilot symbol assisted modulation}
  \acro{PSO}{Particle Swarm Optimization}
  \acro{PZF}{Projected Zero-Forcing}
  \acro{QAM}{Quadrature Amplitude Modulation}
  \acro{QPSK}{Quadri-Phase Shift Keying}
  \acro{RAISES}{Reallocation-based Assignment for Improved Spectral Efficiency and Satisfaction}
  \acro{RAN}{radio access network}
  \acro{RAT}{Radio Access Technology}
  \acro{RATE}{Rate-based}
  \acro{RB}{resource block}
  \acro{RBG}{Resource block broup}
  \acro{REF}{Reference Grouping}
  \acro{RM}{Rate Maximization}
  \acro{RNC}{Radio Network Controller}
  \acro{RND}{Random Grouping}
  \acro{RRA}{Radio Resource Allocation}
  \acro{RRM}{radio resource management}
  \acro{RSCP}{Received Signal Code Power}
  \acro{RSRP}{Reference Signal Receive Power}
  \acro{RSRQ}{Reference Signal Receive Quality}
  \acro{RR}{Round Robin}
  \acro{RRC}{Radio Resource Control}
  \acro{RSSI}{Received Signal Strength Indicator}
  \acro{RT}{Real Time}
  \acro{RU}{Resource Unit}
  \acro{RUNE}{RUdimentary Network Emulator}
  \acro{RV}{Random Variable}
  \acro{SAC}{Session Admission Control}
  \acro{SCM}{Spatial Channel Model}
  \acro{SC-FDMA}{Single Carrier - Frequency Division Multiple Access}
  \acro{SD}{Soft Dropping}
  \acro{S-D}{Source-Destination}
  \acro{SDPC}{Soft Dropping Power Control}
  \acro{SDMA}{Space-Division Multiple Access}
  \acro{SER}{Symbol Error Rate}
  \acro{SES}{Simple Exponential Smoothing}
  \acro{S-GW}{Serving Gateway}
  \acro{SINR}{signal-to-interference-plus-noise ratio}
  \acro{SI}{Satisfaction Indicator}
  \acro{SIP}{Session Initiation Protocol}
  \acro{SISO}{single input single output}
  \acro{SIMO}{Single Input Multiple Output}
  \acro{SIR}{signal-to-interference ratio}
  \acro{SLNR}{Signal-to-Leakage-plus-Noise Ratio}
  \acro{SMA}{Simple Moving Average}
  \acro{SNR}{signal-to-noise ratio}
  \acro{SORA}{Satisfaction Oriented Resource Allocation}
  \acro{SORA-NRT}{Satisfaction-Oriented Resource Allocation for Non-Real Time Services}
  \acro{SORA-RT}{Satisfaction-Oriented Resource Allocation for Real Time Services}
  \acro{SPF}{Single-Carrier Proportional Fair}
  \acro{SRA}{Sequential Removal Algorithm}
  \acro{SRS}{Sounding Reference Signal}
  \acro{SSB}{synchronization signal block}
  \acro{SU-MIMO}{single-user multiple input multiple output}
  \acro{SU}{Single-User}
  \acro{SVD}{Singular Value Decomposition}
  \acro{TCP}{transmission control protocol}
  \acro{TDD}{time division duplexing}
  \acro{TDMA}{Time Division Multiple Access}
  \acro{TETRA}{Terrestrial Trunked Radio}
  \acro{TP}{Transmit Power}
  \acro{TPC}{Transmit Power Control}
  \acro{TTI}{transmission time interval}
  \acro{TTR}{Time-To-Rendezvous}
  \acro{TSM}{Throughput-based Satisfaction Maximization}
  \acro{TU}{Typical Urban}
  \acro{UE}{user equipment}
  \acro{UEPS}{Urgency and Efficiency-based Packet Scheduling}
  \acro{UL}{uplink}
  \acro{UMTS}{Universal Mobile Telecommunications System}
  \acro{URI}{Uniform Resource Identifier}
  \acro{URM}{Unconstrained Rate Maximization}
  \acro{UT}{user terminal}
  \acro{VR}{Virtual Resource}
  \acro{VoIP}{Voice over IP}
  \acro{WAN}{Wireless Access Network}
  \acro{WCDMA}{Wideband Code Division Multiple Access}
  \acro{WF}{Water-filling}
  \acro{WiMAX}{Worldwide Interoperability for Microwave Access}
  \acro{WINNER}{Wireless World Initiative New Radio}
  \acro{WLAN}{Wireless Local Area Network}
  \acro{WMPF}{Weighted Multicarrier Proportional Fair}
  \acro{WPF}{Weighted Proportional Fair}
  \acro{WSN}{Wireless Sensor Network}
  \acro{WWW}{World Wide Web}
  \acro{XIXO}{(Single or Multiple) Input (Single or Multiple) Output}
  \acro{ZF}{zero-forcing}
  \acro{ZMCSCG}{Zero Mean Circularly Symmetric Complex Gaussian}
\end{acronym}
%\end{singlespace}

\acresetall

% -------------------------
% Content
% -------------------------
\renewcommand{\contentsname}{Contents}
\tableofcontents
\addtocontents{toc}{\protect\setcounter{tocdepth}{3}}  

% -------------------------
% Paper Sections
% -------------------------
% ===================================================================
% SECTION 1: Introduction
% ===================================================================
% \clearpage
\section{Introduction}
\label{sec:intro}

\Acp{RAN} are large-scale, real-time distributed systems that must operate reliably in highly dynamic and uncertain radio environments, while serving a broad range of connectivity services and applications. Currently, these systems rely heavily on manual intervention for configuration optimization and functional fine-tuning. This dependence on human expertise limits scalability, slows adaptation to environmental changes, and increases operational costs.

The next generation of communication networks is expected to address these limitations by becoming increasingly autonomous. This evolution—already underway in 5G-Advanced through standardized intent management frameworks, e.g.,~\cite{3GPP28312} and \cite{TMF:21}—envisions self-configuring, self-optimizing, and self-healing systems guided by high-level \emph{network intents}. Intents specify performance objectives, requirements, and constraints for a connectivity service or management workflow~\cite{3GPP28312}, allowing operators to express \emph{what} the network should achieve rather than \emph{how}. For example, an operator may specify a goal as \emph{“maximize user coverage while minimizing energy consumption,”} leaving the network to autonomously determine the appropriate actions, such as antenna tilt adjustments to improve coverage or carrier deactivation to save energy. In this context, intents act as directives, while the network abstracts away the implementation details, much like a compiler translates high-level code into machine-executable instructions.

Converting intents into network actions is fundamentally a problem of planning and reasoning across multiple abstraction layers—from natural-language specifications to optimization formulations, and ultimately to control policies executed at the \ac{RAN}. These requirements exceed the capabilities of current heuristic and rule-based approaches. Bridging this gap calls for a new class of \ac{AI} systems that move beyond perception and prediction, linking abstract objectives with dynamic decision-making through iterative reasoning and planning.

Agentic \ac{AI} has recently emerged as a promising paradigm for building autonomous, goal-driven systems capable of interpreting objectives, planning multi-step actions, and adapting to dynamic environments with minimal human oversight. Unlike traditional \ac{AI} approaches based on fixed heuristics or monolithic models, Agentic \ac{AI} structures intelligence into specialized agents that interact and cooperate through well-defined workflows~\citep{Sapkota:25}. Central to this paradigm are large-scale generative models—particularly \acp{LLM}—which enable agents to understand and generate natural language, decompose goals, generalize across tasks, invoke specialized tools, and reason in open-ended contexts~\citep{liu:24apigen}. As such, Agentic \ac{AI} offers a compelling architectural foundation for autonomous and intent-driven network management and optimization.

This paper takes a step toward realizing this vision by introducing an Agentic \ac{AI} system comprising an \textbf{interpreter}, an \textbf{optimizer}, and a \textbf{controller}. Our contributions are:

\begin{enumerate}
    \item \textbf{Cognitive intent processing.} The interpreter is a supervisory cognitive agent with two core functions: converting high-level intents into structured templates and recursively refining them on a slow timescale by reasoning over network observations and feedback on intent fulfillment. To meet \ac{RAN} compute and memory constraints, we adopt a dual-SLM architecture that separates intent translation and in-context reasoning among two \acp{SLM}.
    \item \textbf{Preference optimization.} The optimizer agent transforms \acp{OTM} into constrained optimization problems over a preference space, performs preference planning via Bayesian optimization to dynamically adapt preferences to network conditions, and steers the controller policy to satisfy the service intents expressed by the \ac{OTM}.
    \item \textbf{Multi-objective control.} The controller leverages \ac{MORL} to realize adaptive policies that operate near the Pareto front of network performance. A central technical contribution is \ac{D-EQL}, a scalable distributed variant of \ac{EQL} \cite{YSN:19} that: (i) decouples learner–actors with sharded prioritized replay for high-throughput training; (ii) distributes the exploration of the preference simplex across actors while learning a single preference-conditioned network; (iii) uses envelope updates with vector TD targets plus a cosine-stability loss; and (iv) refreshes priorities with hindsight preference relabeling. Together, these extensions improve scalability, accuracy and exploration over established \ac{MORL} art \cite{YSN:19, Basaklar:23}.
    \item \textbf{Proof of concept.} We showcase the agentic system through an intent-aware \ac{RRM} use case combining interpreter and optimizer agents with a novel MORL-based \ac{LA}, and adapt its policy on the fly to diverse connectivity service goals. Our approach outperforms traditional \ac{RL}—which cannot adapt a single policy across goals—and exceeds the state-of-the-art \ac{LA} baseline of 5G/5G-A systems.
\end{enumerate}

Results from high-fidelity system-level simulations of a 5G-compliant network suggest that Agentic \ac{AI} can transform high-level human intents into self-optimizing control mechanisms for next-generation networks, thereby paving the way toward scalable network autonomy.

% ===================================================================
% SECTION 2: Related Work
% ===================================================================
\section{Related Work}
\label{sec:related_work}

\paragraph{Agentic AI:} Agentic \ac{AI} is an emerging paradigm that structures intelligence as a modular network of specialized agents collaborating to achieve complex, high-level goals~\citep{Hughes:25}. Recent surveys highlight recurring design patterns and challenges related to reliability and evaluation~\citep{Guo:24LLMMAS,Li:24LLMMAS}. A central mechanism is \emph{goal decomposition}, whereby broad objectives are divided into subtasks handled by agents with distinct functions. Prior work has demonstrated that agents can integrate reasoning and action in recursive loops~\citep{Yao:23ReAct}, improve performance through reflective memory~\citep{Shinn:23Reflexion}, and operate collectively via structured communication~\citep{Wu:23AutoGen}. To coordinate distributed intelligence, orchestration layers or meta-agents assign roles, manage life cycles and task dependencies, and resolve conflicts using centralized or decentralized mechanisms~\citep{Qian:24}. Furthermore, persistent goals and memory enable adaptation over long time horizons~\citep{Wang:24Voyager,Agashe:25}. Domain-specific systems, such as MAGIS~\citep{Tao:24MAGIS}, illustrate how these principles scale to collaborative workflows.

\paragraph{Bayesian optimization:} \cite{ZhX:20} reviews the evolution of expected improvement (EI) as an acquisition function for surrogate-based optimization, detailing its extensions to parallel, multi-objective, constrained, noisy, multi-fidelity, and high-dimensional settings, analyzing their theoretical properties, and highlighting future research directions. \cite{ZYQ+:24} shows that the performance of high-dimensional Bayesian optimization is strongly limited by poor random initialization of acquisition function maximizers and proposes AIBO, a simple framework that uses past evaluations and heuristic search to generate better starting points, significantly boosting optimization efficiency.

\paragraph{Multi-objective reinforcement learning:} \ac{MORL} addresses control problems in which optimality is defined by a Pareto front of policies, each capturing different trade-offs among multiple objectives.

Early approaches to multi-objective optimization~\citep{Kim:05, Konak:06, Yoon:09} reduced the problem to scalar optimization—typically via utility functions with fixed weights across objectives—followed by standard \ac{RL}. These methods are tied to a single preference setting and cannot adapt when goals or constraints change~\citep{Liu:14}, thereby necessitating retraining. To improve generality, subsequent work sought to approximate the entire Pareto front by learning multiple optimal policies over the preference space~\citep{Sriraam:05, Barrett:08, Mossalam:16}. However, training a separate policy for each preference combination quickly becomes computationally infeasible in large domains.

A more scalable approach is to learn a single universal policy conditioned on preferences~\citep{YSN:19, Xu:20, Abdolmaleki:20}, enabling adaptation across tasks without retraining. For instance, \cite{YSN:19} proposed envelope Q-learning, which generalizes the Bellman equation to optimize the convex envelope of multi-objective Q-values under linear preferences using deep networks. Extensions such as those in~\cite{Basaklar:23} introduced parallelization to improve sample efficiency and Pareto approximation. Nonetheless, efficiently exploring the preference space and learning universal \ac{MORL} policies remain open challenges~\citep{Hayes:22}.

\paragraph{Agentic AI in Communication Systems:} Intent-based management is already part of modern 5G-Advanced systems~\citep{3GPP28312}, and its extension toward 6G is strongly supported in current standardization efforts~\citep{SA5-108:25}. Concurrently, academic and industrial interest in Agentic AI is rapidly growing, positioning it as a key enabler of next-generation autonomous networks, particularly for intent-driven operations~\citep{bimo2025, ZTE:25, NEC:25}. Recent work on agent-based and LLM-guided control frameworks for network optimization and service management~\citep{QAJ+:25, RSTN:25, bimo2025} highlights a shift toward systems capable of reasoning, adaptation, and collaboration. This trajectory is reflected across 3GPP, Open RAN, and TM Forum. For example, 3GPP TR 22.870~\cite{3GPP22.870} identifies AI-agent–enabled service coordination, LLM-assisted interactions, and agent-supported UE–network cooperation as 6G use cases, while~\cite{IETF:25} defines protocols for AI-agent communication. Furthermore, the 3GPP SA5 workgroup has identified \emph{intent-driven agentic autonomous management} as a priority areas for 6G~\cite{SA5workareas:25, SA5SI:25} while SA2 is examining agentic mechanisms for the 6G core network~\cite{SA2huawei:25}. Together, these developments indicate that agentic and intent-based paradigms are increasingly viewed as foundational elements of future 6G architectures.

\paragraph{Differentiation from Prior Agentic AI Work:} Existing Agentic AI systems have largely been applied to reasoning, planning, and tool use, where control loops operate over long timescales in relatively stable environments. By contrast, we integrate agentic AI into the fast control loops of \ac{RRM}, where sub-millisecond decisions must adapt to fading channels, mobility, and heterogeneous service requirements. To our knowledge, this is among the first applications of Agentic AI in highly dynamic, stochastic environments, extending its reach to performance-critical autonomous networks.

We demonstrate the workflow with an end-to-end, cognitively guided intent-aware \ac{RRM} design for supporting different connectivity services, where control policies adapted by reasoning over individual service goals and network observations are then executed in time-varying, frequency-selective environments to meet the goals. Our results show superior performance compared to traditional \ac{RL} and the state-of-the-art \ac{LA} algorithm adopted in 5G/5G-A systems.
% ===================================================================
% SECTION 3: Agentic AI System for RAN Control
% ===================================================================

\section{Agentic AI System for RAN Control}
\label{sec:method}

At its core, the proposed \emph{Agentic AI system} comprises three specialized agents—interpreter, optimizer, and controller—whose interactions form an \emph{agentic workflow} consisting of two loops: an \emph{intent management loop}, executed by the interpreter–optimizer pair, and an \emph{intent fulfillment loop}, executed by the optimizer–controller pair. Each loop operates on a distinct timescale, forming a two-timescale control architecture analogous to Kahneman’s dual-process theory~\citep{Kahneman:11book}, with a slower, deliberative outer System~2 and a faster, reactive inner System~1.

\begin{figure}[!ht]
    \centering
    \includegraphics[width=0.95\linewidth]{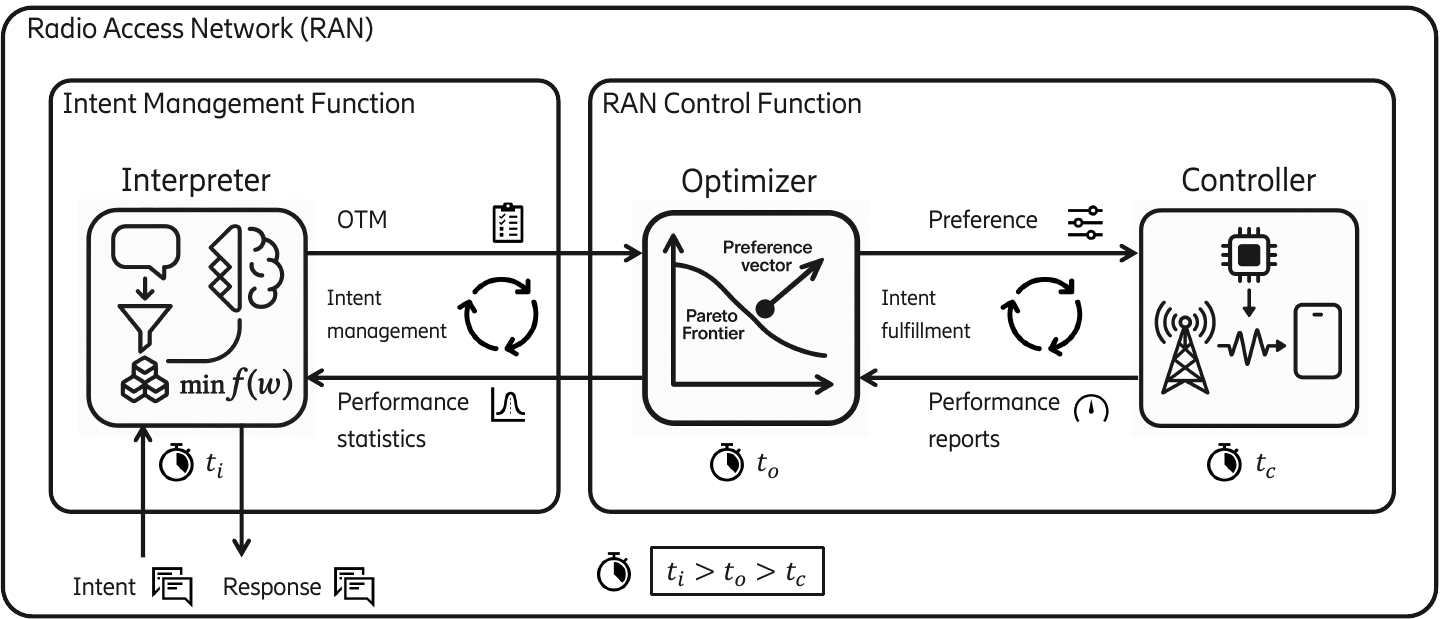}
    \caption{Agentic AI system for intent and resource management in autonomous networks.}
    \label{fig1:Agentic_workflow}
\end{figure}

The \emph{interpreter} is a supervisory cognitive agent that converts high-level intents into structured templates and adaptively refines them on a slow timescale using network states and fulfillment feedback. The \emph{optimizer} recursively plans and adjusts the downstream controller configurations to satisfy the intent, aggregating controller feedback into slower-timescale statistical summaries returned to the interpreter. The \emph{controller} executes real-time decision-making, collects observations, and provides periodic performance reports to the optimizer.

This triadic workflow provides a blueprint for a broader Agentic AI system for autonomous management and optimization of communication networks. Its realization, however, requires a twofold extension. \emph{Horizontally}, the interpreter may coordinate with multiple optimizer–controller pairs supporting different \ac{RAN} functions within a single architectural layer. \emph{Vertically}, the workflow can be embedded across different layers of the \ac{RAN} protocol stack, whose operational timescales range from slow (for network management at higher layers) to very fast (for \ac{RRM} at lower layers).

\subsection{Timescales Separation}

The workflow separates responsibilities across three timescales. The controller handles real-time decisions and thereby establishes the system’s reference timescale $t_{c}$. Because this agent replaces an existing \ac{RAN} control function, it inherits that function’s native latency budget, which may range from sub-millisecond operation for \ac{RRM} functions (e.g., link adaptation) to minutes or hours for network optimization tasks (e.g., cell shaping). The optimizer adjusts the controller’s policy at a deliberately slower timescale $t_{o}$, spanning hundreds of milliseconds to seconds for fast \ac{RRM} functions and up to hours for \ac{RAN} management functions, ensuring that its decisions do not interfere with the primary control loop. The interpreter operates on the slowest supervisory cadence $t_{i}$, which spans seconds to minutes for \ac{RRM} supervision and up to hours for \ac{RAN}-wide management. At this timescale, the interpreter evaluates intent feasibility, reasons over observed \ac{KPI} deviations, and generates refined intents without imposing timing constraints on downstream agents.

Decoupling long-term reasoning and intermediate adaptation from real-time control ensures that (a) the interpreter supervisory role is non-latency-critical; (b) latency-critical operations are confined to the controller—for any RAN control function involved; and (c) the fast control loop remains stable.

% ===================================================================
% SECTION 4: Language-Guided Intent Management
% ===================================================================
\section{Language-Guided Intent Management}\label{sec:intent_management}

% ===================================================================
\subsection{Interpreter Agent}
% ===================================================================

The interpreter is a language-guided supervisory agent aligned with the scope of an \ac{IMF}~\citep{TMF:24}. It performs two complementary functions: (a) \emph{transforming intents} into structured \acp{OTM}, and (b) \emph{cognitive reasoning} for recursive intent adaptation.

The interpreter agent must integrate domain awareness, intent stabilization, and adherence to the computational and memory constraints of the \ac{RAN} system.~Domain awareness includes understanding which control agents operate within each sub-domain, their capabilities, parameters, and timescales, as well as the \acp{KPI} they influence. This knowledge enables the interpreter to produce feasible \ac{OTM} formulations for a given intent, route each intent to the appropriate \ac{RAN} control agent, and ensure intent stabilization by reasoning over system observations, optimizer feedback, and network dynamics to perform safe, explainable \ac{OTM} refinements when required.

Meeting these requirements within current 5G/5G-A \ac{RAN} hardware necessitates a design that is both computationally efficient and functionally modular. Deploying a single large general-purpose \ac{LLM} is impractical due to compute and memory constraints in current \ac{RAN} platforms, and integrating dedicated accelerators is neither scalable nor cost-effective. To address this, we adopt a dual-SLM architecture that separates the interpreter’s two core functions—intent translation and cognitive reasoning—across two lightweight, complementary \acp{SLM}, as detailed in~\Cref{appendix:interpreter_interpreter_agent_details}.

\textbf{Intent translation.} This module is the workflow entry point. It interprets the intent, decomposes it into sub-intents, selects the appropriate downstream control agent, and initiates the intent-fulfillment loop. A fine-tuned \ac{SLM} renders the intent as a structured, schema-compliant \ac{OTM} by disambiguating objectives, constraints, requirements, and metadata. This step extends beyond lexical parsing: the model must map high-level intents into optimization structures grounded in domain knowledge. Using a fine-tuned \ac{SLM} ensures low-complexity generation of machine-readable \acp{OTM} that reflect RAN semantics and remain robust to linguistic variability.~\Cref{appendix:B:otm} discusses the generality of the \ac{OTM} schema, while~\Cref{app:C:data_curation} outlines the fine-tuning of a Qwen-2.5-7B-Instruct model~\cite{QWEN:25}, which achieves the high schema validity and \ac{OTM} accuracy shown in~\Cref{sec4:Interpreter_accuracy}.

\textbf{Cognitive reasoning and adaptation.} Complementing the translator, a lightweight general-purpose \ac{SLM} performs supervisory reasoning via in-context learning. It evaluates feasibility, diagnoses constraint violations, and refines \acp{OTM} when strict requirements cannot be met, proposing alternative trade-offs or adapting objectives to evolving network conditions. Intent stabilization is achieved through structured monitoring, advisory evaluation, and guarded execution (see~\Cref{appendix:interpreter_interpreter_agent_details}). This supervisory closed-loop reasoning extends beyond static templates or rule-based logic and is essential for autonomous, intent-driven, network management under real-world network dynamics.

This division of labor preserves contextual knowledge and ensures adaptability for intent handling, while remaining compatible with practical constraints of contemporary \ac{RAN} deployments. The dual-\ac{SLM} interpreter—built from small-scale models—and the infrequent, non-latency-critical nature of \ac{SLM} inference within the agents’ timescale separation allow the system to maintain low compute and energy overhead. As a result, the overall design is feasible on current 5G/5G-Advanced hardware.

\begin{table}[t]
\centering
\begin{tabular}{l c c c c}
\toprule
\textbf{Model} &
\textbf{Schema accuracy} &
\multicolumn{3}{c}{\textbf{OTM accuracy}}\\
 & & Objectives & Constraints & Overall \\
\midrule
Qwen-2.5-7B-Instruct (Before fine-tuning)    & 100.0\%  & 45.00\% & 21.50\% & 11.30\%   \\
Qwen-2.5-7B-Instruct (After fine-tuning)  & 100.0\% & \textbf{100.0\%} & \textbf{98.00\%}  & \textbf{98.00\%}\\
\bottomrule
\end{tabular}
\caption{Schema and \ac{OTM} accuracy for interpreters using the Qwen-2.5-7B-Instruct model.}
\label{sec4:Interpreter_accuracy}
\end{table}

%===============================================================================
\subsection{Optimizer Agent}
%===============================================================================

The optimizer agent performs three key tasks: (i) decoding the \ac{OTM} received from the interpreter, (ii) recursively solving the associated optimization problem to align the controller’s policy with the intent, and (iii) coordinating the two feedback loops within the workflow. Upon receiving an \ac{OTM}, the optimizer formulates a constrained optimization problem aligned with the specified intent, such as
\begin{equation}
    \begin{aligned}
        \underset{\boldsymbol{\omega} \in \Omega}{\text{minimize}} \quad & f(\boldsymbol{\omega}) \\
        \text{subject to} \quad & g_{i}(\boldsymbol{\omega}) \leq b_i, \quad i = 1, \dots, p,
    \end{aligned}
    \label{eq:optimization}
\end{equation}
where \( f(\boldsymbol{\omega}) \) quantifies the system performance (e.g., energy, latency, throughput), and the decision variable \( \boldsymbol{\omega} \) belongs to a feasible set \( \Omega \subseteq \mathbb{R}^m \). The inequality constraints \( g_i(\boldsymbol{\omega}) \leq b_i \) capture operational limitations—e.g., bandwidth, latency, or power—or service requirements. Since both objective and constraints are often non-convex, the solution landscape may contain multiple local optima, making the identification of feasible or optimal solutions challenging.

The decision variables \( \boldsymbol{\omega} \) link the optimizer to the controller by representing hyperparameters that tune the controller’s policy. In our framework, the controller follows a \ac{MORL} approach (\Cref{sec:intent_fulfillment}), so \( \boldsymbol{\omega} \) corresponds directly to the preference weights in its multi-dimensional reward function.

Since the explicit forms of \( f \) and \( g_i \) are unknown and their evaluations are computationally expensive, the optimizer employs \ac{BO}, leveraging surrogate models trained on \ac{RAN} performance data (e.g., throughput, spectral efficiency, \ac{BLER}) relevant to the intent. These models guide the exploration of preference weights \( \boldsymbol{\omega} \) (i.e., decision variables), which steer the controller’s actions. Additional details of the \ac{BO} design are provided in~\Cref{appendix:Bayesian_optim}.

%===============================================================================
\subsubsection{PAX-BO: Preference-Aligned eXploration Bayesian Optimization}
%===============================================================================

We next address the preference-based constrained \ac{BO} problem~(\ref{eq:optimization}) in the multi-service case, where $S$ connectivity services must be jointly optimized under $p$ constraints that capture requirements such as data rate, latency, and reliability. The optimization problem~(\ref{eq:optimization}) becomes
\begin{equation}
    \begin{aligned}
        \underset{\mathbf{W} \in \Omega^S}{\text{minimize}} \quad 
        & f(\mathbf{W}) \\
        \text{subject to} \quad 
        & g_i(\mathbf{W}) \leq b_i, \quad i = 1, \dots, p,
    \end{aligned}
    \label{eq:optimization_multi}
\end{equation}
where $\mathbf{W} = [\vomega^{(1)}, \dots, \vomega^{(S)}]$ collects the service-specific preference vectors $\vomega^{(s)} \in \Omega$ ($\Omega = \Delta^{m-1}$) on the probability simplex. The objective $f(\mathbf{W})$ quantifies system-wide performance, while the constraints $g_i(\mathbf{W}) \leq b_i$ enforce joint service requirements. Problem (\ref{eq:optimization_multi}) reduces to Problem (\ref{eq:optimization}) when $S=1$.

\textbf{PAX-BO}, shown in~\Cref{alg:paxbo}, solves Problem~(\ref{eq:optimization_multi}) by optimizing preference vectors on the simplex through \ac{BO} in an unconstrained internal space. Let $U=[u^{(1)},\dots,u^{(S)}]\in\mathbb{R}^{m\times S}$ and $\bar u=\mathrm{vec}(U)$. Each service $s$ has a projected simplex weight $\vomega^{(s)}=\Pi_\Delta(u^{(s)})\in\Delta^{m-1}$, and $\mathbf{W}(U)=[\vomega^{(1)},\dots,\vomega^{(S)}]\in(\Delta^{m-1})^S$. At each iteration, we fit surrogate models that approximate the system objective and constraints as $\mathcal{F}(\bar u)\approx f(\mathbf{W}(U))$ and $\mathcal{G}_i(\bar u)\approx g_i(\mathbf{W}(U))$, and build a constraint-aware acquisition $\alpha(\bar u)$ (e.g., Log-EI times a feasibility probability).

A \emph{\ac{TR}}—an $\ell_\infty$ box with center $s_c$ and radius $L\in[L_{\min},L_{\max}]$—constrains local exploration. At each iteration, the acquisition function is maximized within the \ac{TR}, and the solution is projected back onto the simplex:  
\[
\bar u_t=\arg\max_{\ \|\bar v-s_c\|_\infty\le L}\ \alpha(\bar v),\quad
U_t=\mathrm{mat}(\bar u_t),\quad \mathbf{W}_t=\Pi_\Delta(U_t).
\]  
After evaluating $o_t=f(\mathbf{W}_{t-1})$ and $c_t^{(i)}=g_i(\mathbf{W}_{t-1})$, we declare success if $c_t^{(i)}\le 0$ for all $i$ and $o_t\ge f_{t-1}^\star+\epsilon$, with $\epsilon\ll 1$. On success, we set $f_t^\star\!\leftarrow o_t$, $s_c\!\leftarrow\bar u_{t-1}$, and expand $L$ after $s_{\mathrm{th}}$ consecutive successes; otherwise, $L$ is shrunk after $f_{\mathrm{th}}$ failures, clamped to $[L_{\min},L_{\max}]$. If the TR stalls at $L_{\min}$ for $w$ rounds, a \emph{reset} is triggered: $n$ candidates are sampled from $(\Delta^{m-1})^S$, scored by (acquisition)$\times$(feasibility)$\times$(novelty), and the best candidate reinitializes $s_c$ with $L\leftarrow L_0$.  

Overall, PAX-BO is a lift-and-project \ac{BO} method with \ac{TR} safeguards and reset mechanisms, tailored to simplex-valued preferences that jointly influence a constrained system objective.

% ===================================================================
% SECTION 5: Preference-Guided Intent Fulfillment
% ===================================================================
\section{Preference-Guided Intent Fulfillment}\label{sec:intent_fulfillment}

The optimizer and controller agents operate in a closed loop to achieve intent fulfillment. The optimizer recursively adapts the preference vector $\boldsymbol{\omega}$ based on performance feedback from the controller. The optimal (or near-optimal) vector $\vomega^{\star}$, obtained by solving~(\ref{eq:optimization}), is then passed to the controller, which aligns network actions with the original intent.

\subsection{Controller Agent}

The controller implements a policy trained via \ac{D-EQL}, a distributed extension of \ac{EQL}~\citep{YSN:19}. \ac{D-EQL} learns a single policy/value network conditioned on a linear preference vector $\vomega \in \Omega$ (the probability simplex) and scales exploration through a learner–actor architecture with prioritized replay (cf.~\cite{Horgan:18}).

During training, actors are assigned to distinct strata of the simplex defined by a simplex-lattice partition. Each actor samples preferences uniformly within its stratum using barycentric sampling, executes an $\varepsilon$-greedy policy with the scalarization
\[
Q_{\vomega}(s,a;\vtheta)=\vomega^\top Q(s,a,\vomega;\vtheta),
\]
and initializes replay priorities by drawing an independent preference $\tilde{\vomega}$ to compute a scalar temporal-difference error. Transitions and priorities are batched locally and sent to sharded replay buffers.  

The learner assigns strata of the simplex to actors for distributed exploration, retrieves prioritized minibatches from all shards, samples preferences from a Dirichlet distribution, and forms a Cartesian product so that each transition is evaluated under every sampled preference. The learner performs envelope backups by maximizing over actions and supporting preferences, updates parameters using a regression loss with an optional cosine-alignment term, refreshes priorities, and periodically synchronizes the target network. Updated weights are then broadcast to all actors.  

The envelope backup is expressed as
\[
\vy=\vr+\gamma(1-d)\,Q(s',a^\star,\tilde\vomega^\star;\vtheta^{-}),\quad
(a^\star,\tilde\vomega^\star)=\argmax_{a',\,\vomega'\in\Omega}\ \vomega^\top Q(s',a',\vomega';\vtheta).
\]
Compared with state-of-the-art \ac{MORL} algorithms such as~\cite{YSN:19} and~\cite{Basaklar:23}, \ac{D-EQL} introduces (i) a hindsight replay memory with prioritized sampling and priority updates, (ii) partitioned exploration of the preference space across distributed asynchronous actors, and (iii) a sharded replay memory. This architecture improves scalability in environments with large state–action–preference spaces by enabling systematic simplex exploration, dense preference supervision, and high-throughput stable learning. As shown in~\Cref{sec5:DEQLtable}, \ac{D-EQL} achieves a 22.1\% performance CFR1 improvement over~\cite{YSN:19} and an additional 8\% gain over~\cite{Basaklar:23} in the Fruit Tree Navigation environment with depth 7, as well as 89.37\% hypervolume improvement over~\cite{YSN:19} and an extra 6.05\% gain over~\cite{Basaklar:23}. Additional design details and extended comparisons are provided in~\Cref{appendix:distributed_morl}.

\begin{table}[t]
\centering
\footnotesize
\begin{tabular}{l c c c c c c c c c}
\toprule
\textbf{Algorithm} &
\textbf{Partition} &
\multicolumn{4}{c}{\textbf{Replay memory}} &
\multicolumn{2}{c}{\textbf{Actor}} &
\textbf{CFR1} &
\textbf{Hypervol.} \\
 & & Hindsight & Sampling & Update & Sharded & Distrib. & Comm.& Improv.  & Improv.\\
\midrule
\cite{YSN:19}      & No  & Yes & Prioritized & No  & No  & No  & --  & --  & -- \\
\cite{Basaklar:23}  & Yes & Yes & Uniform     & No  & No  & Yes & Synch.  & 12.33\%  & 78.56\% \\
\ac{D-EQL} (ours)   & Yes & Yes & Prioritized & Yes & Yes & Yes   & Asynch. & \textbf{22.10\%} & \textbf{89.37\%}\\
\bottomrule
\end{tabular}
\caption{Comparison of D-EQL with \cite{YSN:19} and \cite{Basaklar:23} in terms of design features and achieved CFR1 performance in Fruit Tree Navigation with depth 7.}\label{sec5:DEQLtable}
\end{table}

% ===================================================================
% SECTION 6: Case Study: Agentic Radio Resource Management
% ===================================================================
\section{Case Study: Agentic Radio Resource Management}\label{sec:MO_radio_resource_management}

\Ac{RRM} encompasses some of the most demanding and dynamic control functions in \acp{RAN}, including user scheduling, resource allocation, link adaptation, power control, and beamforming. These mechanisms operate on sub-millisecond timescales and must continuously adapt to the stochastic nature of the wireless channel to maintain reliable and efficient over-the-air communications.

As proof of concept, we apply our Agentic AI system to support differentiated connectivity services using a \ac{MORL}-based controller agent for \ac{LA}—a key function that tunes \ac{MCS} parameters to the radio link capacity. The detailed description of the \ac{MORL} \ac{LA} controller agent is provided in~\Cref{appendix:study_case}. Here, we note that the reward is a vector $\vr = [r_1, r_2]^{\top} \in \mathbb{R}^{2}$ with two competing components: $r_1$ measures the number of information bits successfully delivered per packet, and $r_2$ captures the time–frequency resources consumed per packet transmission.

In our agentic system, the \ac{MORL} \ac{LA} controller agent defines the fastest operational timescale, running on a sub-millisecond cadence. This cadence sets the reference timescale for dimensioning the optimizer and interpreter. The optimizer updates the preference weights of the \ac{MORL} controller once per second, based on performance reports and observed network conditions. This update rate is fast enough to steer the controller toward \ac{MCS} selections aligned with the intent goals, yet slow enough not to interfere with the stability of the \ac{LA} decision loop.

At the same cadence, the optimizer agent provides feedback to the interpreter agent for supervisory monitoring of intent fulfillment. However, the interpreter's cognitive loop is triggered only on an event-driven basis. Upon receiving an alert message from the optimizer, the interpreter leverages its general-purpose \ac{SLM} to perform cognitive reasoning over \acp{KPI} deviations, intent-fulfillment, and evolving network conditions to determine whether the intent must be refined. In our case study, such intervention occurs when changing network conditions render the service requirements temporarily infeasible.

% ===================================================================
% SECTION 7: Experiment
% ===================================================================
\section{Experiment}\label{sec:experiment}

This section evaluates the empirical performance of our  Agentic AI system for intent-aware \ac{RRM} using a 5G-compliant event-driven network simulator. We validate our approach in three steps using a multi-cell setup described in~\Cref{appendix:Experiments}: First we validate the \ac{MORL} controller agent design; secondly, we evaluate the optimizer-controller loop; and lastly we benchmark the overall workflow. 

\begin{figure*}[t!]
    \centering
    \begin{subfigure}{0.48\textwidth}
        \centering
        \includegraphics[width=\linewidth]{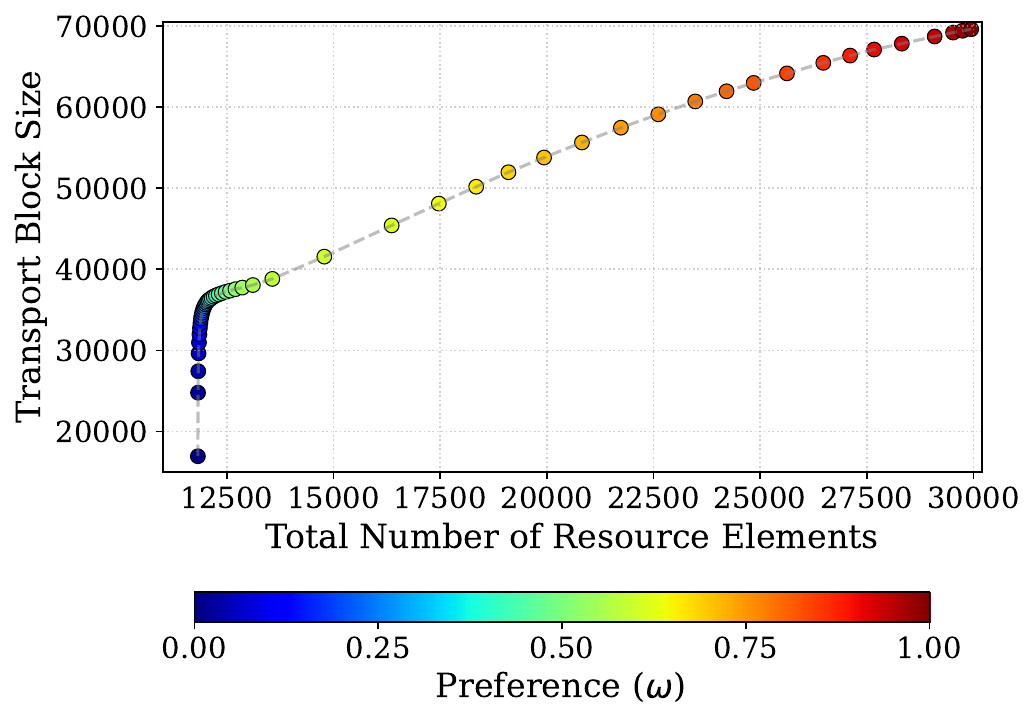}
        \caption{Pareto frontier illustrating the trade-off between transport block size and resource utilization.}
        \label{fig2a:pareto_frontier}
    \end{subfigure}
    \hfill
    \begin{subfigure}{0.48\textwidth}
        \centering
        \includegraphics[width=\linewidth]{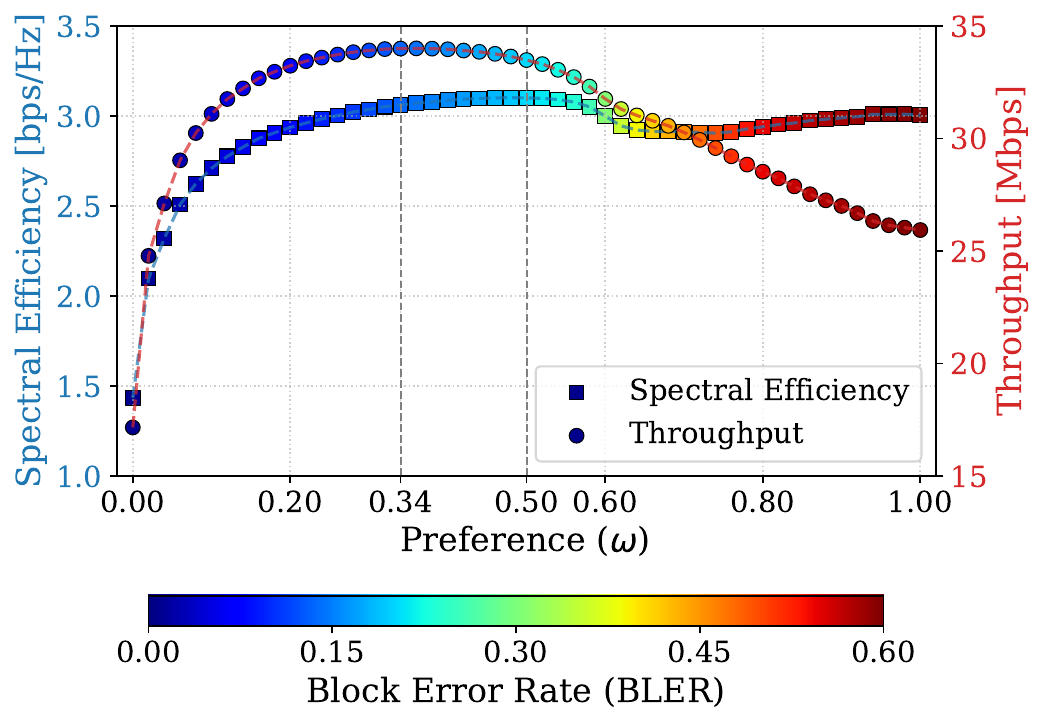}
        \caption{Joint characterization of connectivity service \acp{KPI} dependence on preference weights $\omega$.}
        \label{fig2b:link_KPIs}
    \end{subfigure}
    \caption{Characterization of preference-guided \ac{LA} using \ac{MORL} to satisfy service intents.}
    \label{fig:MOMDP_characterization}
\end{figure*}

% ===================================================================
\subsection{MORL Controller Agent for Link Adaptation}
\label{subsec:MORL_controller_LA}
% ===================================================================

\Cref{fig:MOMDP_characterization} illustrates how the preference-guided \ac{MORL} controller for \ac{LA} steers trade-offs among service \acp{KPI}, like spectral efficiency, throughput, and \ac{BLER}, assuming long communication sessions (e.g., streaming services). \Cref{fig2a:pareto_frontier} shows the Pareto frontier for the two reward components, while \Cref{fig2b:link_KPIs} maps each point on the frontier to link-level \acp{KPI}. When $\omega \approx 0$, the controller selects conservative \ac{MCS} values, resulting in resource efficient and high-reliable transmissions (with near-zero \ac{BLER}), but at the cost of low throughput (i.e., due to small \acp{TBS}) and spectral efficiency. At the other extreme, $\omega \approx 1$ drives aggressive \ac{MCS} choices that exploit retransmissions to target a spectral efficiency beyond the channel capacity, inducing resource-hungry and unreliable transmissions (with \ac{BLER} $\approx60\%$). The best operating points emerge for intermediate preferences, with $\omega \approx 0.34$ maximizing throughput and $\omega \approx 0.5$ maximizing spectral efficiency.~\Cref{appendix:Experiments} extends the analysis to examples with multiple connectivity services.

% ===================================================================
\subsection{Intent-fulfillment loop validation}
\label{subsec:optimizer_controller_loop}
% ===================================================================

Next, we evaluate \emph{only} the optimizer–controller loop, assuming a single forward interaction with the interpreter to obtain an \ac{OTM}. That is, when stochastic changes in the RAN environment render the \ac{OTM} specifications infeasible, the interpreter’s cognitive refinement loop is not triggered. While the optimizer–controller pair cannot resolve temporary infeasibility caused by evolving \ac{RAN} conditions.

We illustrate this by considering an intent that combines two contrasting connectivity services:

\begin{quoteBox}
\emph{Maximize cell throughput while serving mobile broadband users on a best-effort basis, and guaranteeing 99.99\% reliability for a ultra-reliable traffic.}
\end{quoteBox}

This intent reflects \ac{QoS} requirements for streaming and reliable services. In the agentic workflow, the interpreter constructs an \ac{OTM} that (a) identifies the two services, (b) defines an overall objective based on their achieved throughput, and (c) formulates a reliability constraint for the reliable service. The optimizer then instantiates an optimization problem to adapt the two vectors, $\boldsymbol{\omega}_{\mathrm{mbb}} = [\omega_{\mathrm{mbb}}, 1-\omega_{\mathrm{mbb}}]^{\top}$ and $\boldsymbol{\omega}_{\mathrm{rel}} = [\omega_{\mathrm{rel}}, 1-\omega_{\mathrm{rel}}]^{\top}$, each aligned to a service, by maximizing the aggregate throughput $f(\boldsymbol{\omega}_\mathrm{mbb}, \boldsymbol{\omega}_\mathrm{rel}) = f_\mathrm{mbb}(\boldsymbol{\omega}_\mathrm{mbb}, \boldsymbol{\omega}_\mathrm{rel}) + f_\mathrm{rel}(\boldsymbol{\omega}_\mathrm{mbb}, \boldsymbol{\omega}_\mathrm{rel})$  subject to the reliability constraint $g_\mathrm{rel}(\boldsymbol{\omega}_\mathrm{mbb}, \boldsymbol{\omega}_\mathrm{rel})\geq 0.9999$.

\Cref{fig:3a} shows the optimizer–controller dynamics over a two-minute simulation. After an initial warm-up phase, the PAX-BO optimizer steers $\boldsymbol{\omega}_\mathrm{mbb}$ and $\boldsymbol{\omega}_\mathrm{rel}$ so that the \ac{D-EQL} controller applies Pareto-optimal policies matched to each service’s requirements under varying network conditions. For reliable services, the optimizer converges to $\omega_{\mathrm{rel}}\approx 0$ (consistent with~\Cref{fig2b:link_KPIs}), driving the controller toward conservative \ac{MCS} selections that deliver ultra-reliable performance throughout the simulation—exceeding 99.99\% reliability in 94\% of the run. Only a few packets are lost during isolated deep-fading episodes; under persistent fading, the interpreter could be invoked to relax the reliability target. For enhanced-streaming traffic, the optimizer converges to $\omega_{\mathrm{mbb}}\approx 0.45$, prioritizing higher mean user throughput.~\Cref{appendix:Experiments} provides additional analysis and results.

\Cref{fig:3b} and \Cref{fig:3c} show that our agentic system outperforms both the state-of-the-art \ac{OLLA} used in 5G systems and the traditional \ac{RL}-based \ac{LA} of~\cite{demirel:2025}. Unlike our approach—which adapts a single D-EQL model on-the-fly to different connectivity requirements and radio conditions—both \ac{OLLA} and traditional \ac{RL} require separate configurations optimized for each service type. For \ac{OLLA}, we consider a standard target \ac{BLER} of 10\% for maximizing throughput in streaming services and 1\% for highly reliable transmissions\iffalse(corresponding to a packet error rate of $10^{-10}$ with four retransmissions)\fi. Traditional \ac{RL} similarly requires distinct models with reward functions tailored to each service; following~\cite{demirel:2025}, we use robustness parameters $\alpha=0.5$ for throughput and $\alpha=2$ for reliability. \Cref{fig:3b} shows that our agentic system achieves substantially lower \ac{BLER} for reliable services than both \ac{OLLA} and the \ac{RL} baseline with $\alpha=2$, yielding more reliable transmissions. \Cref{fig:3c} further shows that the same D-EQL model also attains throughput comparable to an \ac{RL} model explicitly trained for throughput optimization. While D-EQL handles both services with a single model, using multiple \ac{RL} models is impractical: inference must complete within a few hundred microseconds for all users, making rapid model switching across services infeasible.

\begin{figure*}[t!]
    \centering
    \begin{subfigure}{0.33\textwidth}
        \centering
        \includegraphics[width=\linewidth]{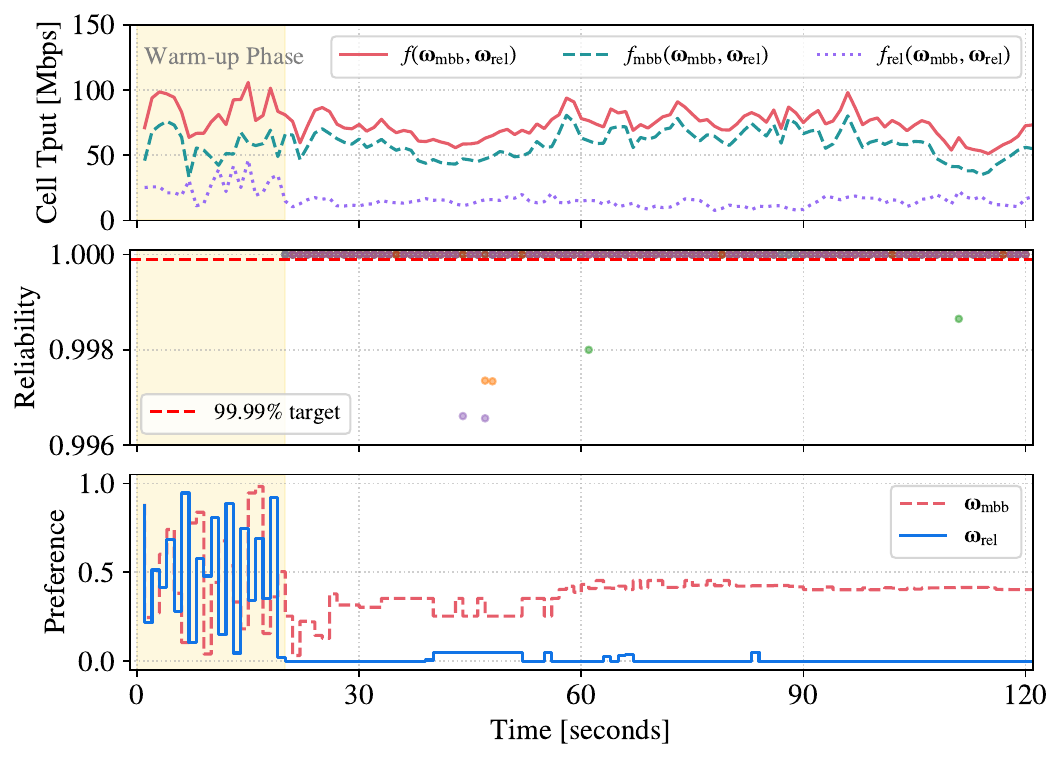}
        \caption{Time series of services \acp{KPI}.}
        \label{fig:3a}
    \end{subfigure}
    \hfill
    \begin{subfigure}{0.32\textwidth}
        \centering
        \includegraphics[width=\linewidth]{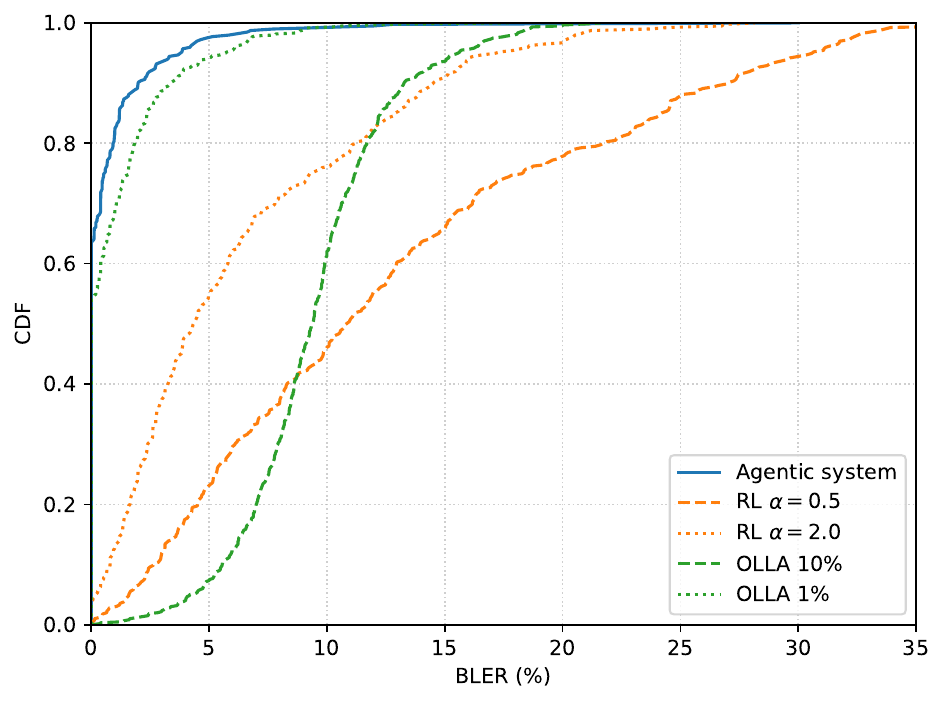}
        \caption{BLER for reliable service.}
        \label{fig:3b}
    \end{subfigure}
    \hfill
    \begin{subfigure}{0.32\textwidth}
        \centering
        \includegraphics[width=\linewidth]{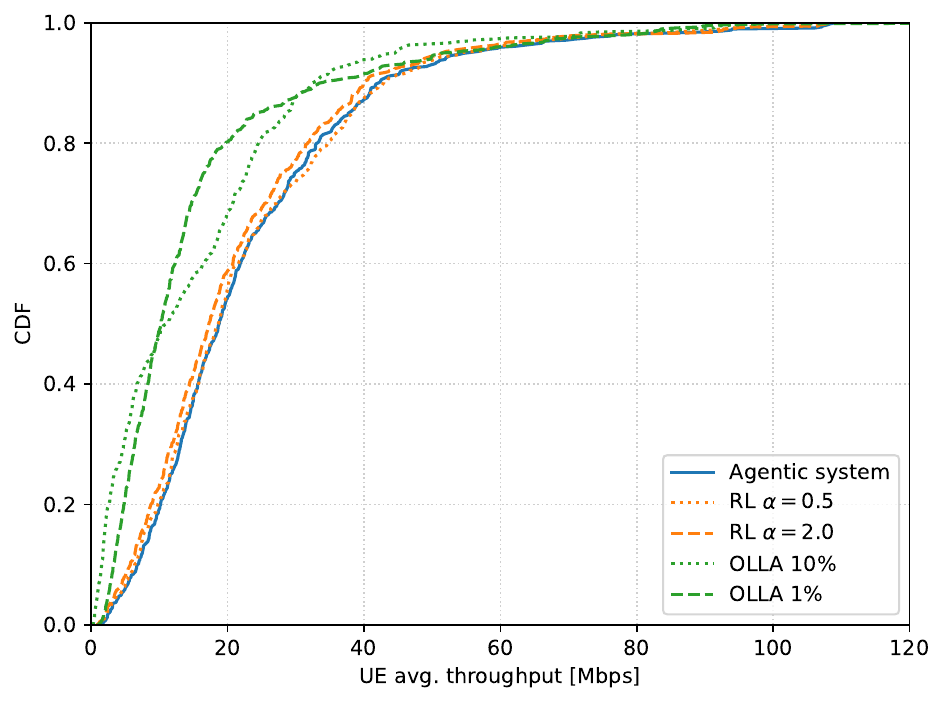}
        \caption{Streaming users throughput.}
        \label{fig:3c}
    \end{subfigure}
    \caption{Validation of the intent fulfillment loop between optimizer-controller for two examples.}
    \label{fig:3_optimizer_controller_loop}
\end{figure*}

% ===================================================================
\subsection{Triadic Agent Workflow Validation}
\label{subsec:triadic_agents_workflow_validation}
% ===================================================================

We next evaluate the complete agentic AI system, with both intent management and intent fulfillment loops working in unison to provide a continuous solution to an intent formulation that combines a primary system objective (i.e., cell throughput) with \ac{QoS} requirements of a connectivity service:

\begin{quoteBox}
\emph{Maximize cell throughput and serve streaming users with a minimum average data rate of \qty{7}{Mbps} whenever possible.}
\end{quoteBox}

The peculiarity of this problem stems from the highly likelihood of the \ac{QoS} requirements to become infeasible for users with poor channel conditions (such as cell-edge and high mobility users). When such an event occurs, persisting with a rigid \ac{QoS} requirement would induce the system to over-provision users with poor channel regardless of their inability to meet the \ac{QoS} goal, at the expense of users with a better channel quality. In turns, this may induce users with better channel to achieve lower throughput (due to less resources) and therefore compromise the primary intent objective.

\Cref{fig:4_three_agents_workflow_validation} compares the agentic \ac{AI} system with two settings: (a) a formulation with rigid \ac{QoS} requirements; and (b) a formulation with flexible \ac{QoS} requirements. In the latter case, when the optimizer agent alerts the interpreter agent of a consistent violation of the service constraint, the interpreter reasons over the cause of the problem and plans a solution to relax the \ac{QoS} requirements. 

\Cref{fig:3agents_B} shows an instance of this intent management loop between the interpreter-optimizer agents, where the latter reacts to the constraint violation by relaxing the service threshold, in an attempt to improve the primary objective, and providing a revised \ac{OTM}. This choice allows the optimizer agent to choose an $\omega$ setting that guides the \ac{MORL} controller towards a less aggressive \ac{MCS} selection policy for \ac{LA}, making packets transmissions more reliable for users with poor channel conditions.

Despite the interpreter’s recursive adaptation of \ac{QoS} requirements, infeasibility may still persist. This occurs because (a) the adaptor module includes guardrails that prevent abrupt \ac{QoS} changes during \ac{OTM} refinement (cf.~\Cref{appendix:interpreter_interpreter_agent_details}); and (b) prolonged poor channel conditions—such as deep fading, high pathloss, or shadowing—may yield spectral efficiencies too low to satisfy the \ac{QoS} constraints, regardless of how the interpreter adjusts them. Nonetheless, adapting the \ac{OTM} still yields tangible system-level benefits. By relaxing \ac{QoS} targets for users in persistently poor channel conditions, the system frees radio resources that can be reallocated to users with better channel quality, thus with higher spectral efficiency. This redistribution increases the primary intent objective (cell throughput), even if some individual \ac{QoS} constraints remain infeasible. As illustrated in~\Cref{fig:3agents_A}, once \ac{OTM} adaptation begins in the second half of the simulation, the cell throughput improves by a 4.79\%.

\begin{figure*}[t!]
    \centering
        \begin{subfigure}{0.50\textwidth}
        \centering
        \includegraphics[width=\linewidth]{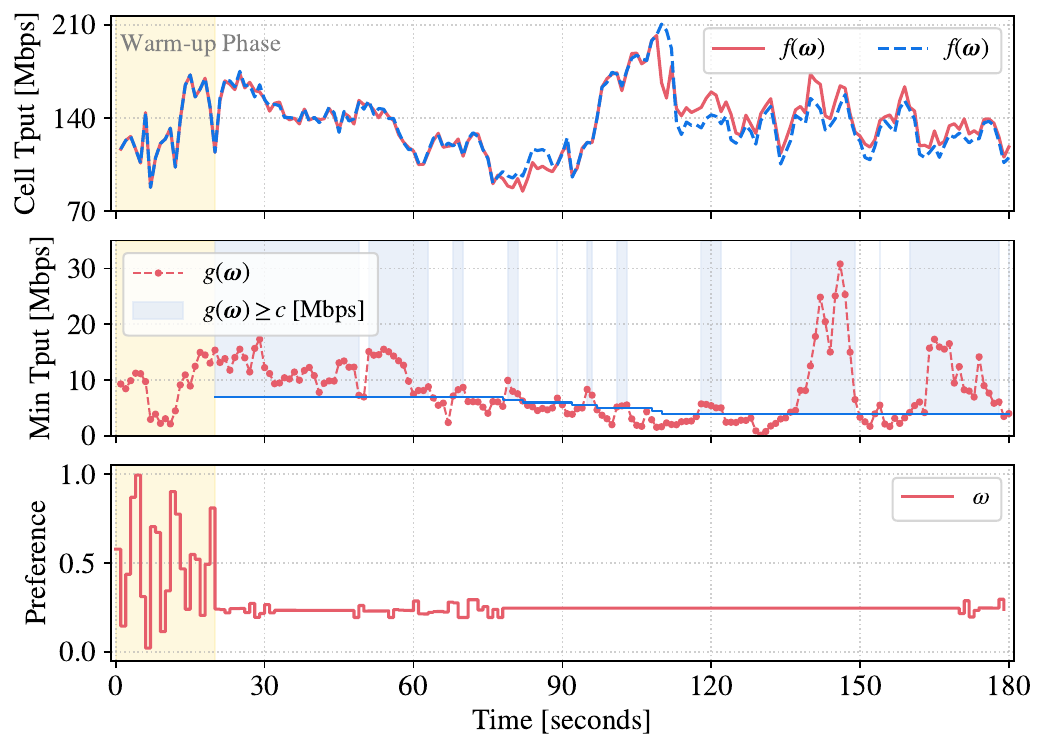}
        \caption{Agentic system with/without \ac{OTM} refinement.}
        \label{fig:3agents_A}
    \end{subfigure}
        %\hfill
    \begin{subfigure}{0.47\textwidth}
        \centering
        \includegraphics[width=\linewidth]{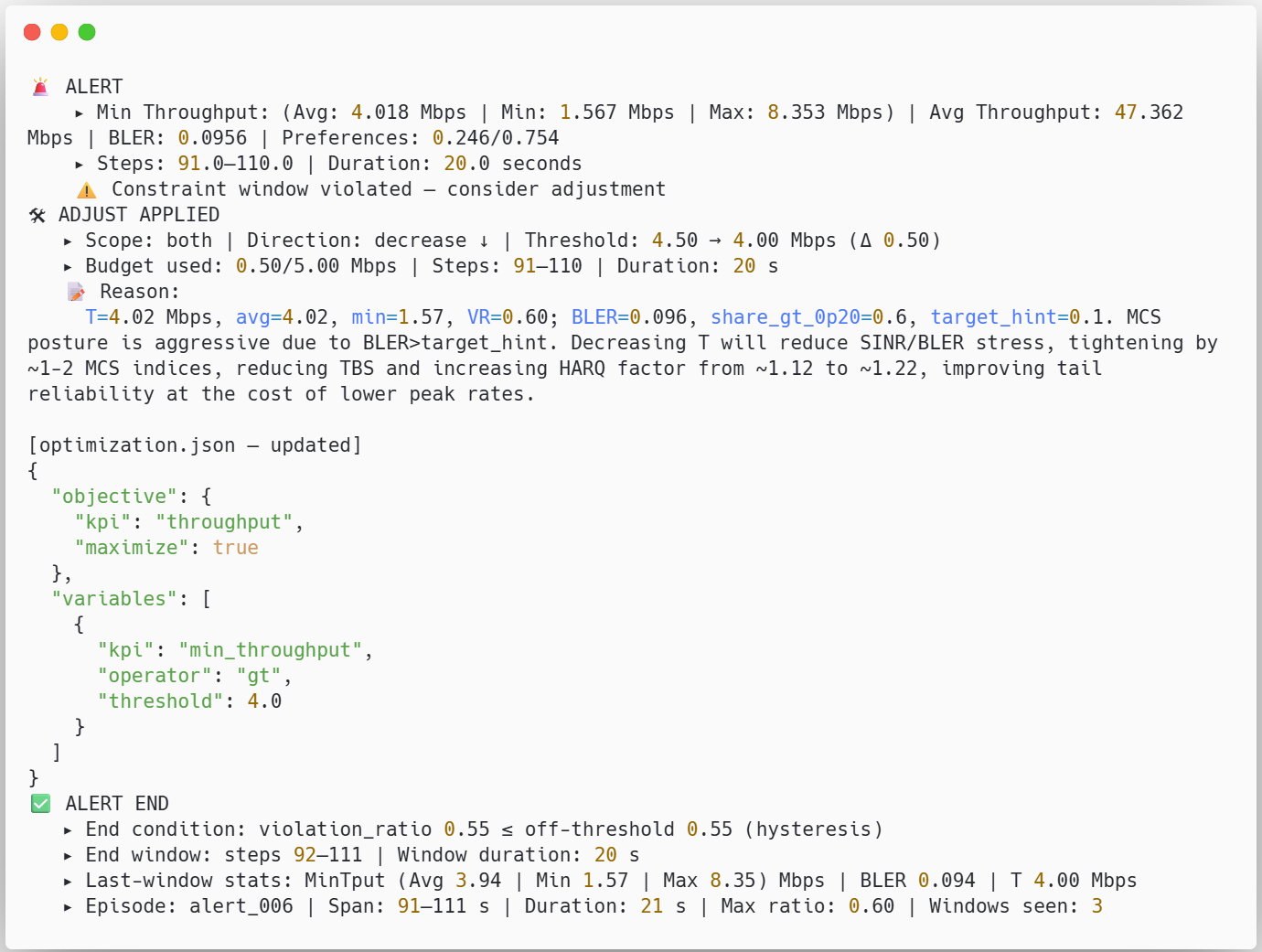}
        \caption{Intent management loop signaling.}
        \label{fig:3agents_B}
    \end{subfigure}
    \caption{Validation of the full agentic workflow, with intent management loop and intent fulfillment loop working in unison. We compare two formulations with rigid and flexible service requirements.}
    \label{fig:4_three_agents_workflow_validation}
\end{figure*}

% ===================================================================
% SECTION 8: Conclusions
% ===================================================================
\section{Conclusions}\label{sec:conclusion}

We presented an Agentic AI system for intent-driven control in autonomous networks, structured around three cooperating agents: interpreter, optimizer, and controller. Their coordinated interaction links high-level service intents to concrete network actions, enabling continuous reasoning, trade-off resolution, and real-time adaptation across multiple timescales of autonomous network control.

Our contributions span the full intent-to-control pipeline. The interpreter uses a lightweight dual-\ac{SLM} architecture to convert natural-language intents into structured optimization templates, assess feasibility, diagnose constraint violations, and refine templates using optimizer feedback. The optimizer performs preference planning via \ac{BO}, dynamically adjusting the downstream controller’s policy to meet the service requirements encoded in the template. The controller builds on \ac{MORL} to execute fast-timescale actions and adapt policies to evolving network conditions. To support this role, we introduce a distributed \ac{MORL} algorithm that integrates envelope Q-learning with actor–learner decoupling, preference-space exploration, and prioritized hindsight replay, improving scalability, exploration coverage, and performance over state-of-the-art \ac{MORL} approaches.

Proof-of-concept experiments in a high-fidelity, 5G-compliant \ac{RAN} simulator demonstrate that the proposed system reconciles heterogeneous service requirements—including throughput and reliability—while operating near the Pareto front of network performance and adapting effectively to dynamic conditions, exceeding traditional \ac{RL} and state-of-the-art functions of in 5G/5G-A systems.

Looking ahead, a key challenge is scaling this workflow across hierarchical layers of the \ac{RAN}—from cell-level control to cluster-level coordination and end-to-end service orchestration—while ensuring intent consistency, agent interoperability, and robustness to uncertainty at each level.

% -------------------------
% References
% -------------------------
\clearpage
\bibliographystyle{assets/plainnat}
\bibliography{bibliography}

% \bibliography{references}

% -------------------------
% Supplementary
% -------------------------
\newpage
% ===================================================================
% SECTION 9: Interpreter Agent: Responsibilities, Design, Implementation
% ===================================================================
\section{Interpreter Agent: Responsibilities, Design, Implementation}
\label{appendix:interpreter_interpreter_agent_details}

% ===================================================================
\subsection{Scope and Responsibilities}
% ===================================================================

The \emph{interpreter agent} is the gateway from high-level intent to optimization-ready control. It fulfills two primary responsibilities: (1) translating intents expressed in natural language into an initial structured \ac{OTM}; and (2) recursively reasoning over system observations and optimizer feedback to stabilize intent fulfillment by revising the \ac{OTM} when required (e.g., when constraints become infeasible).

\paragraph{Division of Labour (Dual-SLM).} To address these responsibilities under tight computational budgets, we employ two complementary \acp{SLM}: (1) a \textbf{fine-tuned SLM} for intent-to-\ac{OTM} translation; and (2) an \textbf{\ac{ICL} based SLM} for adaptive intent management, which reasons over structured prompts and windowed \ac{KPI} statistics, refines intent requirements when needed, and provides an explicit textual rationale.

While alternative realizations of an interpreter agent are possible, our design enables the use of lightweight \acp{SLM} that adhere to the compute and memory constraints of 4G/5G \ac{RAN} systems (see~\Cref{appendix:A8_complexity}).

% ===================================================================
\subsection{Architectural Overview}
% ===================================================================

The interpreter agent architecture, showed in~\Cref{fig:interpreter_dual_llm}, consists of four tightly coupled modules:

\begin{itemize}
    \item \textbf{Translator (\Cref{app:translator})} uses a fine-tuned \ac{SLM} to convert an incoming intent into a structured, machine-readable \ac{OTM} that specifies objectives, constraints, aggregation units, and provenance for different connectivity services and operational goals.
    \item \textbf{Monitor (\Cref{app:monitor})} subscribes to optimizer telemetry, aligns the telemetry stream to the \ac{OTM}-defined intent-management timescale, extracts per-window summaries, and bridges short gaps.
    \item \textbf{\ac{ICL}-based Advisor (\Cref{app:advisory})} uses an \ac{ICL}-based \ac{SLM} to reason over window summaries and active policy thresholds, selects an advisory direction \[a\in\{\texttt{increase}, \texttt{decrease}, \texttt{no\_change}\}\] and generates a compact rationale \(\mathcal{R}\) grounded in \ac{RRM}. It proposes only a direction, not a magnitude.
    \item \textbf{Adaptor (\Cref{app:controller})} converts the advisory action $a$ into a bounded threshold update $\Delta b$ under guardrails (e.g., caps, lifetime budget, floor/ceiling, cooldown), persists the updated threshold atomically into the \ac{OTM}, and emits an audit record.
\end{itemize}

During the intent-management loop, the \ac{OTM} is treated as a \emph{living document} jointly maintained by the interpreter and optimizer agents. The optimizer continuously solves against the current \ac{OTM} snapshot and reports telemetry (e.g., windowed KPI statistics) to the monitor. Guided by this feedback, the \ac{ICL}-based advisor recommends adjustments when intent requirements become overly tight or infeasible under the current network state. The adaptor then applies bounded updates to the corresponding \ac{OTM} constraints, yielding a refreshed \ac{OTM} for the optimizer.

\begin{figure}[!thb]
    \centering
    \includegraphics[width=0.9\linewidth]{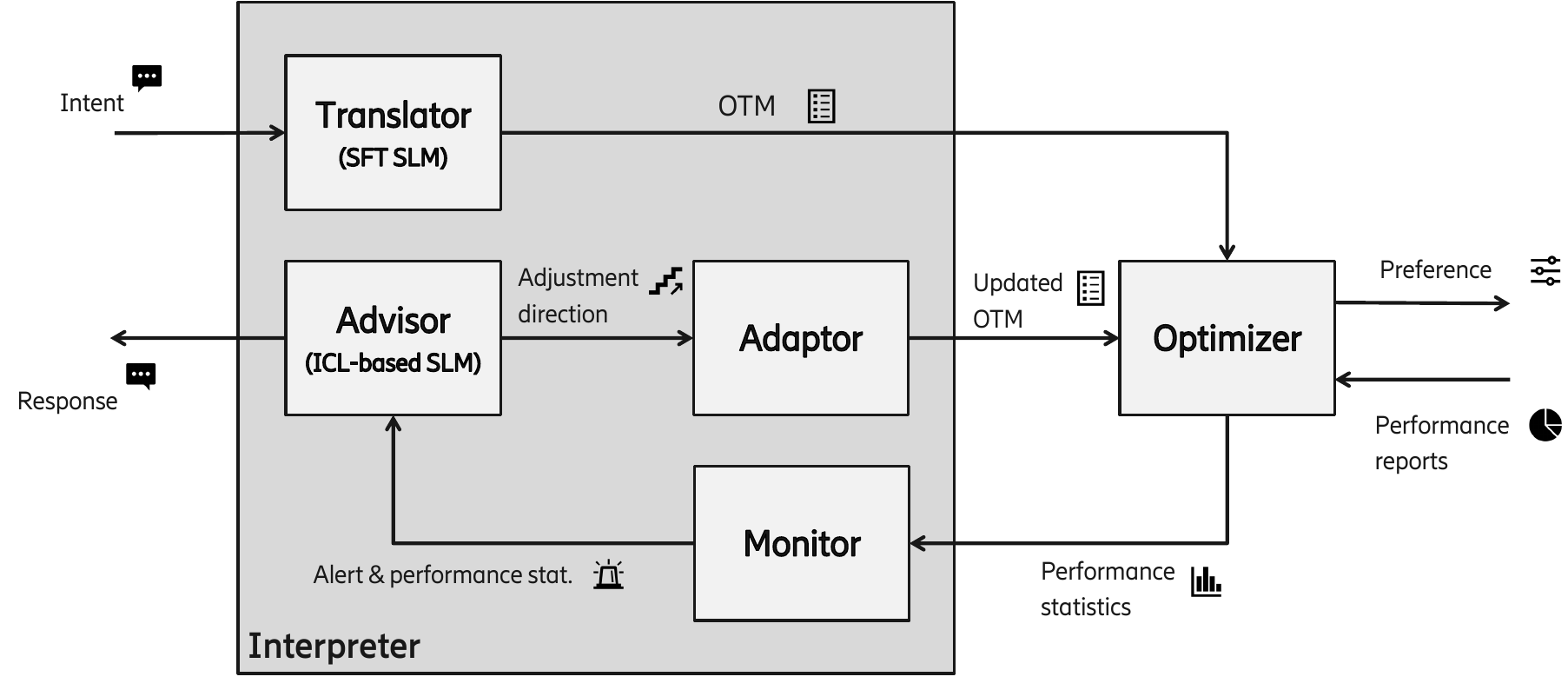}
    \caption{\textbf{Dual-SLM interpreter agent.}
    A (supervised) fine-tuned \ac{SLM} generates the \ac{OTM}; an intent monitor aligns telemetry; an \ac{ICL}-based advisory module outputs discrete adjustment directions with rationale; and an adaptor applies bounded updates and persists them atomically. The optimizer then solves against the latest \ac{OTM} snapshot, with telemetry closing the loop.}
    \label{fig:interpreter_dual_llm}
\end{figure}

% ===================================================================
\subsubsection{Translator}\label{app:translator}
% ===================================================================

The translator employs a fine-tuned \ac{SLM} to convert intents into deterministic, schema-compliant \ac{OTM} instances. Its role extends well beyond lexical parsing: it must interpret natural-language intents into meaningful optimization structures grounded in domain knowledge, and identify the appropriate downstream control agent to execute them. For example, a service intent requesting high reliability—such as the case in~\Cref{subsec:optimizer_controller_loop}—may translate into a non-obvious constraint formulated in terms of \ac{BLER}.

Further details on the translator design are provided in~\Cref{appendix:B:otm} and~\Cref{app:C:data_curation}, which discuss the \ac{OTM} schema and the supervised fine-tuning and evaluation of the translator \ac{SLM}, respectively.

% ===================================================================
\subsubsection{Sliding-Window Monitor}\label{app:monitor}
% ===================================================================

Consider a single constraint $k^A \!=\! \langle \texttt{kpi},\, \odot,\,b ,\, A,\, \texttt{unit}\rangle$ of an \ac{OTM}, where $\odot\in\{\le,\ge\}$ and $A$ is the per-step aggregation operator declared in the \ac{OTM} (e.g., \texttt{mean}, \texttt{min}, \texttt{max}, \texttt{p95}). 
To simplify notation, we refer to the constraint function $k^A(\cdot)$ as $y(\cdot)$, and let $y_t$ denote the \ac{KPI} value at step $t$ \emph{after} applying $A$ over the telemetry bin of length $\Delta$ (e.g., \qty{10}{\second}). With window length $W$, the monitor maintains a ring buffer over $\{y_i\}_{i=t-W+1}^{t}$ and computes a \emph{signed margin}:
\[
m_i \;=\; s(y)\,\bigl(y_i - b\bigr)\qquad \text{where}\qquad
s(y) \;=\; 
\begin{cases}
+1, & \odot \in \{\ge\} \quad\text{(lower bound)}\\
-1, & \odot \in \{\le\} \quad\text{(upper bound)}.
\end{cases}
\]
A step \(i\) is a violation if and only if \(m_i<0\) (negative margin). The window statistics are then
\[
\mathrm{violation\_ratio}(t) = \frac{1}{W}\sum_{i=t-W+1}^{t}\mathbf{1}[m_i<0],
\]
\[
\bar{y} = \frac{1}{W}\sum_{i=t-W+1}^{t}y_i, \qquad
y_{\min} = \min_{i}y_i, \qquad
y_{\max} = \max_{i}y_i,
\]
and the average \emph{shortfall/slack} (useful for controllers and prompts):
\begin{align*}
    \mathrm{shortfall\_avg}\;=\;&\frac{1}{W}\sum_{i=t-W+1}^{t}\max\{0,-m_i\}, \\
    \mathrm{slack\_avg}\;=\;&\frac{1}{W}\sum_{i=t-W+1}^{t}\max\{0, m_i\}.  
\end{align*}

\paragraph{Hysteresis and alerting.} Hysteresis prevents chattering: An \textsc{ALERT\_START} event is declared when \(\mathrm{VR}>\rho_{\text{on}}\) and an \textsc{ALERT\_END} event when 
\(\mathrm{VR}<\rho_{\text{off}}\) with \(\rho_{\text{on}}>\rho_{\text{off}}\). At each window end (the decision point), if an alert is active, the monitor produces a compact, constraint-centric context:
\[
\begin{aligned}
\{ &&\\
    & \;\;\text{window}:\;\{W,\,t\!-\!W\!+\!1\ldots t,\,\mathrm{VR}\},\\%[2pt]
    & \;\;\text{constraint\_metric}:\; \{\bar{y},\,y_{\min},\,y_{\max},\,b,\, \mathrm{shortfall\_avg},\,\mathrm{slack\_avg},\,\texttt{unit}\},\\[2pt]
    &\;\;\text{constraint\_id}:\;\texttt{id},\\
 \}  &\\
\end{aligned}
\]
Optionally, the context may be augmented with domain-specific auxiliaries (e.g., \texttt{aux\_kpis}) if available.

\paragraph{Complexity.} The monitor executes in $O(1)$ time per step through the use of a fixed-size ring buffer and incremental summary updates, with no rescans required. Memory usage grows linearly with the window size, i.e., $O(W)$.

\begin{table}[!tp]
\centering
\caption{Key hyperparameters of the interpreter agent.}
\label{tab:hyperparams}
\begin{tabular}{llp{6cm}}
\toprule
\textbf{Symbol} & \textbf{Name} & \textbf{Description} \\ \midrule
$W$ & Window size & Number of samples used to compute moving averages and the violation ratio. \\
$b$ & Threshold & Current target value for the monitored \ac{KPI}. \\
$\rho_{\mathrm{on}}$ & Alert-on ratio & Violation ratio above which an alert episode is initiated. \\
$\rho_{\mathrm{off}}$ & Alert-off ratio & Violation ratio below which an alert episode is terminated. \\
$d$ & Step size & Base increment or decrement applied to threshold updates. \\
$g_{\uparrow}, g_{\downarrow}$ & Guardrail gains & Maximum upward or downward adjustment permitted per update. \\
$s_{\max}$ & Smoothing cap & Maximum smoothing applied across consecutive updates. \\
$B$ & Budget & Maximum number of updates allowed within a single alert episode. \\
$b_{\min}, b_{\max}$ & Bounds & Minimum and maximum permissible threshold values. \\
$C$ & Cooldown & Minimum number of steps that must elapse before another update can be applied. \\ 
\bottomrule
\end{tabular}
\end{table}

% ===================================================================
\subsubsection{Advisor (Advisory Layer)}\label{app:advisory}
% ===================================================================

The advisory layer determines the \emph{direction of adaptation} and supplies a textual justification $\mathcal{R}$. It does \emph{not} specify the magnitude of change. Two modes are supported:
\begin{enumerate} \iffalse[label=(\alph*)]\fi
\item \textbf{Rule-based.} Thresholds on summary statistics (e.g., violation ratio, mean deviation from the target, minimum deviation from the target, auxiliary posture indicators) determine an advisory action \(a\).
\item \textbf{ICL-based \ac{SLM}} A structured prompt encodes (i) the set of allowed actions, (ii) the decision policy, (iii) domain-specific guardrails, and (iv) a strict JSON output schema. The \ac{SLM} produces an advisory adjustment
\[\verb|{"action": "...", "justification": "..."}|\]
conditioned on the parsed telemetry payload from the intent monitor.
\end{enumerate}

\paragraph{Guardrails in Prompting.} Schema fidelity and reproducibility are enforced through:
\begin{itemize}
    \item[(a)] JSON-only outputs;
    \item[(b)] end-of-sentence token fences;
    \item[(c)] banned tokens (e.g., URLs, markdown code fences); and
    \item[(d)] near-deterministic decoding with low-variance sampling to avoid verbatim repetition while maintaining stability.
\end{itemize}
The justification must cite explicit numerical values extracted from the payload (e.g., target $b$, mean $\bar{y}$, minimum $y_{\min}$, violation ratio $\mathrm{VR}$) and must classify posture relative to a domain-specific auxiliary metric (e.g., “aggressive” vs. “conservative”).

\paragraph{Prompt Contract (Abridged).}
Allowed actions are $\{\texttt{increase},\texttt{decrease},\texttt{no\_change}\}$. The required output format is strictly JSON:
\[\texttt{\{"action":"...", "justification":"..."\}}.\]
The justification must reference the relevant statistics and the auxiliary posture label. Domain-specific instantiations (e.g., using \ac{BLER} as the auxiliary metric) appear in examples in~\Cref{sec:experiment}.

% ===================================================================
\subsubsection{Adaptor (Magnitude, Safety, Persistence)}
\label{app:controller}
% ===================================================================

Given an advisory action $a \in \{\texttt{increase}, \texttt{decrease}, \texttt{no\_change}\}$, the adaptor computes a candidate step size $\Delta b$ using a deadband $d$ and asymmetric gains $(g_{\uparrow}, g_{\downarrow})$:
\[
\Delta b =
\begin{cases}
g_{\downarrow}\max\!\big(0,(b-\bar{x})-d\big), & a=\texttt{decrease},\\[2pt]
g_{\uparrow}\max\!\big(0,(\bar{x}-b)-d\big), & a=\texttt{increase},\\[2pt]
0, & a=\texttt{no\_change}.
\end{cases}
\]
Safety guardrails limit the actuation:
\[
\Delta b \leftarrow \min\{\Delta b,\, s_{\max},\, B_{\text{left}},\, b-b_{\min},\, b_{\max}-b\},\qquad
b \leftarrow \mathrm{clip}(b \pm \Delta b,\, b_{\min},b_{\max}).
\]
Budgets and cooldown counters are updated after each actuation. Final thresholds are written atomically to the \ac{OTM}, ensuring that the optimizer and monitor operate on consistent snapshots. \Cref{sec:experiment} illustrates with concrete examples (e.g., throughput maximization with minimum guarantees per user, or bounds on \ac{BLER}) how this generic mechanism applies across \acp{KPI}.

% ===================================================================
\subsection{Algorithmic Summary and Interfaces}
% ===================================================================

The closed-loop operation of the interpreter agent—integrating monitoring, advisory, and adaptation—is summarized in~\Cref{alg:interpreter_agent}. The procedure shows how the agent detects constraint violations, issues advisory actions, and applies bounded adaptations under guardrails.

\begin{algorithm}[H]
\caption{Interpreter Agent (Monitor $\rightarrow$ Advisor $\rightarrow$ Adaptor)}
\begin{algorithmic}[1]
\State \textbf{Input:} window size $W$; thresholds $(b,\rho_{\mathrm{on}},\rho_{\mathrm{off}})$; guardrails $(d,g_{\uparrow},g_{\downarrow},s_{\max},B,b_{\min},b_{\max},C)$
\For{each step $t$}
  \State Push observation $y_t$ into ring buffer; update $(\bar{y},y_{\min},\mathrm{VR})$
  \If{$\mathrm{VR}>\rho_{\mathrm{on}}$ and not in alert} 
    \State Start episode; reset budget and cooldown
  \EndIf
  \If{in alert}
    \State Build parsed telemetry payload; select action $a$ via rules or \ac{ICL} \ac{SLM}; log rationale $\mathcal{R}$
    \If{$a \neq \texttt{no\_change}$ and cooldown expired and $B_{\text{left}}>0$}
      \State Compute $\Delta b$; apply guardrails; update $b$; persist \ac{OTM}; decrement budget; reset cooldown $C$
    \EndIf
    \If{$\mathrm{VR}<\rho_{\mathrm{off}}$} 
      \State End episode; log summary
    \EndIf
  \EndIf
\EndFor
\end{algorithmic}
\label{alg:interpreter_agent}
\end{algorithm}

% ===================================================================
\subsubsection{Interfaces}
% ===================================================================

\paragraph{(i) From Monitor to Advisor.}
Upon receiving telemetry from the optimizer, the intent monitor produces a compact summary aligned to the \ac{OTM} timescale. This parsed payload becomes the sole input to the \ac{ICL}-based advisory module. An example summary from our experiments is:
\begin{lstlisting}[language=json]
{
  "window":            {
                            "start": 1020, "end": 1139, "W": 12, "violation_ratio": 0.60
                        },
  "constraint_metric": {
                            "name":"throughput", 
                            "avg": 6.92, 
                            "min": 3.08, 
                            "monitor_threshold": 7.00, 
                            "unit":"Mbps"
                        },
  "radio_kpis":        {"bler": {"avg":0.14, "target_hint":0.10}}
}
\end{lstlisting}

\paragraph{(ii) From Advisor to Adaptor.}
The advisory module returns only an \emph{adjustment direction} along with a textual justification, both constrained by the \emph{current \ac{OTM}} used by the optimizer. It never proposes numeric magnitudes. Example output:
\begin{lstlisting}[language=json]
    {
      "action": "decrease",
      "justification": "relax to reduce MCS pressure and HARQ overhead."
    }
\end{lstlisting}

\paragraph{(iii) From Adaptor to \ac{OTM} (atomic).}
The adaptor converts the advisory direction into a bounded step \(\Delta b\), applies guardrails (e.g., clipping to \([b_{\min}, b_{\max}]\)), persists the updated threshold atomically, and records the rationale:
\begin{lstlisting}[language=json]
    {
      "kpi":                "throughput",
      "aggregation":        "min"
      "old_threshold":      7.00,
      "new_threshold":      6.92,
      "delta":              -0.08,
      "episode":            "alert_002",
      "rationale":          "VR=0.60; BLER aggressive"
    }
\end{lstlisting}

% ===================================================================
\subsection{Models}
% ===================================================================

\paragraph{Fine-tuned~\ac{SLM} (Intent-to-\ac{OTM}).} A domain-specialized causal \ac{SLM} is fine-tuned to generate \ac{OTM} JSON directly from natural-language intents. Training uses instruction-style pairs of the form (intent, \ac{OTM}) that adhere to domain schemas (objective, \ac{KPI}, operator, threshold). The model is evaluated using exact-match accuracy and schema validity. This component is implemented using the Qwen-2.5-7B-Instruct model~\citep{QWEN:25} with supervised fine-tuning; additional details are provided in~\Cref{app:C:data_curation}.

\paragraph{ICL-based~\ac{SLM} (Constraint Adaptation).} A general-purpose \ac{SLM}—also based on Qwen-2.5-7B-Instruct~\citep{QWEN:25} but without task-specific weight updates—is prompted with: (i) the allowed actions and guardrails, (ii) policy rules governing the violation ratio (VR) and \ac{KPI} slack/shortfall, (iii) \ac{BLER} posture rules with target hints, and (iv) a strict JSON schema. Outputs are assessed for schema validity, internal consistency (e.g., adherence to policy rules), and justification quality.

% ===================================================================
\subsection{Stability, Safety, and Complexity}\label{appendix:A8_complexity}
% ===================================================================

Guardrails constrain actuation by ensuring that the target parameter $b$ remains within the safe interval $[b_{\min},b_{\max}]$. A hysteresis mechanism further prevents rapid oscillations caused by frequent threshold updates. The computational overhead of the method is minimal: each control step requires constant time $O(1)$, and memory usage grows linearly with the window size $O(W)$. This design minimizes the impact on \ac{RAN} compute and memory resources.

To evaluate the practical performance of the agentic \ac{AI} system for intent management, we report the following metrics:
(i) reduction in violation ratio relative to baseline operation;
(ii) percentage of observation windows that request a change;
(iii) percentage of updates clipped by guardrails;
(iv) validity rate of JSON payloads against the schema;
(v) observed episode lengths; and
(vi) adaptation latency per update.

% ===================================================================
\subsection{Failure Modes and Mitigations}
% ===================================================================

Despite these safeguards, the system remains susceptible to several failure modes. The corresponding mitigation strategies are:
\begin{itemize}
  \item \textbf{Prompt sensitivity:} Malformed or ambiguous payloads may arise from language model outputs. This risk is mitigated through strict schema enforcement, exclusion of unsafe tokens, and regression testing on canonical telemetry payloads.  
  \item \textbf{Distribution shift:} Variations in traffic or channel conditions can create discrepancies between training and deployment distributions. The system addresses this through window normalization and by providing \ac{BLER} posture hints to the model. In extreme cases, the controller can revert to a rules-only mode to preserve stability.
  \item \textbf{Over-actuation:} Excessive threshold adjustments may cause oscillations or instability. To prevent this, the system enforces lifetime update budgets, per-step update caps, cooldown intervals, and explicit floor/ceiling bounds on \(b\). 
  \item \textbf{Explainability drift:} Generated rationales may deviate from the underlying numerical evidence. The advisory module \(\mathcal{R}\) must cite explicit numerical values, and all rationale cards are logged and checked against policy expectations to ensure traceability and consistency.  
\end{itemize}
This section outlines how the interpreter agent determines \emph{when to act}, \emph{how to act and why}, and \emph{to what extent to act}. These behaviors are realized through dual \acp{SLM}, classical control guardrails, and auditable \ac{OTM} persistence.

% ===================================================================
% SECTION 10: 
% ===================================================================
\clearpage
\section{Optimization Template Model}
\label{appendix:B:otm}

\paragraph{Purpose.} The \ac{OTM} defines the contract between the interpreter agent and the downstream optimizer. It (i) specifies the optimization \emph{objective} and the associated \emph{constraints}, including explicit units and aggregation semantics;
(ii) records provenance for auditability (\texttt{origin}, \texttt{modified\_by}); and
(iii) serves as a \emph{living document} that can be safely updated by the adaptor during execution.

\paragraph{Formal view.} Let $\mathcal{X}$ denote the optimizer’s decision space, and let $k(\cdot)$ be a network \ac{KPI} evaluated under an aggregation operator $A$ (e.g., mean, min, $p95$). We define an \ac{OTM} instance as
\begin{align} 
\max_{x\in\mathcal{X}} \; k_{\text{obj}}^{A_{\text{obj}}}(x)
\quad \text{s.t.} \quad
\forall i\in\{1, \cdots, m\}:\;
\begin{cases}
k_i^{A_i}(x) \le b_i & \text{if } \texttt{operator}\in\{\texttt{lt},\texttt{le}\}\\
k_i^{A_i}(x) \ge b_i & \text{if } \texttt{operator}\in\{\texttt{gt},\texttt{ge}\}
\end{cases}\label{eq:OTM}
\end{align}
where each constraint $i$ specifies \texttt{service}, \texttt{kpi}, \texttt{operator}, threshold $b_i$, aggregation $A_i$, units, and scope. In essence, this formulation revisits the optimization~(\ref{eq:optimization}) by rewriting the objective $f(\cdot)$ and the constraints $g_i(\cdot)$ in terms of a more generic \ac{KPI} construct $k(\cdot)$ used in the \ac{OTM} schema.

\subsection{OTM Schema and Domain Semantics}
The \ac{OTM} schema is a minimal versioned JSON contract comprising four blocks, \texttt{objective}, \texttt{constraints}, and \texttt{metadata}, \texttt{version}, characterizing the \ac{OTM} formalism in~\Cref{eq:OTM}.

\begin{lstlisting}[language=json, caption={Generic OTM schema applicable to different RAN control problems.},label={lst:otm_schema}]

{
  "objective": {
    "service": <service_name>,       // {"mbb", "urllc", "gaming", "streaming", slice, ...},  
    "kpi": <kpi_name>,               // {"throughput", "reliability", "latency", ...}
    "scope": <scope_name>,           // {"per_user","per_cell", "per_slice", ...}
    "aggregation": <aggr_name>,      // {"mean","min","max","p95","sum", ...}
    "unit": <unit_name>,             // {"Mbps","Gbps","ms","s","%", ...}
    "maximize": <value>,             // boolean value {true, false}
  },
  "constraints": [
    {   "id": <value>                // string value (e.g., "C1", "C2")
        "service": <service_name>,   // {"mbb", "urllc", "gaming", "streaming", ...},
        "kpi": <kpi_name>,           // {"throughput", "reliability", "latency", ...}
        "scope": <scope_name>,       // {"per_user","per_cell", "per_slice", ...}
        "aggregation": <aggr_name>,  // {"mean","min","max","p95","sum", ...}
        "unit": <unit_name>,         // {"Mbps","Gbps","ms","s","%", ...}
        "operator": <operator_type>, // {"lt","le","ge","gt"}
        "threshold": <value>,        // float expressed in "unit"
        "modified": <value>,         // boolean value {true, false}
    },
    {
        ...
    }
  ],
  "metadata": {
    "otm"{
        "id": <value>                // string value (e.g., "O1", "O2")
        "created_by": <model_id>,    // string value {"SFT_LLM", ...}
        "timestamp": <value>,        // formatted as iso-8601
        "timescale": <value>         // string with window value (e.g., "10s_window")
        }
    "episode: {
        "id": <value>                // string value (e.g., "E1", "E2")
        "episode_type: <type_name>   // {"alert", "alert_resolved", ...}
        "modified_by": <model_id>,   // string value {"ICL_LLM", ...}
        "timestamp": <value>         // formatted as iso-8601
        }
    "adaptation_log": []
  },
  "version": "1.0",
}
\end{lstlisting}
\noindent\footnotesize\emph{Note:} Listings include \texttt{//} comments for readability; they are illustrative and not strict JSON.\normalsize

This structure is simple, yet generic enough to accommodate a wide range of problems, from simple single-service policies and more complex multi-service optimization directives spanning typical mobile traffic types—e.g., URLLC, mMTC, streaming, web, gaming, and voice—each associated with domain-appropriate KPIs (e.g., reliability and latency for URLLC, throughput and jitter for gaming and streaming). Specifically, the \ac{OTM} blocks define:

\begin{itemize}
    \item \textbf{\texttt{objective}}: This block specifies the primary goal to optimize for a plurality of services, including \acp{KPI}, their aggregation level (\texttt{mean}/\texttt{min}/\texttt{max}/\texttt{p95}/\texttt{sum}), scope, unit, and optimization sense (i.e., maximize or minimize). 
    \item \textbf{\texttt{constraints}}: This block encodes optional service-specific \ac{KPI} bounds, each expressed as an inequality \(k^{A}\,\odot\,b\). To this end, it shares the same fields of the \texttt{objective} block, and additionally includes an \texttt{operator} (\texttt{lt}/\texttt{le}/\texttt{ge}/\texttt{gt}) the specifies the relation to a \texttt{threshold} \(b\) expressed in the stated \texttt{unit}. Optionally, it includes fields indicating modification to a service constraint (\texttt{modified}, \texttt{modified\_by}, \texttt{id}).
    \item\textbf{\texttt{metadata}}: This block records \ac{OTM} static in formation, such as provenance, time of origination, etc. and dynamic information related to the last epsode event that triggered a modification of the \ac{OTM} to the  and the aggregation \texttt{timescale}, a \texttt{timestamp}, an \texttt{episode} identifier, and an append-only \texttt{adaptation\_log}.
    \item \textbf{\texttt{version}}: specifies the \ac{OTM} version.
\end{itemize}

\paragraph{OTM fields semantics.}
The \ac{OTM} schema currently requires only 15 fields, some of which are common across the schema blocks:
\begin{itemize}
    \item \textbf{\texttt{service}}: Specifies a service class (e.g., \texttt{mbb}, \texttt{urllc}, \texttt{gaming}, \texttt{streaming}, \texttt{slice}).
    \item \textbf{\texttt{kpi}}: Indicates canonical \acp{KPI} key resolvable by both the telemetry layer and the optimizer (e.g., \texttt{throughput}, \texttt{reliability}, \texttt{latency}, \texttt{bler}).
    \item \textbf{\texttt{scope}}: Indicates a spatial or logical domain related to a \ac{KPI} scope (e.g.,  \texttt{per\_user}, \texttt{per\_user}, \texttt{per\_slice}, \texttt{per\_user\_group}, \texttt{per\_cell\_group}, etc.)
    \item \textbf{\texttt{aggregation}}: Indicates an operator defining how raw samples are aggregated to optimize, evaluate or compare a \ac{KPI} (e.g., \texttt{mean}, \texttt{min}, \texttt{max}, \texttt{sum}, \texttt{$p95$}, etc.).
    \item \textbf{\texttt{unit}}: Indicats the type of \texttt{unit} used for \ac{KPI} or a threshold value (e.g., \texttt{Mbps}, \texttt{Gbps}, \texttt{ms}, \texttt{s}, \texttt{\%}, etc.). 
    \item \textbf{\texttt{operator}}: Defines relational semantics in a constraint like $\ge$, $\leq$, $=$ etc. (e.g., \texttt{le}, \texttt{ge}, \texttt{ge}, \texttt{gt}, \texttt{eq} etc).
    \item \textbf{\texttt{threshold}}: Defines the threshold value $b_i$  associated to a constraint stated \texttt{unit}; for modified constraints, the value is updated atomically by the adaptor.
    \item \textbf{\texttt{maximize}}: Defines the direction of an optimization (can be \texttt{true} or \texttt{false}) 
    \item \textbf{\texttt{id}}: Indicate an identifies associated with \ac{OTM}, a constraint, an event, etc.
    \item \textbf{\texttt{episode\_type}}: Indicate the type of event that caused a revision of the \ac{OTM}.
     \item \textbf{\textbf{\texttt{created\_by}}}: Identifies the model or module that originated the \ac{OTM} 
    \item \textbf{\texttt{modified}}/\textbf{\texttt{modified\_by}}: Indicate whether an \ac{OTM} has been modified and by which model or module
    \item \textbf{\texttt{timestamp}}: Records events times, such as \ac{OTM} creation and modification...
    \item \textbf{\texttt{timescale}}: Indicates monitoring window
\end{itemize}
The field \textbf{\texttt{metadata.adaptation\_log}} is append-only and used to trace updates with $\langle$old, new, $\Delta$, rationale, episode, time$\rangle$. \Cref{app:OTM_schema_example} exemplifies how typical lexical descriptions of service goals or requirements are mapped into the \ac{OTM} schema fiels:

\begin{table}[htbp]
\centering
\scriptsize	
\begin{tabular}{l c c c c c c c}
\toprule
\textbf{Lexical description} &
\multicolumn{7}{c}{\textbf{\ac{OTM} schema field values}}\\
& \ac{KPI} & Unit & Aggregation & Scope & Maximize & Operator & Threshold \\
\midrule 
Maximize mean cell throughput & \texttt{throughput} & \texttt{Mbps} & \texttt{mean} & \texttt{per\_cell} & \texttt{true} & \texttt{--} & \texttt{--}\\
Minimum user rate above 7Mbps & \texttt{throughput} & \texttt{Mbps} & \texttt{min} & \texttt{per\_user} & -- & \texttt{ge} & 7\\
Mean users BLER smaller than 10\%&\texttt{bler} & \texttt{\%} & \texttt{mean}  & \texttt{per\_user}& -- & \texttt{le} & 10\\
95\%-tile users latency less than 10ms&\texttt{latency} & \texttt{ms} & \texttt{p95} & \texttt{per\_user} & -- & \texttt{le} & 10\\
\bottomrule
\end{tabular}
\caption{Examples of \ac{OTM} schema values for typical service definitions.}
\label{app:OTM_schema_example}
\end{table}

\paragraph{Lifecycle and updates.}
The fine-tuned \ac{LLM} creates the initial \ac{OTM} combining connectivity service intents and network operational intents. This includes verifying the \ac{OTM} schema validity prior to hand-off to the optimizer  through a set of rules: (i) All \acp{KPI} must declare units and aggregation; (ii) the \texttt{operator} must be consistent with \ac{KPI} directionality; (iii) thresholds must lie within domain bounds. During the execution, the \ac{ICL}-based advisor may propose an update direction with rationale based on telemetry data. The adaptor then computes thresholds adjustments $\Delta b$ under guardrails, updates the target threshold, and persists the new \ac{OTM} snapshot atomically. Each episode produces a versioned \ac{OTM} with a growing \texttt{adaptation\_log}.

\subsection{Example of OTM adaptation}

\Cref{lst:otm_multi_before} illustrates an \ac{OTM} produced by the fine-tuned LLM.
The \texttt{objective} is to maximise mean throughput (Mbps).
Three constraints are active: \textbf{C1} enforces mean BLER $\le 0.10$ over a per-cell window (unitless ratio); \textbf{C2} caps user-level latency at $20$\,ms using the $p95$ aggregator; and
\textbf{C3} requires a per-cell minimum user throughput of at least $b=7.00$\,Mbps. Provenance marks \textbf{C3} as \texttt{modified\_by: ICL\_LLM}, indicating that its threshold may be adjusted online. The \texttt{metadata} block specifies the aggregation timescale (\texttt{10s\_window}) and records a snapshot \texttt{timestamp}/\texttt{episode}.

\begin{lstlisting}[language=json,caption={Illustrative OTM with multiple constraints, before adaptation.},label={lst:otm_multi_before}]
{
  "version": "1.0",
  "objective": {
    "service": "mbb",
    "kpi": "throughput",
    "aggregation": "mean",
    "unit": "Mbps",
    "maximize": true
  },
  "constraints": [
    {
      "id": "C1",
      "service": "mbb",
      "kpi": "bler",
      "operator": "le",
      "threshold": 0.10,
      "aggregation": "mean",
      "unit": "",
      "scope": "per_cell_window",
      "origin": "fine_tuned_LLM"
    },
    {
      "id": "C2",
      "service": "mbb",
      "kpi": "latency_ms",
      "operator": "le",
      "threshold": 20,
      "aggregation": "p95",
      "unit": "ms",
      "scope": "per_user_window",
      "origin": "fine_tuned_LLM"
    },
    {
      "id": "C3",
      "service": "mbb",
      "kpi": "tpt_min_mbps",
      "operator": "ge",
      "threshold": 7.00,
      "aggregation": "min",
      "unit": "Mbps",
      "scope": "per_cell_window",
      "origin": "fine_tuned_LLM",
      "adapted_by": "ICL_LLM"
    }
  ],
  "metadata": {
    "timescale": "10s_window",
    "timestamp": "2025-09-22T10:20:00Z",
    "episode": "alert_001",
    "adaptation_log": []
  }
}
\end{lstlisting}

At runtime, the sliding-window monitor observes a violation ratio $\mathrm{VR}=0.60$ for \texttt{tpt\_min\_mbps}, together with $\overline{\mathrm{BLER}}=0.14$ (classified as \emph{aggressive} relative to the $0.10$ target). The \ac{ICL}-based advisory selects action \texttt{decrease}; the adaptor computes a clipped update $\Delta b=-0.08$ (subject to caps, budgets, and bounds) and persists the new threshold. \Cref{lst:otm_multi_after} shows the resulting \emph{living} \ac{OTM}: only \textbf{C3} changes ($b:7.00\rightarrow 6.92$\,Mbps), while \textbf{C1} and \textbf{C2} remain unchanged. An \texttt{adaptation\_log} entry documents the update with $\langle\texttt{old\_threshold},\texttt{new\_threshold},\Delta b,\texttt{rationale},\texttt{episode},\texttt{timestamp}\rangle$.

\begin{lstlisting}[language=json,caption={Same OTM after one adaptation of constraint C3.},label={lst:otm_multi_after}]
{
  "version": "1.0",
  "objective": {
    "service": "mbb",
    "kpi": "throughput",
    "aggregation": "mean",
    "unit": "Mbps",
    "maximize": true
  },
  "constraints": [
    {
      "id": "C1",
      "service": "mbb",
      "kpi": "bler",
      "operator": "le",
      "threshold": 0.10,
      "aggregation": "mean",
      "unit": "",
      "scope": "per_cell_window",
      "origin": "fine_tuned_LLM"
    },
    {
      "id": "C2",
      "service": "mbb",
      "kpi": "latency_ms",
      "operator": "le",
      "threshold": 20,
      "aggregation": "p95",
      "unit": "ms",
      "scope": "per_user_window",
      "origin": "fine_tuned_LLM"
    },
    {
      "id": "C3",
      "service": "mbb",
      "kpi": "tpt_min_mbps",
      "operator": "ge",
      "threshold": 6.92,
      "aggregation": "min",
      "unit": "Mbps",
      "scope": "per_cell_window",
      "origin": "fine_tuned_LLM",
      "adapted_by": "ICL_LLM"
    }
  ],
  "metadata": {
    "timescale": "10s_window",
    "timestamp": "2025-09-22T10:28:00Z",
    "episode": "alert_002",
    "adaptation_log": [
      {
        "id": "C3",
        "old_threshold": 7.00,
        "new_threshold": 6.92,
        "delta": -0.08,
        "rationale": "VR=0.60; avg=6.92<b=7.00; BLER posture aggressive; relax b to stabilize HARQ."
      }
    ]
  }
}
\end{lstlisting}

\paragraph{Design rationale.}
The \ac{OTM} is deliberately minimal (objective, constraints, metadata) yet extensible (aggregation, scope, provenance). This ensures interface stability across \ac{RAN} domains while enabling adaptive operation and full auditability of constraint updates.

% ===================================================================
% SECTION 11: Translator SLM Fine-Tuning
% ===================================================================
\section{Translator SLM Fine-Tuning}
\label{app:C:data_curation}

% ===================================================================
% SECTION 11.1: Dataset Curation
% ===================================================================
\subsection{Dataset Curation}

This section describes the methodology used to construct the supervised corpus for training the Intent-to-\ac{OTM} translator, together with a statistical characterization of the resulting dataset. The objective of the curation process is to create a corpus that captures the semantic breadth of natural-language \ac{QoS} intents encountered in operational networks while ensuring strict adherence to the \ac{OTM} schema required for structured policy generation. The design integrates domain knowledge from 5G/6G communication systems, \ac{QoS}-engineering practice, service semantics, and the linguistic variability typical of operator-to-system interactions. The final dataset comprises 90{,}000 samples derived from 30{,}000 distinct \ac{OTM} structures, each paired with three three paraphrased intent utterances.

The construction process is guided by four principles: \emph{schema consistency}, \emph{domain realism}, \emph{linguistic diversity}, and \emph{multi-service generality}. Every instance conforms to the prescribed \ac{OTM} JSON structure to eliminate structural ambiguity. Services, KPIs, thresholds, aggregation functions, and operator semantics are selected to reflect realistic \ac{RAN}-engineering practice rather than arbitrary sampling. Multiple paraphrases express the same underlying intent using different linguistic styles, while both single-service and multi-service formulations are included to reflect realistic optimization scenarios such as cross-slice coordination or heterogeneous multi-tenant workloads. Together, these principles ensure that the model learns not only syntactically correct outputs but also semantically grounded mappings aligned with operational decision-making.

Seven \acp{KPI} central to \ac{QoS} and \ac{QoE} optimization are represented: latency, packet delay budget, jitter, packet error rate, block error rate, throughput, and spectral efficiency. Each \ac{KPI} is characterized by its physical unit, optimization orientation (minimize or maximize), and a plausible operational range. Service-specific threshold distributions are used to maintain realism. URLLC thresholds, for example, are drawn from tight low-delay intervals consistent with ultra-reliable low-latency requirements; gaming jitter values are sampled from moderate-sensitivity ranges; and streaming throughput thresholds reflect bandwidth levels typical of video services. These calibrated ranges ensure that the model encounters thresholds reflective of actual RAN-optimization tasks rather than arbitrary numeric values.

Real-world \ac{QoS} requirements frequently rely on percentile-based performance metrics, and the dataset reflects this by including mean, minimum, maximum, and percentile aggregations from p25 to p99. Sampling is intentionally biased toward domain-appropriate usage: reliability-sensitive KPIs such as latency, jitter, and error rates predominantly use high percentiles (p95 or p99), whereas throughput-oriented \acp{KPI} typically rely on mean values. This probabilistic, domain-aware selection encourages the model to internalize the relationship between service reliability expectations and suitable aggregation choices. Constraint operators are chosen in accordance with \ac{KPI} orientation, with minimization \acp{KPI} paired with $``\leq"$ constraints and maximization \acp{KPI} paired with $``\geq"$. Semantically invalid combinations, such as lower bounds on error rates, are excluded to prevent the model from learning physically implausible relations.

Each \ac{OTM} instance is paired with three natural-language paraphrases produced from four stylistic registers: operator-style technical phrasing, 3GPP-inspired formal language, casual expressions, and terse imperative commands. These stylistic variants emulate the diverse ways in which human operators, analysts, and automated systems articulate \ac{QoS} intents. The paraphrases incorporate synonyms for \acp{KPI} and services, linguistic variations in percentile expressions, and syntactic diversity ranging from multi-sentence descriptions to compact directives. This controlled diversity promotes robustness to real-world phrasing while preserving semantic consistency across paraphrases.

To reflect realistic optimization scenarios in multi-slice and multi-tenant RAN deployments, a controlled fraction of \acp{OTM} include multi-service dependencies in which the optimization objective applies to one service while constraints reference another. These cases emulate common operational patterns such as managing cross-service interference or guaranteeing simultaneous user-experience requirements across heterogeneous traffic types. Their presence strengthens the model’s ability to process complex interdependencies and to generate coherent, jointly feasible policies.

All generated samples include metadata fields such as an ISO-8601 timestamp and an episode identifier. The episode field is fixed to “unspecified” to avoid introducing unintended temporal semantics while maintaining compatibility with future policy-orchestration workflows requiring contextual metadata.

The statistical structure of the corpus reflects these design choices. Each \ac{OTM} specifies one optimization objective and between zero and three constraints consistent with \ac{QoS}-engineering practices in 5G and 6G networks. Because each template is associated with three paraphrases, the full corpus contains 90{,}000 samples. Constraint cardinality follows a non-uniform distribution chosen to represent operator practice: approximately 10\% of \acp{OTM} contain no constraints, 45\% contain one, 30\% contain two, and 15\% contain three. Consequently, roughly 90\% of the corpus includes at least one constraint, and nearly half include multiple constraints. This distribution exposes the model to a broad range of multi-constraint optimization scenarios rather than biasing it toward oversimplified workflows.

Service representation spans eight canonical categories—gaming, streaming, web, messaging, URLLC, mMTC, VoLTE, and VoIP. Sampling is intentionally skewed toward services with stringent \ac{QoS} requirements. URLLC and gaming each account for approximately 20–25\% of \acp{OTM}, streaming contributes about 15\%, voice (VoLTE and VoIP combined) contributes roughly another 15\%, and the remainder corresponds to web, messaging, and mMTC use cases. This distribution ensures adequate coverage across both throughput-oriented and latency-critical traffic classes.

\ac{KPI} coverage is similarly broad: all seven \acp{KPI} appear throughout the dataset following the service-specific domain profiles described above. Over 95\% of reliability-related constraints use high-percentile aggregations (p90–p99), preserving realism in statistical \ac{QoS} modeling. Approximately 12–18\% of \ac{OTM} instances include multi-service dependencies, providing inductive signals for joint optimization patterns common in next-generation RAN automation. Finally, linguistic variation reflects the stylistic sampling weights: operator style (40\%), 3GPP-inspired formal expressions (30\%), casual phrasing (20\%), and terse directives (10\%). This variation enhances generalization to heterogeneous real-world intent expressions while maintaining semantic consistency across samples.

Collectively, the dataset provides extensive coverage of service semantics, KPI behavior, constraint types, and linguistic variation. Its strict structural consistency, calibrated numerical modeling, and broad paraphrastic diversity make it well suited for supervised fine-tuning of models tasked with translating diverse natural-language intents into precise, schema-compliant \ac{OTM} structures aligned with \ac{RAN}-optimization practice.

% ===================================================================
% SECTION 11.2: Training Methodology
% ===================================================================
\subsection{Training Methodology}
\label{app:training-methodology}

The Intent-to-\ac{OTM} translator is trained using supervised fine-tuning on the curated corpus described in~\Cref{app:C:data_curation}. The task is formulated as a conditional sequence-generation problem: given a natural-language intent and a fixed system prompt, the model must produce a complete and structurally valid \ac{OTM} in JSON format. Since the mapping between intents and \acp{OTM} is deterministic and schema-constrained, the training objective emphasizes exact reproduction of field names, values, ordering, and hierarchical structure.

A transformer-based, instruction-tuned model (Qwen-2.5-7B-Instruct~\cite{QWEN:25}) serves as the underlying architecture. Fine-tuning proceeds in a left-to-right autoregressive manner in which each token is generated conditioned on both the input intent and the previously generated output. This preserves the strengths of the pre-trained model while enabling specialization toward domain-specific reasoning over services, \acp{KPI}, and \ac{QoS} constraints.

The optimization pipeline follows established practices for adapting large language models. Stability and parameter-efficiency mechanisms are incorporated, and the specific optimizer settings, learning-rate schedule, and adaptation configuration are reported in~\Cref{tab:sft-hyperparams} of~\Cref{appendix:hyperparams}. These components ensure reliable convergence when generating long, nested JSON structures that are sensitive to single-token variations.

Fine-tuning is performed for a small number of epochs, as the deterministic target format and internal consistency of the dataset enable rapid convergence without significant risk of overfitting. A held-out validation set is used to monitor generalization performance and to detect potential memorization of stylistic artifacts present in the synthetic paraphrases.

Model evaluation combines syntactic and semantic criteria. Token-level accuracy measures fidelity to the target JSON sequence, while a schema-validity check verifies exact compliance with the required \ac{OTM} specification. In addition, semantic alignment metrics assess whether the model correctly identifies the optimization objective, reproduces the appropriate constraints, and matches the complete ground-truth \ac{OTM}. Together, these metrics provide a comprehensive assessment of structured intent translation accuracy.

\begin{table}[th!]
\centering
\caption{Comparison of Evaluation Metrics Between the Fine-Tuned and Baseline Models}
\label{tab:model_comparison}
\renewcommand{\arraystretch}{1.25}
\begin{tabular}{lcc}
\hline
\textbf{Metric} & \textbf{Fine-Tuned Model} & \textbf{Baseline Model} \\
\hline
Total evaluation examples 
    & 1000 
    & 1000 \\
JSON valid rate 
    & 1.000 
    & 1.000 \\
Objective match rate 
    & \textbf{1.000} 
    & 0.450 \\
Constraints match rate 
    & \textbf{0.98} 
    & 0.215 \\
Full \ac{OTM} match rate 
    & \textbf{0.98} 
    & 0.113 \\
Number of constraint mismatches 
    & 20
    & 785 \\
Number of objective mismatches 
    & 0 
    & 550 \\
Full-match examples 
    & 980 
    & 113 \\
\hline
\end{tabular}
\end{table}

% ===================================================================
% SECTION 11.3: Comparative Evaluation
% ===================================================================
\subsection{Comparative Evaluation}

This section presents a comparative evaluation of two models for translating natural-language \ac{QoS} intents into structured \ac{OTM} representations. The \textit{fine-tuned model} is a supervised LoRA-adapted Qwen-2.5-7B-Instruct trained on a curated corpus of 90{,}000 intent--\ac{OTM} pairs (see~\Cref{app:training-methodology}), whereas the \textit{baseline model} is the unmodified Qwen-2.5-7B-Instruct. Both models were evaluated on a held-out test set of 1{,}000 examples under a strict schema-constrained matching protocol.

As summarized in~\Cref{tab:model_comparison}, the fine-tuned model demonstrates near-perfect structural and semantic adherence to the \ac{OTM} schema, achieving 100\% JSON validity, 100\% objective correctness, and a 98\% full \ac{OTM} match rate. In contrast, the baseline model---although also producing syntactically valid JSON in every case---achieves only 45\% objective match and 11.3\% full \ac{OTM} match. These results indicate that prompt-only use of a generic instruction-tuned model is insufficient for reliable schema-grounded semantic parsing.

The most pronounced disparity appears in the reconstruction of constraint sets. The fine-tuned model correctly predicts the complete constraint set---including KPI type, operator, aggregation level, threshold, and service scope---in 98\% of cases. The baseline model succeeds in only 21.5\% of examples, frequently selecting incorrect \acp{KPI}, operators, or units despite emitting valid JSON. As shown in~\Cref{tab:error-comparison}, the baseline model exhibits severe error inflation, particularly in the ``missing constraint’’ (724 occurrences) and ``extra constraint’’ (737 occurrences) categories. It is worth noting that a single erroneous sample may contribute to multiple error categories simultaneously. In contrast, the fine-tuned model’s residual errors remain modest and are concentrated primarily in scope mismatches (50\% of erroneous samples), with all other categories occurring infrequently.

Two complementary mechanisms explain the fine-tuned model’s performance advantage. First, supervised adaptation aligns the model’s internal representation with the deterministic structure of the \ac{OTM} schema. Although the baseline model possesses broad linguistic and domain knowledge, it lacks incentives to prioritize schema-specific conventions such as canonical KPI naming, service--KPI associations, valid threshold ranges, and operator semantics (e.g., $\leq$ for reliability-oriented \acp{KPI}). Fine-tuning effectively anchors the model’s output distribution to the \ac{OTM} schema.

Second, the curated dataset encodes domain priors that are internalized during training. For example, URLLC latency values cluster in the 1--10~ms range; gaming intents commonly include jitter constraints with p95 aggregation; streaming intents typically optimize for throughput; and error-rate \acp{KPI} almost always appear with $\leq$ operators. Lacking these priors, the baseline model frequently produces semantically plausible but non-canonical KPI selections, percentile aggregations, or threshold values.

Importantly, the fine-tuned model produces no objective-field errors (0 occurrences in~\Cref{tab:error-comparison}), indicating complete mastery of service--objective alignment. Its remaining errors occur almost exclusively within the constraint block and are typically minor: substituting p90 for p95, small numerical deviations around thresholds, or occasional variations in service scope. These contrast sharply with the baseline model’s structurally inconsistent outputs, which reflect a lack of grounding in the semantics encoded by the \ac{OTM} schema.

Taken together,~\Cref{tab:model_comparison,tab:error-comparison} show that supervised, domain-specific fine-tuning is essential for this task. Although both models generate syntactically valid JSON, only the fine-tuned model functions as a high-precision compiler from natural-language intents to machine-interpretable \ac{OTM} structures. For practical deployment in autonomous RAN-optimization pipelines, relying solely on prompting a pretrained instruction-tuned model is insufficient; structured domain adaptation is required to ensure correctness, robustness, and operational safety.

\begin{table}[t]
\centering
\caption{Comparison of Error Categories Between the Fine-Tuned and Baseline Models}
\label{tab:error-comparison}
\renewcommand{\arraystretch}{1.25}
\begin{tabular}{lcccc}
\toprule
\multirow{2}{*}{\textbf{Error Category}} 
    & \multicolumn{2}{c}{\textbf{Fine-Tuned Model}}
    & \multicolumn{2}{c}{\textbf{Baseline Model}} \\
\cmidrule(lr){2-3} \cmidrule(lr){4-5}
& \textbf{Count} & \textbf{Percentage}
& \textbf{Count} & \textbf{Percentage} \\
\midrule
Scope mismatch          & 10  & 50\%   & 76   & 9.7\% \\
Missing constraint      & 4   & 25\%   & 724  & 92.2\% \\
Extra constraint        & 5   & 20\%   & 737  & 93.9\% \\
Threshold mismatch      & 2   & 10\%   & 60   & 7.6\%  \\
Operator mismatch       & 2   & 10\%   & 68   & 8.7\%  \\
Aggregation mismatch    & 2   & 10\%   & 20   & 2.5\%  \\
Count mismatch          & 1   & 5\%    & 72   & 9.2\%  \\
\bottomrule
\end{tabular}

\vspace{2mm}
\footnotesize
\end{table}

% ===================================================================
% SECTION 12: Optimizer Agent Design
% ===================================================================

%-------------------------------------------------------------------------
\clearpage
\section{Optimizer Agent Design}
\label{appendix:Bayesian_optim}
%-------------------------------------------------------------------------

%-------------------------------------------------------------------------
\subsection{Bayesian Optimization} 
%-------------------------------------------------------------------------
We consider the problem of minimizing an unknown objective function $f : \mathcal{X} \to \mathbb{R}$, over a domain $\mathcal{X}\subseteq \mathbb{R}^d$. The goal is to identify  
\[
x^{\star} = \arg \min_{x \in \mathcal{X}} f(x).
\]  
In classical numerical optimization, one distinguishes between \emph{global optimization}, where the absolute minimum of $f$ is sought, and \emph{local optimization}, where the search is restricted to neighborhoods of an initial point. If $f$ is convex and $\mathcal{X}$ is a convex set, the problem reduces to convex optimization, for which efficient algorithms exist.  

However, in many modern applications, $f$ is a \emph{black-box function}, meaning that its closed-form expression is unavailable and evaluations are costly (e.g., expensive simulations or training machine learning models). In such cases, \emph{Bayesian optimization (BO)} provides an efficient framework for global optimization by maintaining a probabilistic model of $f$ and selecting query points via a surrogate criterion known as an \emph{acquisition function}.

\subsubsection{Gaussian Process Priors}

Bayesian optimization typically employs a \emph{Gaussian process (GP)} prior to model the unknown function $f$. A GP is defined by a mean function $\mu(x)$ and covariance function $K(x, x')$:  
\[
p(f) = \mathcal{GP}(f; \mu, K).
\]  
Given a set of noiseless observations 
\[
\mathcal{D} = \{(x_i, f(x_i))\}_{i=1}^n,
\]  
the posterior distribution over $f$ is again a GP with updated mean and covariance functions $\mu_{f|\mathcal{D}}(x)$ and $K_{f|\mathcal{D}}(x,x')$. This posterior provides both a predictive mean (exploitation) and predictive uncertainty (exploration), which form the basis of acquisition functions.

\subsubsection{Acquisition Functions}

An \emph{acquisition function} $a(x)$ encodes the utility of evaluating $f$ at a candidate point $x$. Since acquisition functions are cheap to evaluate, the optimization problem is reduced to  
\[
x_{t+1} = \arg \max_{x \in \mathcal{X}} a(x),
\]  
or its minimization equivalent. Common acquisition functions include:

\paragraph{\Ac{PI}:}  
Selects points with the highest probability of improving upon the best observed value $f'$:  
\[
a_{\text{PI}}(x) = \Phi\!\left(\frac{f' - \mu(x)}{\sigma(x)}\right),
\]  
where $\Phi$ is the Gaussian \ac{CDF}.

\paragraph{\Ac{EI}:}  
Accounts for the \emph{magnitude} of improvement:  
\[
a_{\text{EI}}(x) = (f' - \mu(x)) \, \Phi\!\left(\frac{f' - \mu(x)}{\sigma(x)}\right) 
+ \sigma(x) \, \phi\!\left(\frac{f' - \mu(x)}{\sigma(x)}\right),
\]  
where $\phi$ is the Gaussian \ac{PDF}. This criterion naturally balances exploration ($\sigma(x)$) and exploitation ($\mu(x)$).

\paragraph{\Ac{ES}}  
Reduces uncertainty about the optimizer’s location by minimizing the entropy of the distribution $p(x^\ast|\mathcal{D})$. While analytically intractable, approximations make this approach feasible.

\paragraph{\Ac{UCB}:}  
Promotes exploration via an optimism-in-the-face-of-uncertainty principle:  
\[
a_{\mathrm{UCB}}(x; \beta) = \mu(x) - \beta \sigma(x),
\]  
where $\beta > 0$ is a tunable trade-off parameter. Despite lacking an expected-utility interpretation, \ac{UCB} has strong theoretical guarantees for asymptotic convergence to the global optimum.

\medskip
This Bayesian decision-theoretic framework provides a principled way to trade off exploration and exploitation, making \ac{BO} a powerful tool for solving expensive black-box optimization problems.

%-------------------------------------------------------------------------
\subsection{PAX-BO: Preference-Aligned eXploration Bayesian Optimization}
%-------------------------------------------------------------------------

We have 

\begin{align*}
    \Delta^{d-1}=\{ \omega\in\mathbb{R}_{\geq 0}^{d} : \mathbf{1}^{\top}\omega = 1 \}
\end{align*}

% PAX-BO: Preference-Aligned, Constraint-Aware BO on the Simplex (with w/W notation)

\paragraph{Decision space and projection.}
We have $S$ services $\mathcal{S}=\{1,\dots,S\}$ and preference dimension $d$.
The internal (unconstrained) variable is
\[
U=\big[u^{(1)},\dots,u^{(S)}\big]\in\mathbb{R}^{d\times S},\qquad
\bar u=\mathrm{vec}(U)\in\mathbb{R}^{dS}.
\]
Each service $s$ uses a probability-simplex preference
\[
\mathbf{w}^{(s)} \in \Delta^{d-1}
\quad\text{with}\quad
\Delta^{d-1}=\{\mathbf{w}\in\mathbb{R}^d_{\ge 0}:\ \mathbf{1}^\top\mathbf{w}=1\},
\]
obtained by the columnwise Euclidean simplex projection
\[
\mathbf{w}^{(s)}=\Pi_\Delta\!\big(u^{(s)}\big),\qquad
\mathbf{W}(U)=\big[\mathbf{w}^{(1)},\dots,\mathbf{w}^{(S)}\big]\in(\Delta^{d-1})^S.
\]
Closed form for $\Pi_\Delta$: sort $u$ in descending order, find the threshold $\theta$, and set
$\Pi_\Delta(u)=\max(u-\theta\mathbf{1},0)$ with $\mathbf{1}^\top \Pi_\Delta(u)=1$.

\paragraph{Objective, constraints, and data.}
We optimize a single objective $f:(\Delta^{d-1})^S\to\mathbb{R}$ subject to scalar constraints
$\{g_i:(\Delta^{d-1})^S\to\mathbb{R}\}_{i=1}^p$. A configuration $U$ is feasible iff
\[
g_i\!\big(\mathbf{W}(U)\big)\le 0,\qquad i=1,\dots,p.
\]
At iteration $t$, we evaluate at $U_{t-1}$ and observe
\[
o_t=f\!\big(\mathbf{W}(U_{t-1})\big),\qquad
c_t^{(i)}=g_i\!\big(\mathbf{W}(U_{t-1})\big)\ (i=1,\dots,p),
\]
forming the dataset
\[
\mathcal{D}_t=\Big\{(\bar u_k,\;o_k,\;c_k^{(1)},\dots,c_k^{(p)})\Big\}_{k=1}^{t}.
\]

\paragraph{Surrogates and acquisition in $U$-space.}
Fit surrogates for the compositions
\[
\mathcal{F}(\bar u)\approx f\!\big(\mathbf{W}(U)\big),\qquad
\mathcal{G}_i(\bar u)\approx g_i\!\big(\mathbf{W}(U)\big)\ (i=1,\dots,p).
\]
Let the incumbent best feasible value be
\[
f_{t-1}^\star=\max\{\,o_k:\ c_k^{(i)}\le 0\ \forall i,\ k\le t-1\,\}\quad(\text{use }-\infty\text{ if none}).
\]
Define the acquisition on $\bar u$ (e.g., constrained Log-EI) by
\[
\alpha(\bar u)=\mathrm{ACQ}\!\left(\mathcal{F},\{\mathcal{G}_i\};\ \bar u,\ f_{t-1}^\star\right)
\approx \mathrm{LogEI}\big(\mu_{\mathcal{F}}(\bar u),\sigma_{\mathcal{F}}(\bar u); f_{t-1}^\star\big)
\times \prod_{i=1}^p \Phi\!\Big(-\mu_{\mathcal{G}_i}(\bar u)/\sigma_{\mathcal{G}_i}(\bar u)\Big).
\]

\paragraph{Single trust region (TR) in $U$-space.}
Maintain center $s_c\in\mathbb{R}^{dS}$ and radius $L>0$ (half side-length in $\ell_\infty$), giving the box
\[
\mathcal{B}_t=\big\{\bar v\in\mathbb{R}^{dS}:\ \|\bar v-s_c\|_\infty\le L\big\}.
\]
Let $\kappa_s,\kappa_f,\kappa_\ell$ be the success, failure, and “stuck-at-floor” counters.
Parameters: $L_0$ (initial), $L_{\min}$ (floor), $L_{\max}$ (cap); thresholds $s_{\mathrm{th}}$ (expand), $f_{\mathrm{th}}$ (shrink); tolerance $\epsilon>0$.

\paragraph{Local proposal (inner step).}
Choose the next internal point by maximizing $\alpha$ within the TR:
\[
\bar u_t\in\arg\max_{\bar v\in\mathcal{B}_t}\ \alpha(\bar v),\qquad
U_t=\mathrm{mat}(\bar u_t),\qquad
\mathbf{W}_t=\mathbf{W}(U_t)=\Pi_\Delta(U_t).
\]

\paragraph{Success test and TR adaptation.}
After executing $\bar u_{t-1}$, define the success flag
\[
\mathrm{SF}_t=\Big(\forall i:\ c_t^{(i)}\le 0\Big)\ \wedge\ \Big(o_t\ge f_{t-1}^\star+\epsilon\Big).
\]
If $\mathrm{SF}_t$:
\[
f_t^\star \leftarrow o_t,\quad s_c \leftarrow \bar u_{t-1},\quad
\kappa_s \leftarrow \kappa_s+1,\ \kappa_f\leftarrow 0,\ \kappa_\ell\leftarrow 0,
\]
and if $\kappa_s\ge s_{\mathrm{th}}$ then
\[
L \leftarrow \min(2L,L_{\max}),\qquad \kappa_s\leftarrow 0,\ \kappa_f\leftarrow 0.
\]
Else (failure):
\[
f_t^\star \leftarrow f_{t-1}^\star,\quad
\kappa_f \leftarrow \kappa_f+1,\ \kappa_s\leftarrow 0,
\]
and if $\kappa_f\ge f_{\mathrm{th}}$ then
\[
L \leftarrow \max(L/2,L_{\min}),\qquad
\kappa_s\leftarrow 0,\ \kappa_f\leftarrow 0,\qquad
\kappa_\ell \leftarrow \kappa_\ell + \mathbf{1}\{L=L_{\min}\}.
\]

\paragraph{Smart reset (escape when stuck).}
If $L=L_{\min}$ and $\kappa_\ell\ge w$, draw $n$ candidates on the product simplex:
$\{\mathbf{W}^{(j)}\in(\Delta^{d-1})^S\}_{j=1}^n$ (e.g., Dirichlet/QMC per column), and set
$\bar u^{(j)}=\mathrm{vec}(\mathbf{W}^{(j)})$.
For each $j$ compute
\[
\alpha^{(j)}=\alpha(\bar u^{(j)}),\quad
P_{\mathrm{feas}}^{(j)}=\prod_{i=1}^p \Phi\!\Big(-\mu_{\mathcal{G}_i}(\bar u^{(j)})/\sigma_{\mathcal{G}_i}(\bar u^{(j)})\Big),\quad
d^{(j)}=\min_{k\le t}\ \big\|\mathbf{W}^{(j)}-\mathbf{W}(U_k)\big\|_F.
\]
Normalize $z\in\{\alpha,P_{\mathrm{feas}},d\}$ to $\tilde z^{(j)}\in[0,1]$ and score
\[
\mathrm{score}^{(j)}=\tilde\alpha^{(j)}\,\tilde P_{\mathrm{feas}}^{(j)}\,\big(\tilde d^{(j)}\big)^{\beta}.
\]
Choose $j^\star=\arg\max_j \mathrm{score}^{(j)}$ and reset
\[
s_c\leftarrow \bar u^{(j^\star)},\qquad L\leftarrow L_0,\qquad
\kappa_s,\kappa_f,\kappa_\ell\leftarrow 0.
\]

\paragraph{Action selection (vector-valued $Q$).}
Given state $s_t$ and requested service $\sigma_t$, act by preference-aligned scalarization:
\[
a_t\in\arg\max_{a\in\mathcal{A}}\ \langle Q(s_t,a),\ \mathbf{w}_t^{(\sigma_t)}\rangle,
\qquad \mathbf{w}_t^{(\sigma_t)}=\Pi_\Delta\!\big(u_t^{(\sigma_t)}\big).
\]

\paragraph{Remarks.}
(i) All modeling and optimization happens in the unconstrained $\bar u$-space; feasibility is enforced by the projection $\mathbf{W}(U)=\Pi_\Delta(U)$.
(ii) The single TR stabilizes steps; expand/shrink is governed by $(s_{\mathrm{th}},f_{\mathrm{th}})$ and improvement tolerance $\epsilon$.
(iii) The smart reset proposes diverse, high-acquisition, high-feasibility candidates directly on $(\Delta^{d-1})^S$.

\begin{algorithm}[!thb]
\footnotesize
\caption{PAX-BO — Preference-Aligned eXploration Bayesian Optimization}
\label{alg:paxbo}
\textbf{Objects.}
\begin{itemize}\itemsep2pt
\item Services $\mathcal{S}\triangleq\{1,\dots,S\}$, preference dimension $d$.
\item Internal (unconstrained) variables $U=[u^{(1)},\dots,u^{(S)}]\in\mathbb{R}^{d\times S}$; vectorization $\bar u=\mathrm{vec}(U)\in\mathbb{R}^{dS}$.
\item Columnwise simplex projection: $p^{(s)}=\Pi_\Delta(u^{(s)})\in\Delta^{d-1}$; stack $P(U)=[p^{(1)},\dots,p^{(S)}]\in(\Delta^{d-1})^S$.
\item Single objective $f:(\Delta^{d-1})^S\!\to\!\mathbb{R}$; constraints are \emph{scalar} functions $\{g_i:(\Delta^{d-1})^S\!\to\!\mathbb{R}\}_{i=1}^p$; feasible iff $g_i(P)\le 0\ \forall i$.
\item Surrogates (e.g., GPs) model compositions $\mathcal{F}(\bar u)\approx f(P(U))$ and $\mathcal{G}_i(\bar u)\approx g_i(P(U))$ for $i=1,\dots,p$.
\item Best feasible value $f_t^\star=\max\{o_j:\,g_i(P_j)\le 0\ \forall i\}$ (use $-\infty$ if none).
\item Acquisition on $U$: $\alpha(\bar u)=\mathrm{ACQ}\!\big(\mathcal{F},\{\mathcal{G}_i\}_{i=1}^p;\,\bar u,\,f_t^\star\big)$ (e.g., constrained Log-EI).
\end{itemize}

\textbf{Projection.} For $u\in\mathbb{R}^d$: $\Pi_\Delta(u)=\arg\min_{p\in\Delta^{d-1}}\|p-u\|_2$, closed form: sort/threshold; $p=\max(u-\theta\mathbf{1},0)$ with $\mathbf{1}^\top p=1$.

\textbf{Single trust region in $u$-space.}
Center $s_c\in\mathbb{R}^{dS}$, radius $L>0$, box $\mathcal{B}\triangleq\{\bar v:\ \|\bar v-s_c\|_\infty\le L\}$.
Counters: successes $\kappa_s$, failures $\kappa_f$, stuck-at-$L_{\min}$ $\kappa_\ell$.

\begin{algorithmic}[1]
\State \textbf{Given:} window $W$; trust–region radii $L_0$ (initial), $L_{\min}$ (shrink floor), $L_{\max}$ (expansion cap); thresholds $s_{\mathrm{th}}$ (successes to expand), $f_{\mathrm{th}}$ (failures to shrink); tolerance $\epsilon$; reset window $w$; candidate count $n$; diversity exponent $\beta$
\State \textbf{Init:} choose $U_0$; $P_0=\Pi_\Delta(U_0)$; $s_c\gets\mathrm{vec}(U_0)$; $L\gets L_0$; $\kappa_s\gets 0$; $\kappa_f\gets 0$; $\kappa_\ell\gets 0$; $\mathcal{D}_0\gets\emptyset$
\For{$t=1,2,\dots$}
  \State \textbf{Evaluate} at $\bar u_{t-1}$: $P_{t-1}=\Pi_\Delta(U_{t-1})$; observe $o_t=f(P_{t-1})$ and $c_t^{(i)}=g_i(P_{t-1})$ for $i=1,\dots,p$
  \State \textbf{Update data} $\mathcal{D}_t=\mathcal{D}_{t-1}\cup\{(\bar u_{t-1},o_t,(c_t^{(i)})_{i=1}^p)\}$; keep last $W$; refit $\mathcal{F},\{\mathcal{G}_i\}$ on $(\bar u,o,(c^{(i)}))$; compute $f_{t-1}^\star$
  \State \textbf{Suggest next $U$ (local TR maximization)}
    \[
      \bar u_t \in \arg\max_{\bar v\in \mathcal{B}}\ \alpha(\bar v),\qquad U_t=\mathrm{mat}(\bar u_t),\qquad P_t=\Pi_\Delta(U_t).
    \]
  \State \textbf{TR update} with success flag $\mathrm{SF}\leftarrow \Big(\forall i:\ c_t^{(i)}\le 0\Big)\ \wedge\ \Big(o_t\ge f_{t-1}^\star+\epsilon\Big)$
  \If{$\mathrm{SF}$}
    \State $f_t^\star\gets o_t$;\ $s_c\gets \bar u_{t-1}$;\ $\kappa_s\gets \kappa_s + 1$;\ $\kappa_f\gets 0$;\ $\kappa_\ell\gets 0$
    \If{$\kappa_s\ge s_{\mathrm{th}}$} $L\gets \min(2L,L_{\max})$; $\kappa_s\gets 0$; $\kappa_f\gets 0$ \EndIf
  \Else
    \State $\kappa_f\gets \kappa_f + 1$;\ $\kappa_s\gets 0$
    \If{$\kappa_f\ge f_{\mathrm{th}}$}
      \State $L\gets \max\!\big(L/2,\,L_{\min}\big)$; $\kappa_s\gets 0$; $\kappa_f\gets 0$; \If{$L=L_{\min}$} $\kappa_\ell\gets \kappa_\ell + 1$ \EndIf
    \EndIf
  \EndIf
  \State \textbf{Smart reset (if stuck).} If $L=L_{\min}$ and $\kappa_\ell\ge w$:
    \State \hskip1em Sample $n$ candidates $P^{(j)}\in(\Delta^{d-1})^S$ (Dirichlet/QMC per service); set $\bar u^{(j)}=\mathrm{vec}(P^{(j)})$
    \State \hskip1em For each $j$: compute $\alpha(\bar u^{(j)})$;\quad $P_{\mathrm{feas}}^{(j)}=\prod_{i=1}^p \Phi\!\Big(-\mu_{\mathcal{G}_i}(\bar u^{(j)})/\sigma_{\mathcal{G}_i}(\bar u^{(j)})\Big)$;
                     \quad $d^{(j)}=\min_{(\bar u_k,\cdot,\cdot)\in\mathcal{D}_t}\|P^{(j)}-P_k\|_F$
    \State \hskip1em Normalize $\tilde z=\dfrac{z-\min z}{\max z-\min z+\varepsilon}$ for $z\in\{\alpha,P_{\mathrm{feas}},d\}$;\quad $\mathrm{score}^{(j)}=\tilde\alpha^{(j)}\,\tilde P_{\mathrm{feas}}^{(j)}\,\big(\tilde d^{(j)}\big)^{\beta}$
    \State \hskip1em Set $s_c\gets \bar u^{(j^\star)}$ where $j^\star=\arg\max_j \mathrm{score}^{(j)}$; $L\gets L_0$; $\kappa_s\gets 0$; $\kappa_f\gets 0$; $\kappa_\ell\gets 0$
  \State \textbf{Action (vector-valued $Q$).} For state $s_t$ and requested service $\sigma_t$:
    \[
      a_t \in \arg\max_{a\in\mathcal{A}} \ \langle Q(s_t,a,p_t^{(\sigma_t)}),\ p_t^{(\sigma_t)}\rangle
      \quad\text{with}\quad p_t^{(\sigma_t)}=\Pi_\Delta\!\big(u_t^{(\sigma_t)}\big).
    \]
\EndFor
\end{algorithmic}
\end{algorithm}

% ===================================================================
% SECTION 13: Multi-Objective Reinforcement Learning
% ===================================================================
\clearpage
\section{Multi-Objective Reinforcement Learning}
\label{appendix:MORL_prel}

\subsection{Multi-Objective Markov Decision Process}\label{subsec:MOMDP}

A \ac{MOMDP} extends the traditional \ac{MDP} framework \citep{Puterman1994, SuttonBarto2018} by considering not just one, but multiple objectives, which may conflict with each other \citep{RVW+:13, YSN:19}. Within this framework, an agent seeks to optimize several reward functions simultaneously, each corresponding to a different objective. These objectives can either conflict or complement each other; thus, improvements in one may adversely affect another or contribute positively to shared goals. The primary goal of an \ac{MOMDP} is to derive a policy that achieves an optimal balance among multiple objectives. This trade-off is typically represented by a Pareto front comprising a set of optimal policies such that no policy can outperform another across all objectives simultaneously, thereby making them non-dominated \citep{RVW+:13, MoN:14}.

Formally, an \ac{MOMDP} is defined by the tuple $\left\langle \mathcal{S}, \mathcal{A}, \mathcal{P}, \vr, \Omega, f_{\omega}, \gamma\right\rangle$, where $\mathcal{S}$ denotes the state space, $\mathcal{A}$ the action space, $\mathcal{P}(s' \mid s,a)$ the transition probabilities, and $\gamma\in [0, 1)$ a discount factor. The vector $\vr = [r_{1}, r_{2}, \cdots, r_{m}]^{\top}$ represents the $m$-dimensional reward vector, and we assume the preference space $\Omega$ to be the standard $(m-1)$-dimensional simplex (probability simplex), defined~as
\begin{align}
\Delta^{m-1} = \left\{ \boldsymbol{\omega}\in\mathbb{R}^{m} : \sum_{i=1}^{m}\omega_{i}=1 ~ \textrm{and} ~ \omega_{i}\geq 0 ~ \textrm{for} ~ i =1, \cdots, m \right\}. \label{eqn:simplex}
\end{align}
Here, each preference vector $\boldsymbol{\omega}\in\Omega$ assigns a normalized non-negative weight to each objective, reflecting its relative importance. We focus on a class of \acp{MOMDP} with a \emph{linear preference function}, in which a scalarization function $f_{\omega}(\vr)=\boldsymbol{\omega}^{\top}\vr$ converts the reward vector into a scalar return using the preference vector $\boldsymbol{\omega}$~\citep{RVW+:13, Hayes:22}. The cumulative expected return under a policy $\pi$ is then given by
\[
\hat{\vr}^{\pi} = \mathbb{E}\!\left[\sum_{t=0}^{\infty} \gamma^{t} \vr(s_{t}, a_{t}) \;\middle|\; \pi \right].
\]

\begin{observation}
When the preference vector $\boldsymbol{\omega}\in\Omega$  is fixed, an \ac{MOMDP} reduces to a standard \ac{MDP} with scalar rewards.
\end{observation}

\begin{remark}[Interpretation of Linear Scalarization]
The scalarization function $f_{\omega}(\vr)=\boldsymbol{\omega}^{\top}\vr$ can be interpreted as an expectation. 
If $\vr = [r_1, r_2, \dots, r_m]^{\top}$ is the reward vector and 
$\boldsymbol{\omega} = [w_1, w_2, \dots, w_m]^{\top}$ is a weight vector with 
$\sum_{i=1}^{m} w_i = 1$ and $w_i \geq 0$, then
\[
r_s = \boldsymbol{\omega}^{\top}\vr = \sum_{i=1}^{m} w_i r_i 
     = \mathbb{E}_{i \sim \boldsymbol{\omega}}[r_i].
\]
Thus, $\boldsymbol{\omega}$ can be interpreted as a probability distribution over objectives, 
and the function $f_{\omega}$ corresponds to the expected reward under this distribution~\citep{RVW+:13}.
\end{remark}

\begin{example}[Multi-Objective Q-Learning]
In traditional single-objective Q-learning, the agent estimates a scalar Q-value for each state--action pair,
\[
Q(s,a) = \mathbb{E}\!\left[\sum_{t=0}^{\infty} \gamma^t r(s_t,a_t) \;\middle|\; s_0=s, a_0=a \right].
\]
In contrast, in multi-objective reinforcement learning, each action may yield a \emph{vector of rewards} 
corresponding to different objectives  \citep{RVW+:13, NNN:18}.

For example, in a self-driving car scenario, the agent may consider:
\begin{itemize}
    \item $r_{\text{speed}}$: how fast the car goes,
    \item $r_{\text{fuel}}$: fuel efficiency,
    \item $r_{\text{safety}}$: safety score.
\end{itemize}

Thus, the Q-value function becomes vector-valued:
\[
Q(s,a) = 
\begin{bmatrix}
Q_{\text{speed}}(s,a) \\
Q_{\text{fuel}}(s,a) \\
Q_{\text{safety}}(s,a)
\end{bmatrix} \in \mathbb{R}^{3}.
\]

Each decision-maker may have different trade-offs between these objectives, expressed as a \emph{preference vector} 
$\boldsymbol{\omega} \in \Omega = \Delta^{2}$. For instance, $\boldsymbol{\omega} = [0.7, 0.2, 0.1]^{\top}$ indicates $70\%$ priority on speed, $20\%$ on fuel efficiency, and $10\%$ on safety. The corresponding scalarized utility is given by $\boldsymbol{\omega}^{\top} Q(s,a)$.
\end{example}

%-------------------------------------------------------------------------
\subsection{Convex Coverage Set}
\label{subsec:CCS}
%-------------------------------------------------------------------------

In multi-objective optimization, the \emph{Pareto front} $\mathcal{F}^{\star}$ contains all Pareto optimal solutions \citep{RVW+:13, MoN:14}, defined as
$$
\mathcal{F}^{\star} \triangleq \left\{ \vr\in\mathbb{R}^{m} : \nexists\vr^{\prime}\in\mathbb{R}^{m} ~ \textrm{such that} ~ \vr_{i}^{\prime} \geq \vr_{i} ~ \textrm{for all $i$ and} ~ \vr_{j}^{\prime} > \vr_{j} ~ \textrm{for at least one} ~ j \right\} \;.
$$

However, not all Pareto-optimal points are relevant when preferences are restricted to linear scalarizations. In this case, only a subset of the Pareto front---the \emph{\ac{CCS}}---is sufficient \citep{RVW+:15}.

The \ac{CCS} is a subset of $\mathcal{F}^{\star}$ consisting of non-dominated solutions that are optimal under some linear preference vector. Mathematically, it is defined as
\begin{align*}
\mathcal{C} \triangleq \left\{\hat{\vr}\in\mathcal{F}^{\star} : \exists\boldsymbol{\omega}\in\Delta^{m-1} ~ \textrm{such that} ~ \boldsymbol{\omega}^{\top}\hat{\vr} \geq \boldsymbol{\omega}^{\top}\hat{\vr}^{\prime} ~ \textrm{for all} ~ \hat{\vr}^{\prime}\in\mathcal{F}^{\star} \right\}.
\end{align*}
Thus, the \ac{CCS} comprises those points on the outer convex boundary of $\mathcal{F}^{\star}$ that maximize utility for at least one preference vector $\boldsymbol{\omega}$ (with $\Omega = \Delta^{m-1}$ as defined above) \citep{RVW+:15}.

\begin{example}
Consider a bi-objective optimization problem with the objectives of maximizing accuracy and maximizing interpretability. The \ac{CCS} would include those points on the Pareto front that maximize a weighted sum of the two objectives for some given trade-off between them; these points lie on the convex outer boundary of the feasible set in the objective space.
\end{example}

Therefore, the \ac{CCS} represents the minimal subset of the Pareto front that guarantees optimality under some linear preference. It is particularly valuable in decision-making scenarios where preferences may vary, since it identifies exactly those solutions that are relevant for all possible linear trade-offs \citep{RVW+:13}.

%-------------------------------------------------------------------------
\subsection{Envelope Q-Learning}
\label{subsec:EQL}
%-------------------------------------------------------------------------

\textbf{Scope.} Learn a \emph{single} preference-conditioned action-value
function $Q:\mathcal{S}\times\mathcal{A}\times\Delta^{m-1}\to\sR^{m}$ such that, for any linear preference $\vomega\in\Delta^{m-1}$, the scalar projection $\vomega^{\top}Q(s,a,\vomega)$ equals the optimal scalarized value for acting under $\vomega$~\citep{YSN:19}. The induced policy is
\begin{equation}
\pi_{\omega}(s)=\arg\max_{a\in\mathcal{A}}\,\vomega^{\top}Q(s,a,\vomega;\vtheta).
\label{eq:eql-policy}
\end{equation}

\paragraph{Envelope maximizer selection (\ac{DDQN} style).}
Given a transition $(s,a,\vr,s')$ with reward vector $\vr(s,a)\in\sR^{m}$, \ac{EQL}~\citep{YSN:19} bootstraps from the \emph{envelope} at the next state by selecting the action–preference pair that maximizes the $\omega$-projection using the \emph{online} network (parameters $\vtheta$):
\begin{equation}
(a^{\star},\omega^{\star})
=\arg\max_{a'\in\mathcal{A},\,\omega'\in\Delta^{m-1}}
\omega^{\top}Q(s',a',\vomega';\theta).
\label{eq:eql-select}
\end{equation}
This couples preferences because $\vomega^{\star}$ is not necessarily equal to the current $\omega$.

\paragraph{Vector \ac{TD} target (envelope bootstrap).}
Evaluate the selected pair with the \emph{target} network $\theta^{-}$ to obtain
a \emph{vector} target
\begin{equation}
y \;=\; r(s,a) \;+\; \gamma\,Q\!\big(s',a^{\star},\omega^{\star};\theta^{-}\big)
\;\in\;\mathbb{R}^{m},
\label{eq:eql-target-vector}
\end{equation}
where the expectation over $s'\!\sim\mathcal{P}(\cdot\mid s,a)$ is approximated
by sampling $s'$ from the replay buffer.

% ---------------- Algorithm (updated to match the vector target/losses) -------------
\begin{algorithm}[!t]
\caption{Envelope Q-learning (EQL) with Preference-Guided Replay}
\label{alg:eql}
\begin{algorithmic}[1]
\Require Discount factor $\gamma$, prioritized buffer $\mathcal{B}$, target period $C$, Dirichlet prior $\alpha$, preferences-per-sample $K$, minibatch size $B$
\State Initialize $\theta$; set $\theta^- \gets \theta$
\Statex
\Comment{Interaction (\Cref{fig:eql-arch}, steps 1--5)}
\For{each environment step}
  \State observe $s$ \Comment{(1)}
  \State sample $\omega \sim \mathrm{Dir}(\alpha)$ \Comment{(2)}
  \State choose $a$ by $\varepsilon$-greedy on $\omega^\top Q(s,a,\omega;\theta)$ \Comment{(3)}
  \State execute $a$; observe $r,s'$ \Comment{(4)}
  \State push $(s,a,r,s')$ into $\mathcal{B}$ with initial priority $p_{\max}$ \Comment{(5)}
\EndFor
\Statex
\Comment{Learning (\Cref{fig:eql-arch}, steps 6--11)}
\For{each gradient step}
  \State sample $\{(s_i,a_i,r_i,s'_i)\}_{i=1}^B \sim \mathcal{B}$ by priority \Comment{(7)}
  \For{$i=1$ to $B$}
    \State sample $\mathcal{W}_i=\{\omega'_{ij}\}_{j=1}^K$ \Comment{(6), (8)}
    \State $(a_i^\star,\omega_i^\star) \gets \arg\max_{a',\,\omega'\in\mathcal{W}_i}\ \omega_i^\top Q(s'_i,a',\omega';\theta)$
    \State \textbf{Vector target:} $y_i \gets r_i + \gamma\,Q(s'_i,a_i^\star,\omega_i^\star;\theta^-)$ \Comment{\Cref{eq:eql-target-vector}}
    \State $\delta^{A}_i \gets y_i - Q(s_i,a_i,\omega_i;\theta)$ \Comment{vector TD}
    \State $\delta^{B}_i \gets \omega_i^\top y_i - \omega_i^\top Q(s_i,a_i,\omega_i;\theta)$ \Comment{scalarized TD}
  \EndFor
  \State minimize $\mathcal{L}(\theta)=\frac{1}{B}\sum_i \big[(1-\lambda)|\delta^{B}_i| + \lambda \|\delta^{A}_i\|_2^2\big]$ \Comment{\Crefrange{eq:eql-LA}{eq:eql-Lcomb}}
  \State update priorities in $\mathcal{B}$ using $\|\delta^{A}_i\|_1$ or $|\delta^{B}_i|$ \Comment{\Cref{eq:eql-priority}}
  \If{step mod $C = 0$} \State $\theta^- \gets \theta$ \EndIf
\EndFor
\end{algorithmic}
\end{algorithm}

\paragraph{Losses.}
The primary objective regresses the full vector target (cf. Equation~(6) in
\citep{YSN:19}):
\begin{equation}
\mathcal{L}^{A}(\theta)
=\mathbb{E}_{(s,a,r,s'),\,\omega}\!\left[
\big\|y-Q(s,a,\omega;\theta)\big\|_{2}^{2}\right].
\label{eq:eql-LA}
\end{equation}
Because the optimal frontier contains many discrete extreme points (a
nonsmooth landscape), an auxiliary \emph{scalarized} loss improves
optimization stability (cf. Equation~(7) in \citep{YSN:19}):
\begin{equation}
\mathcal{L}^{B}(\theta)
=\mathbb{E}_{(s,a,r,s'),\,\omega}\!\left[
\big\|\omega^{\top}y-\omega^{\top}Q(s,a,\omega;\theta)\big\|_{2}\right].
\label{eq:eql-LB}
\end{equation}
The neural network is trained by employing a simple homotopy:
\begin{equation}
\mathcal{L}(\theta)=(1-\lambda)\,\mathcal{L}^{B}(\theta)
+\lambda\,\mathcal{L}^{A}(\theta),\qquad \lambda\in[0,1],~\lambda\uparrow1.
\label{eq:eql-Lcomb}
\end{equation}

\paragraph{Approximating the inner maximization.}
The maximization over $\omega'\in\Delta^{m-1}$ in \Cref{eq:eql-select} is
approximated by sampling a small candidate set
$\mathcal{W}=\{\omega'_{j}\}_{j=1}^{K}\subset\Delta^{m-1}$ (e.g., from a
Dirichlet distribution) and computing
\begin{equation}
(a^{\star},\omega^{\star})
\approx\arg\max_{a'\in\mathcal{A},\,\omega'\in\mathcal{W}}
\omega^{\top}Q(s',a',\omega';\theta).
\label{eq:eql-approx}
\end{equation}
Each transition is \emph{relabeled} with multiple sampled preferences
(hindsight preference relabeling), which couples learning across the preference space and greatly improves sample efficiency. The complete training loop—with preference sampling, hindsight relabeling, prioritized replay, and the envelope bootstrap—is summarized in~\Cref{fig:eql-arch} and \Cref{alg:eql}.

\paragraph{Replay priority.}
Priorities can be derived from vector or scalarized TD errors, e.g.,
\begin{equation}
p\propto \big\|y-Q(s,a,\omega;\theta)\big\|_{1}
\quad\text{or}\quad
p\propto\big|\omega^{\top}y-\omega^{\top}Q(s,a,\omega;\theta)\big|.
\label{eq:eql-priority}
\end{equation}

\paragraph{Theory and intuition.}
The \emph{envelope Bellman operator} (induced by~\Cref{eq:eql-target-vector}) has a
unique fixed point and is a $\gamma$-contraction under a suitable metric; hence
EQL converges in tabular settings~\citep{YSN:19}. 

\paragraph{Geometric interpretation.}
The selection of the envelope maximizer, seen in \Cref{eq:eql-select}, \emph{backs up from} the upper convex hull of the next state returns.
Define
\[
\mathcal{V}(s') \triangleq \bigl\{\, Q(s',a',\omega') : a'\!\in\!\mathcal{A},\ \omega'\!\in\!\Omega \,\bigr\} \subset \mathbb{R}^{m}.
\]
Since $\omega^\top(\cdot)$ is linear, maximizing it over a set equals maximizing it over that set’s
convex hull (support-function invariance):
\begin{equation}
\max_{a',\,\omega'} \ \omega^\top Q(s',a',\omega')
\;=\;
\max_{v\in \operatorname{conv}\mathcal{V}(s')} \ \omega^\top v.
\label{eq:e3-support}
\end{equation}
Consequently, \Cref{eq:eql-select} selects a supporting extreme point of
$\operatorname{conv}\mathcal{V}(s')$, and the target vector in
\Cref{eq:eql-target-vector} \emph{bootstraps} from this convex envelope of the solution frontier; hence dominated trade-offs are not reinforced and EQL effectively
targets the convex coverage set (CCS).\footnote{With finite $\mathcal{A}$
(and discretized $\Omega$ in practice), $\mathcal{V}(s')$ is finite, so
$\operatorname{conv}\mathcal{V}(s')$ is a polytope and the maximum in~\Cref{eq:e3-support} is attained at a vertex.}

\paragraph{Adaptation.}
At runtime, the trained policy $\pi_{\omega}$ can be executed with any desired preference vector $\omega$ without retraining.

% ---------------- Figure caption (references updated) -------------------
\begin{figure}[!tbh]
  \centering
  \includegraphics[width=0.65\textwidth]{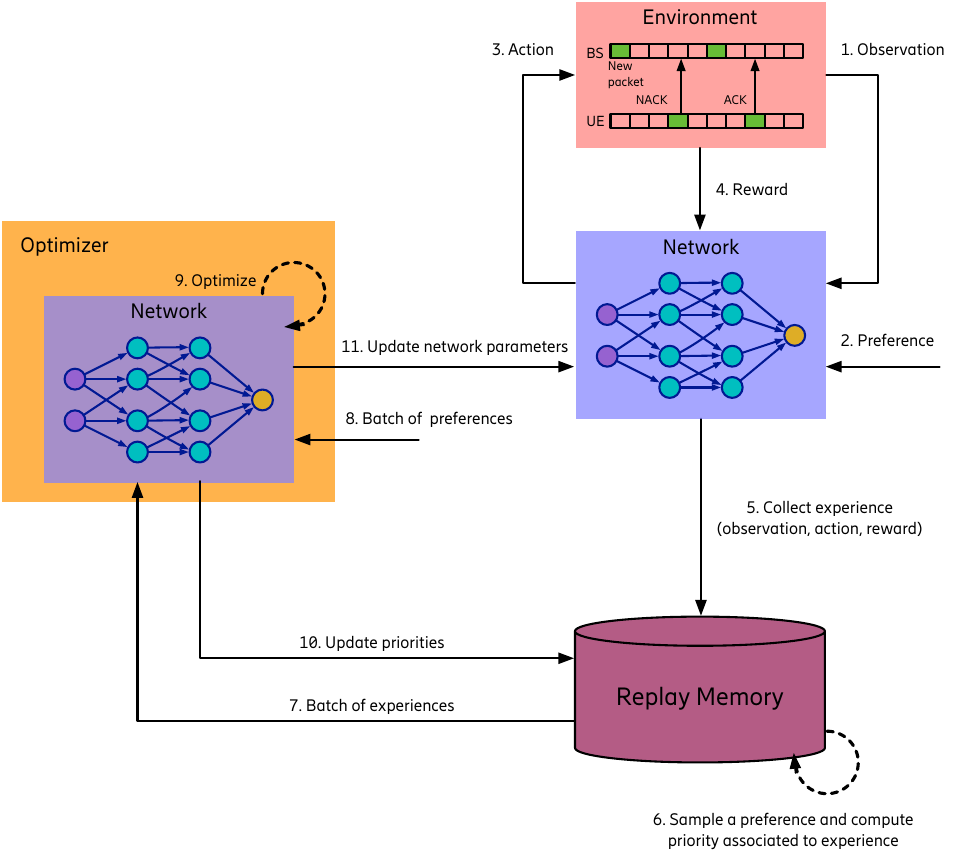}
  \caption{\textbf{Multi-objective RL with preference-guided optimization (EQL training loop).}
  \textbf{(1)} Observe state; \textbf{(2)} sample preference $\omega$; \textbf{(3)} act; \textbf{(4)} receive vector reward $r$; \textbf{(5)} store $(s,a,r,s')$; \textbf{(6)} memory samples auxiliary preferences and computes envelope-based priorities; \textbf{(7)} return prioritized batch; \textbf{(8)} return batch of preferences; \textbf{(9)} optimize using the vector target in~\Cref{eq:eql-target-vector} and losses~\Crefrange{eq:eql-LA}{eq:eql-Lcomb}; \textbf{(10)} update priorities; \textbf{(11)} update network parameters.}
  \label{fig:eql-arch}
\end{figure}

\paragraph{\ac{EQL} shortcomings:} Two design aspects limit the scalability of \ac{EQL}~\cite{YSN:19} in large state–action spaces. First, \ac{EQL} relies on a singleton architecture~\Cref{alg:eql}, where a single actor must explore the full joint state–action–preference space, leading to poor coverage and inefficient learning. Second, sample priorities are assigned only once at generation and never updated during training (unlike, e.g.,~\citep{Horgan:18}), which slows convergence. Since \ac{RAN} control problems involve vast state–action spaces, we propose a distributed \ac{EQL} variant where multiple actors share the exploration load that leads to improved coverage and performance; see~\Cref{appendixDMRL:DEQL_tests}.

% ===================================================================
% SECTION 14: Distributed Envelope Q-Learning
% ===================================================================
\clearpage
\section{Distributed Envelope Q-Learning}
\label{appendix:distributed_morl}

We propose \ac{D-EQL}, a distributed \ac{MORL} algorithm that extends \ac{EQL}~\citep{YSN:19} with an APE\hbox{-}X–style distributed architecture~\citep{Horgan:18} for faster and more efficient exploration of the preference space. As in vanilla \ac{EQL}, our method optimizes a single policy/value network over preferences for multiple competing objectives. Unlike the original setting, we employ learner–actor decoupling~\citep{Horgan:18} and distribute exploration over the preference space across multiple parallel actors. Specifically, we partition the preference simplex into subspaces and allocate different actors to explore different subspaces in parallel. While the partitioning of the preference space is inspired by~\citep{Xu:20}, distributing exploration across actors improves coverage and exploration efficiency. Furthermore, we employ distributed prioritized experience replay with hindsight to improve sample efficiency: prioritized replay selects the most informative experiences at each training step, while hindsight relabeling increases reuse by updating priorities under multiple preferences $\boldsymbol{\omega}$.

\subsection{D-EQL Architecture}
\label{subsec:architecture}

Our framework follows a scalable distributed (multi-objective) reinforcement learning architecture that decouples data collection, storage, and learning; see~\Cref{fig:distributed-arch}. A set of CPU-based actors, each running a replica of the policy network, interact in parallel with multiple simulation environments to generate trajectories of state, action, reward, and next state tuples. Actors generate experiences by exploring \emph{only} a partition the preference simplex. These experiences are first stored locally and then pushed to a sharded replay memory, where data is distributed either via load-balancing or fixed actor-to-shard mappings. Each shard operates as an independent replay buffer, enabling parallel writes and prioritized sampling. A GPU-based learner initially allocates a subspaces of the preference simplex to each actor for distributed exploration. The learner then periodically samples mini-batches from all shards, performs gradient updates on the policy network, and returns updated priorities to maintain efficient replay. Updated network weights are then broadcast to all actors, ensuring consistent synchronization across distributed processes. This design allows the system to scale efficiently with the number of actors and environments, achieving high-throughput of experience collection while preserving stability in training.

\begin{figure}[!thb]
    \centering
    \includegraphics[width=1.0\linewidth]{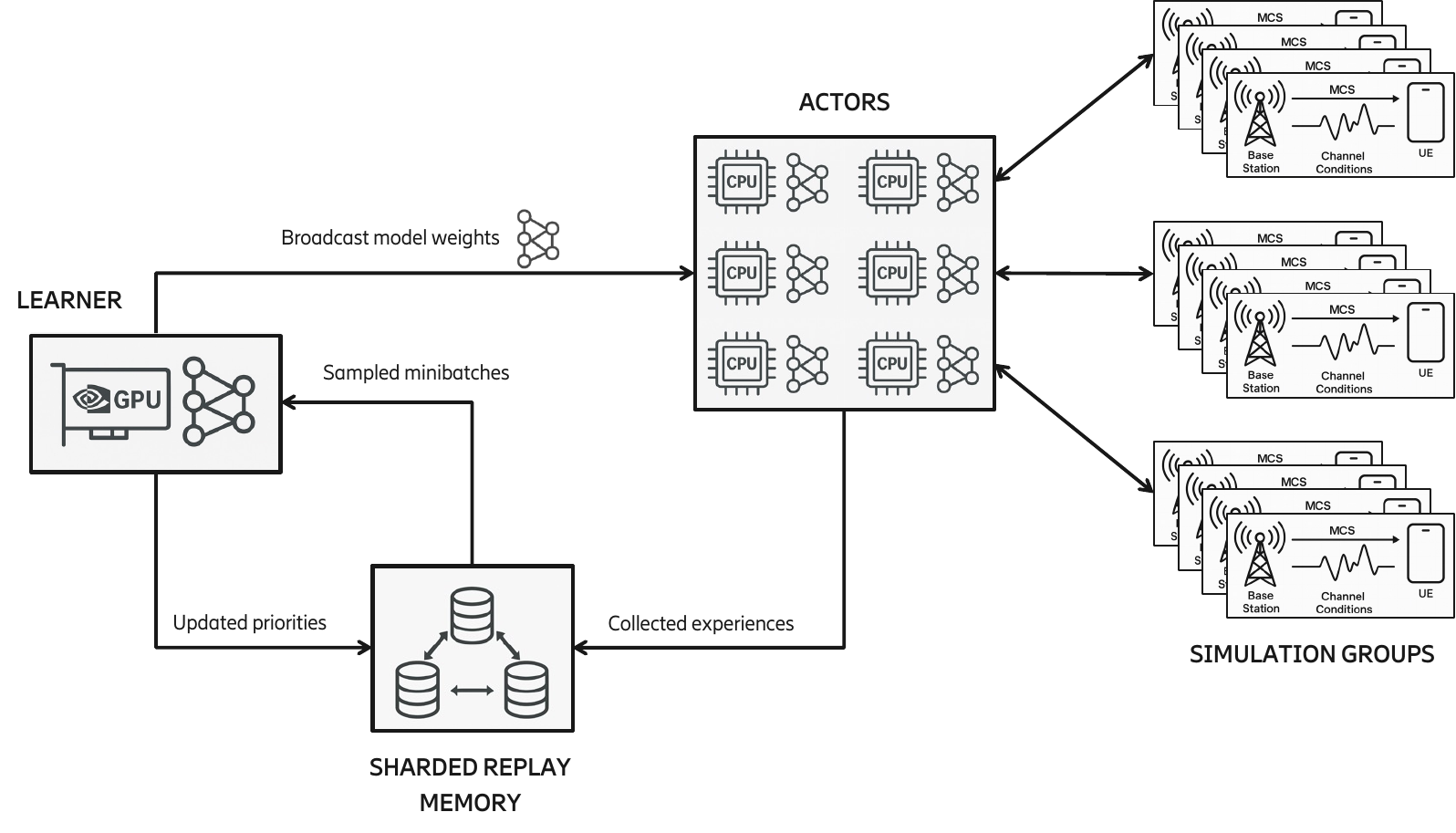}
    \caption{\textbf{Overview of the distributed (multi-objective) reinforcement learning architecture.} CPU-based actors interact in parallel with multiple simulation groups to generate trajectories, which are stored in a sharded replay memory, while exploring different subspaces of the preference simplex. A GPU-based learner samples mini-batches from the shards, updates the policy network, and broadcasts the updated weights back to the actors while sending updated priorities to the replay memory. This design enables scalable experience collection and stable policy learning.}
    \label{fig:distributed-arch}
\end{figure}

%-------------------------------------------------------------------------
\subsection{Distributed Actors}
\label{subsec:DEQL-actor}
%-------------------------------------------------------------------------

We consider a \ac{MOMDP} with state space $\gS$, action set $\gA$, transition kernel $\mathcal{P}(s'\!\mid s,a)$, discount factor $\gamma\!\in[0,1)$, and vectorial reward $\vr(s,a)\!\in\!\sR^{m}$. Preferences lie on the probability simplex $\Omega=\Delta^{m-1}$. The learner maintains vectorial action-value functions $Q(s,a,\vomega;\vtheta)\!\in\!\R^{m}$ (online) and $Q(s,a,\vomega;\vtheta^{-})\!\in\!\sR^{m}$ (target) and periodically shares them with all actors. For any $\vomega\!\in\!\Omega$, we define the scalarization as 
\begin{equation}
Q_{\vomega}(s,a;\vtheta)\;\triangleq\;\vomega^{\top} Q(s,a,\vomega;\vtheta)\in\sR .
\label{eq:DEQL-actor-scalarization}
\end{equation}

Each actor (i) generates experiences under an $\varepsilon$-greedy policy conditioned on a preference $\vomega$, (ii) computes a \emph{scalarized} \ac{DDQN} \ac{TD} error to initialize prioritized experience replay priorities, and (iii) sends batched transitions and priorities to its assigned replay memory shard. Each actor is assigned to a stratum $\Omega_{L}^{(u)}$ from a simplex lattice (see~\Cref{alg:sls-build}) and samples $\vomega$ uniformly in that simplex via barycentric weights (see~\Cref{alg:sls-sample}).

%-------------------
\subsubsection*{Behavior policy and data generation}
%-------------------
At the beginning of episode $e$, the actor draws a preference with support on the actor’s stratum $\Omega_{L}^{(u)}$ (see~\Cref{subsec:stratified-simplex}) and keeps it fixed: 
\[
\vomega^{(e)} \sim p_u(\cdot), 
\qquad 
\vomega_t \equiv \vomega^{(e)} \ \ \text{for all } t \text{ in episode } e .
\]

At the environment step $t\!\in\!\sN$, the exploration rate is linearly annealed from $\varepsilon_{\max}$ to $\varepsilon_{\min}$ over $T_{\mathrm{decay}}$ steps:
\begin{equation}
\varepsilon_t \;=\; \max\!\left\{
\varepsilon_{\min},\;
\varepsilon_{\max}-\frac{\varepsilon_{\max}-\varepsilon_{\min}}{T_{\mathrm{decay}}}\,t
\right\}.
\label{eq:DEQL-actor-eps}
\end{equation}
Here, $\varepsilon_{\max}\!\in\!(0,1]$ is the initial exploration rate, $\varepsilon_{\min}\!\in\![0,\varepsilon_{\max})$ is the floor, and $T_{\mathrm{decay}}\!\in\!\sN$ is the annealing horizon (after which $\varepsilon_t$ is clamped).

We define the greedy action under the fixed preference (using~\Cref{eq:DEQL-actor-scalarization}) as
\begin{equation}
a^\star_{\vomega}(s)\;=\;\argmax_{a\in\gA}\;Q_{\vomega}(s,a;\vtheta).
\label{eq:greedy-map}
\end{equation}
Let $\xi_t \sim \mathrm{Unif}(0,1)$ and let $\mathrm{UnifAct}(\gA)$ denote a single uniform draw from $\gA$ used only when exploring. As a result, the executed action becomes:
\begin{equation}
a_t \;=\;
\begin{cases}
\mathrm{UnifAct}(\gA), & \text{if } \xi_t < \varepsilon_t,\\[3pt]
a^\star_{\vomega}(s_t), & \text{otherwise.}
\end{cases}
\label{eq:eps-switch}
\end{equation}
Applying $a_t$ yields a transition $(s_t,a_t,\vr_t,s_{t+1},d_t)$ with a terminal flag $d_t\!\in\!\{0,1\}$.

%-------------------
\subsubsection*{Initial priority computation}
%-------------------
To initialize prioritized replay, the actor computes a \emph{scalar} \ac{DDQN} \ac{TD} error using a \emph{fresh} preference
$\tilde{\vomega} \sim p_u(\cdot)$ (independent of the behavior preference) to diversify the priorities:
\begin{align}
a^{\star}_{\tilde{\vomega}}(s_{t+1})
&=\argmax_{a'\in\gA}\; \tilde{\vomega}^{\!\top} Q(s_{t+1},a',\tilde{\vomega};\vtheta),
\label{eq:actor-ddqn-select}\\
\delta_{\mathrm{act}}
&=\underbrace{\tilde{\vomega}^{\!\top}\vr_t}_{r_{\tilde{\vomega}}}
+\gamma (1-d_t)\,
\tilde{\vomega}^{\!\top} Q(s_{t+1},a^{\star}_{\tilde{\vomega}}(s_{t+1}),\tilde{\vomega};\vtheta^{-})
-\tilde{\vomega}^{\!\top} Q(s_t,a_t,\tilde{\vomega};\vtheta),
\label{eq:actor-td}\\
p_{\mathrm{init}}
&=\;|\delta_{\mathrm{act}}|+\eps_0,
\qquad \eps_0 \ll 1 .
\label{eq:actor-prio}
\end{align}
The pair $\big((s_t,a_t,\vr_t,s_{t+1},d_t),\,p_{\mathrm{init}}\big)$ is buffered locally and flushed to the assigned replay shard.

%-------------------
\subsubsection*{Local batching and batched communication}
%-------------------
Let $u\in\{0,\dots,U-1\}$ denote the actor id and $K$ be the number of replay shards.
The actor~accumulates transitions in a local circular buffer of capacity $C$
and, when full, sends the batch to a designated shard via a single \ac{RPC}. Shard
selection uses the deterministic mapping
\begin{equation}
k(u)\;=\; u \bmod K .
\label{eq:DEQL-actor-shard}
\end{equation}
The actor $u$ then transmits the batch $\{(s_n,a_n,\vr_n,s'_n,d_n),\,p_{\text{init},n}\}_{n=1}^{C}$ to shard $k(u)$.

\begin{algorithm}[t]
\caption{D\hbox{-}EQL Actor}
\label{alg:DEQL-actor-final}
\begin{algorithmic}[1]
\Require actor id $u$, shards $\{\textsc{Shard}_k\}_{k=0}^{K-1}$, local buffer capacity $C$, discount factor $\gamma$, schedule params $(\varepsilon_{\max},\varepsilon_{\min},T_{\mathrm{decay}})$, stratum $\Omega_{L}^{(u)}$ with sampler $p_u(\cdot)$, online $Q(\cdot;\vtheta)$, target $Q(\cdot;\vtheta^{-})$
\State $k \gets u \bmod K$ \Comment{assigned shard, cf.\ \Eqref{eq:DEQL-actor-shard}}
\State $\mathcal{B}\gets\varnothing$, $\mathcal{P}\gets\varnothing$, $t\gets 0$
\For{episode $e=1,2,\ldots$}
  \State \textbf{Sample episodic preference:} $\vomega^{(e)} \sim p_u(\cdot)$; set $\vomega_t \equiv \vomega^{(e)}$ for this episode
  \While{episode not terminated}
    \State observe $s_t$
    \State compute $\varepsilon_t$ via the linear schedule \Eqref{eq:DEQL-actor-eps}
    \State draw $\xi_t \sim \mathrm{Unif}(0,1)$
    \If{$\xi_t < \varepsilon_t$}
      \State $a_t \gets \mathrm{UnifAct}(\gA)$
    \Else
      \State $a_t \gets \argmax_{a\in\gA} Q_{\vomega_t}(s_t,a;\vtheta)$ \Comment{greedy map \Eqref{eq:greedy-map}}
    \EndIf
    \State execute $a_t$; observe $(\vr_t,s_{t+1},d_t)$
    \State \textbf{Sample priority preference:} $\tilde{\vomega} \sim p_u(\cdot)$
    \State $a^\star \gets \argmax_{a'\in\gA} \tilde{\vomega}^{\!\top} Q(s_{t+1},a',\tilde{\vomega};\vtheta)$ \Comment{selection \Eqref{eq:actor-ddqn-select}}
    \State $\delta_{\mathrm{act}} \gets \tilde{\vomega}^{\!\top}\vr_t + \gamma(1-d_t)\tilde{\vomega}^{\!\top} Q(s_{t+1},a^\star,\tilde{\vomega};\vtheta^{-}) - \tilde{\vomega}^{\!\top} Q(s_t,a_t,\tilde{\vomega};\vtheta)$ \Comment{TD error \Eqref{eq:actor-td}}
    \State $p_{\mathrm{init}} \gets |\delta_{\mathrm{act}}| + \eps_0$ \Comment{priority \Eqref{eq:actor-prio}}
    \State append $\big((s_t,a_t,\vr_t,s_{t+1},d_t),\,p_{\mathrm{init}}\big)$ to $(\mathcal{B},\mathcal{P})$
    \If{$|\mathcal{B}|=C$}
      \State \textsc{AddExperiences}$\big(\textsc{Shard}_k,\mathcal{B},\mathcal{P}\big)$; reset $\mathcal{B},\mathcal{P}\gets\varnothing$
    \EndIf
    \State \textsc{Periodically}($\vtheta\gets\,\!$\textsc{Learner.Parameters()})
    \State $t\gets t+1$
  \EndWhile
\EndFor
\end{algorithmic}
\end{algorithm}

%-------------------------------------------------------------------------
\subsection{Centralized Learner}
\label{subsec:DEQL-learner}
%-------------------------------------------------------------------------
The learner (i) assembles prioritized mini-batches from $K$ replay shards,
(ii) samples a mini-batch of preferences from a Dirichlet and forms a
\emph{Cartesian product} with the transitions, (iii) for each
(transition, preference) pair computes an \emph{envelope} DDQN target by
maximizing over both actions and a finite set of supporting preferences,
(iv) updates the online parameters $\vtheta$ using a vector regression objective
plus a cosine similarity term, (v) refreshes per-transition priorities on the
shards, and (vi) periodically synchronizes the target network $\vtheta^{-}$
and publishes the latest online parameters to actors.

%-------------------
\subsubsection*{Mini-batch assembly from replay shards}
%-------------------
Let assume shard $k\!\in\!\{0,\dots,K-1\}$ stores $N_k$ items with priorities
$\{p_{k,i}\}$, with a total running value
$Z_k\triangleq\sum_{i=1}^{N_k} p_{k,i}^{\alpha}$, where $\alpha\!\in\![0,1]$.
A pair $(k,i)$ is sampled with probability
\begin{equation}
\Pr\!\big((k,i)\big)
\;=\;
\frac{Z_k}{\sum_{\ell=0}^{K-1}Z_{\ell}}
\cdot
\frac{p_{k,i}^{\alpha}}{Z_k}
\;=\;
\frac{p_{k,i}^{\alpha}}{\sum_{\ell}\sum_{j}p_{\ell,j}^{\alpha}}.
\label{eq:learner-per-prob}
\end{equation}
Given $N\triangleq\sum_k N_k$, experience replay importance weights $w_{k,i}$ (with exponent
$\beta\!\in\![0,1]$) are
\begin{equation}
w_{k,i}
=
\left(\frac{1}{N}\cdot\frac{1}{\Pr((k,i))}\right)^{\beta}
\bigg/
\max_{k',i'}\left(\frac{1}{N}\cdot\frac{1}{\Pr((k',i'))}\right)^{\beta}.
\label{eq:learner-per-weights}
\end{equation}
The learner queries the shards values $\{Z_k\}_{k=1}^K$, allocates per-shard batch sizes proportionally, fetches tuples (transition, index, $w_{k,i}$, shard id), and aggregates them into a transition batch of size $B$.
% $(\text{transition},\ \text{index},\ w_{k,i},\ \text{shard id})$,

%-------------------
\subsubsection*{Preference Sampling}
%-------------------
For each training step, we draw an i.i.d.\ mini-batch of preferences from a
Dirichlet distribution:
\begin{equation}
\{\vomega_j\}_{j=1}^{P}\subset\Omega=\Delta^{m-1},
\qquad
\vomega_j \stackrel{\text{i.i.d.}}{\sim} \mathrm{Dir}(\valpha),
\qquad
\valpha\in(0,\infty)^{m}.
\label{eq:learner-pref-sampling-dir}
\end{equation}
The choice $\valpha=\vone_m$ yields the uniform distribution on the simplex; $\valpha\!<\!\vone$ emphasizes corners, while $\valpha\!>\!\vone$ emphasizes the interior.
To couple transitions and preferences, we form the \emph{Cartesian product} index set
\[
\gI \;=\; \{1,\dots,B\}\times\{1,\dots,P\},
\]
so every transition is paired with every preference. For $(i,j)\in\gI$, we write
$(s_i,a_i,\vr_i,s'_i,d_i,\vomega_j)$ and define the scalarization
$Q_{\vomega_j}(s,a;\vtheta)=\vomega_j^{\!\top}Q(s,a,\vomega_j;\vtheta)$.

%-------------------
\subsubsection*{Envelope DDQN selection and vector target}
%-------------------
The EQL's envelope backup requires maximizing over both actions and supporting
preferences. Directly optimizing over $\Omega$ is intractable, so we
\emph{approximate} the inner maximization by searching over a sampled set
$\gW\!=\!\{\vomega_j\}_{j=1}^{P}$. For each pair $(i,j)\in\gI$, we get:
\begin{align}
(a_{i,j}^{\star},\ \tilde{\vomega}_{i,j}^{\star})
&=\argmax_{a'\in\gA,\ \vomega'\in\gW}
\ \vomega_j^{\!\top} Q(s'_i,a',\vomega';\vtheta),
\label{eq:learner-env-select}\\[4pt]
\vy_{i,j}
&=\vr_i\;+\;\gamma(1-d_i)\,
Q\!\big(s'_i,a_{i,j}^{\star},\tilde{\vomega}_{i,j}^{\star};\vtheta^{-}\big)
\ \in\ \sR^{m}.
\label{eq:learner-env-target}
\end{align}
Thus each \emph{query preference} $\vomega_j$ selects a \emph{supporting}
preference $\tilde{\vomega}_{i,j}^{\star}\!\in\!\gW$ and action
$a_{i,j}^{\star}$ that together realize the envelope along direction
$\vomega_j$. The online prediction is
$Q_{\mathrm{pred},i,j}=Q(s_i,a_i,\vomega_j;\vtheta)\in\sR^{m}$.

%-------------------
\subsubsection*{Training Loss}
%-------------------
For each $(i,j)\in\gI$, we define:
\begin{equation}
\Ls_{\text{mmse}}(i,j;\vtheta)
\;=\;
\big\|\vy_{i,j} - Q(s_i,a_i,\vomega_j;\vtheta)\big\|_2^2,
\label{eq:learner-loss-vec}
\end{equation}
\begin{equation}
\Ls_{\text{cos}}(i,j;\vtheta)
\;=\;
1 -
\frac{\vomega_j^{\!\top}Q(s_i,a_i,\vomega_j;\vtheta)}
{\|\vomega_j\|_2\,\|Q(s_i,a_i,\vomega_j;\vtheta)\|_2}.
\label{eq:learner-loss-cos}
\end{equation}
With tradeoff $\lambda\!\ge\!0$ and PER weights $w_{k,i}$ tied to the
\emph{transition} (replicated over its $P$ preference copies), the learner
minimizes
\begin{equation}
\gL(\vtheta)
\;=\;
\frac{1}{BP}\sum_{i=1}^{B} \sum_{j=1}^{P}
w_{k(i),\,idx(i)}\,
\Big(\Ls_{\text{mmse}}(i,j;\vtheta)
+\lambda\,\Ls_{\text{cos}}(i,j;\vtheta)\Big),
\label{eq:learner-loss-total}
\end{equation}
where $k(i)$ and $idx(i)$ are the shard id and local index of transition $i$.

%-------------------
\subsubsection*{Priority refresh (learner side)}
%-------------------
To refresh priority weights, we first compute scalarized residuals per pair $(i,j)$,
\begin{equation}
\delta_{i,j}
\;=\;
\vomega_j^{\!\top}\!\big(\vy_{i,j}-Q(s_i,a_i,\vomega_j;\vtheta)\big),
\label{eq:learner-delta}
\end{equation}
then aggregate to a single priority per \emph{original transition} $i$,
\begin{equation}
p_{\mathrm{new}}(i)
\;=\;
\max_{1\le j\le P} |\delta_{i,j}| \;+\; \eps_0,
\qquad \eps_0>0,
\label{eq:learner-priority}
\end{equation}
and return $\big(\text{indices}(i),\,p_{\mathrm{new}}(i)\big)$ to the
corresponding shards to update PER totals.

%-------------------
\subsubsection*{Target updates and parameter broadcast}
%-------------------
Every $C_{\mathrm{tgt}}$ steps, the learner updates the target network either by a hard copy
\begin{equation}
\vtheta^{-}\ \gets\ \vtheta
\label{eq:learner-hard}
\end{equation}
or a soft update with factor $\tau\!\in\!(0,1]$:
\begin{equation}
\vtheta^{-}\ \gets\ (1-\tau)\,\vtheta^{-}\;+\;\tau\,\vtheta.
\label{eq:learner-soft}
\end{equation}
Every $C_{\mathrm{push}}$ steps, the latest online parameters are published to the
shared model used by all actors.

\begin{algorithm}[t]
\caption{D\hbox{-}EQL Learner}
\label{alg:DEQL-learner}
\begin{algorithmic}[1]
\Require shards $k\in\{0, \dots, K-1\}$, discount $\gamma$, PER exponents $(\alpha,\beta)$, target period $C_{\mathrm{tgt}}$, push period $C_{\mathrm{push}}$, preference batch size $P$, tradeoff $\lambda$, Dirichlet parameter $\valpha$
\State Initialize online $\vtheta$; set target $\vtheta^{-}\!\gets\!\vtheta$
\While{training}
  \State Query $\{Z_k\}$; allocate per-shard batch sizes; fetch prioritized transitions with $(\text{indices}(i),\ w_{k(i),\,idx(i)})$ \Comment{\Eqref{eq:learner-per-prob}--\Eqref{eq:learner-per-weights}}
  \State Sample preferences $\{\vomega_j\}_{j=1}^{P}\stackrel{\text{i.i.d.}}{\sim}\mathrm{Dir}(\valpha)$ \Comment{\Eqref{eq:learner-pref-sampling-dir}}
  \State Form Cartesian product $\gI=\{1..B\}\times\{1..P\}$ (replicate transitions across all $\vomega_j$)
  \ForAll{$(i,j)\in\gI$}
    \State $(a_{i,j}^{\star},\tilde{\vomega}_{i,j}^{\star})\!\gets\!\argmax_{a'\in\gA,\ \tilde{\vomega}\in\{\vomega_1,\dots,\vomega_P\}}
    \ \vomega_j^{\!\top}Q(s'_i,a',\tilde{\vomega};\vtheta)$ \Comment{\Eqref{eq:learner-env-select}}
    \State $\vy_{i,j}\!\gets\!\vr_i+\gamma(1-d_i)\,Q(s'_i,a_{i,j}^{\star},\tilde{\vomega}_{i,j}^{\star};\vtheta^{-})$ \Comment{\Eqref{eq:learner-env-target}}
  \EndFor
  \State Compute $\gL(\vtheta)$ via \Eqref{eq:learner-loss-total}; take a gradient step on $\vtheta$
  \State For each $i\in\{1..B\}$, compute $p_{\mathrm{new}}(i)$ via \Eqref{eq:learner-delta}--\Eqref{eq:learner-priority}; send updates to shards
  \If{step $\bmod\ C_{\mathrm{tgt}}=0$} \State Update $\vtheta^{-}$ via \Eqref{eq:learner-hard} or \Eqref{eq:learner-soft} \EndIf
  \If{step $\bmod\ C_{\mathrm{push}}=0$} \State Publish $\vtheta$ to actors \EndIf
\EndWhile
\end{algorithmic}
\end{algorithm}

\paragraph{Notes.}\emph{(i)} The envelope selection (see~\Cref{eq:learner-env-select}) is the finite-set
approximation of the bi-level inner maximization over preferences; letting
$P\!\uparrow\!\infty$ densifies the approximation.
\emph{(ii)} The Cartesian product ensures that \emph{every} transition is trained
under \emph{every} sampled preference each step (dense supervision).
\emph{(iii)} Priorities are defined per transition by aggregating scalarized TD
residuals over $P$ preference replicas.

\paragraph{Communication summary.}
The learner \emph{pulls} batches proportional to $Z_k$, \emph{pushes} refreshed
priorities $p_{\mathrm{new}}$ (updating shard totals), and periodically
\emph{broadcasts} the latest online parameters. Preference expansion couples
updates across $\Omega$, while \ac{DDQN} selection/evaluation preserves stability.

%-------------------------------------------------------------------------
\subsection{Stratified Sampling on the Probability Simplex}
\label{subsec:stratified-simplex}
%-------------------------------------------------------------------------
We want to sample preferences $\vomega\!\in\!\Omega:=\Delta^{m-1}$ so that \emph{all regions} of the simplex are adequately covered during data generation, thereby reducing variance and avoiding mode collapse toward a few scalarizations.

\paragraph{Baseline (no stratification).}
When no partition is imposed, we draw independent and identically distributed (i.i.d.) preferences from a Dirichlet distribution,
\begin{equation}
\vomega \;\sim\; \mathrm{Dir}(\valpha),\qquad
\valpha\in(0,\infty)^{m}.
\label{eq:dirichlet-baseline}
\end{equation}
Choosing $\valpha=\vone_m$ yields the uniform distribution on $\Omega$.

%-------------------
\subsubsection{Deterministic Equal-Volume Strata via a Simplex Lattice}
%-------------------
We partition $\Omega$ into congruent $(m{-}1)$-simplices by using a barycentric lattice with resolution $L\!\in\!\sN$. We define the lattice vertices as:
\begin{equation}
\gV_L
\;:=\;
\left\{
\frac{\vk}{L}\in\Omega \;:\;
\vk=(k_1,\dots,k_m)\in\sN^{m},\ \sum_{i=1}^{m}k_i=L
\right\},
\label{eq:sls-vertices}
\end{equation}
and we consider a fixed permutation $\pi$ of $\{1,\dots,m\}$. For each \emph{base} point $\vk\in\sN^{m}$ with $\sum_i k_i=L-1$, we form the $m$ lattice points as:
\begin{equation}
\vv_0=\frac{\vk}{L}, \quad \vv_r=\frac{\vk+\ve_{\pi(1)}+\cdots+\ve_{\pi(r)}}{L}, \quad r=1,\dots,m-1,
\label{eq:sls-verts-cell}
\end{equation}
and define the micro-simplex (stratum):
\begin{equation}
\Omega_{L}(\vk)\;:=\;\mathrm{conv}\{\vv_0,\vv_1,\dots,\vv_{m-1}\}\ \subset\ \Omega.
\label{eq:sls-cell}
\end{equation}
The collection $\{\Omega_{L}(\vk):\ \vk\in\sN^{m},\ \sum_i k_i=L-1\}$ tiles $\Omega$ into $L^{\,m-1}$ equal-volume strata.

%-------------------
\subsubsection{Uniform Sampling Within a Stratum}
%-------------------

Let $\Omega_{L}(\vk)=\mathrm{conv}\{\vv_0,\dots,\vv_{m-1}\}$ be any stratum. We draw barycentric weights $\rvz\sim\mathrm{Dir}(\vone_m)$ and map affinely:
\begin{equation}
\vomega \;=\; \sum_{r=0}^{m-1} z_r\,\vv_r
\;\in\; \Omega_{L}(\vk).
\label{eq:sls-affine-sample}
\end{equation}
This yields a sample \emph{uniform} in the stratum.

%-------------------
\subsubsection{Assigning Strata to Actors}
%-------------------

Index the $L^{\,m-1}$ strata in a fixed order as $\{\Omega_{L}^{(u)}\}_{u=0}^{U-1}$ with $U=L^{\,m-1}$ (or group them when $U$ exceeds the number of actors). Actor $u$ repeatedly samples $\omega\in\Omega_{L}^{(u)}$ via~\Cref{alg:sls-sample}, ensuring non-overlapping coverage across actors.

\begin{algorithm}[t]
\caption{Simplex–Lattice Stratification (build strata at resolution $L$)}
\label{alg:sls-build}
\begin{algorithmic}[1]
\Require dimension $m$, resolution $L$, a fixed permutation $\pi$ of $\{1,\dots,m\}$
\State $\gS \gets \varnothing$ \Comment{list of strata (each as $m$ vertices in $\sR^{m}$)}
\ForAll{$\vk\in\sN^{m}$ with $\sum_{i=1}^{m}k_i=L-1$}
  \State $\vv_0 \gets (\vk)/L$
  \For{$r=1$ to $m-1$} \State $\vv_r \gets (\vk+\ve_{\pi(1)}+\cdots+\ve_{\pi(r)})/L$ \EndFor
  \State append $\mathrm{conv}\{\vv_0,\dots,\vv_{m-1}\}$ to $\gS$
\EndFor
\State \Return $\gS$ \Comment{$|\gS|=L^{\,m-1}$ equal-volume strata}
\end{algorithmic}
\end{algorithm}

\begin{algorithm}[t]
\caption{Sample uniformly from a stratum $\Omega_{L}(\vk)$}
\label{alg:sls-sample}
\begin{algorithmic}[1]
\Require vertices $\{\vv_0,\dots,\vv_{m-1}\}$ of $\Omega_{L}(\vk)$
\State draw $\rvz\sim\mathrm{Dir}(\vone_m)$
\State \Return $\vomega=\sum_{r=0}^{m-1} z_r\,\vv_r$ \Comment{uniform in the stratum}
\end{algorithmic}
\end{algorithm}

%-------------------
\subsubsection{Discussion and alternatives}
%-------------------
\textbf{Coverage and variance.} Compared to i.i.d.\ Dirichlet sampling
shown in~\Cref{eq:dirichlet-baseline}, the lattice partition yields systematic
coverage of the entire simplex and reduces estimator variance by ensuring that
each subregion is represented.

\textbf{Resolution.} Larger $L$ gives finer strata ($L^{\,m-1}$ pieces) and
smoother coverage at the cost of more partitions to manage.

\textbf{Clustering alternative.} When equal-volume strata are unnecessary,
a simple alternative is to draw a large pilot set
$\{\omega^{(n)}\}_{n=1}^{N_0}\sim\mathrm{Dir}(\valpha)$ and run $k$-means
on the $(m{-}1)$-dimensional simplex (with cosine or Euclidean distance); the
Voronoi cells of the cluster centers define strata. Sampling within a cell can
be done by re-running Dirichlet draws and accepting points whose nearest
center matches the cell (approximately uniform within each cell).

%-------------------------------------------------------------------------
\subsection{Experiments}
\label{appendixDMRL:DEQL_tests}
%-------------------------------------------------------------------------
\subsubsection{Environments and Setup}
We evaluate our approach using two well-established \ac{MORL} benchmarks: \ac{DST} and \ac{FTN}. Both environments are widely used in the literature \citep{YSN:19,Basaklar:23}, providing standardized testbeds for assessing Pareto front coverage and preference generalization.

The \ac{DST} environment is a grid-world wherein an agent controls a submarine that must navigate~from the surface to collect one of several treasures placed at different depths. Each treasure yields~a~two- dimensional reward: a positive value and a negative time penalty. The task is inherently multi-objective, requiring the agent to balance collecting high-value treasures against minimizing travel time. The Pareto front is well understood and serves as a reliable benchmark for coverage~and~accuracy.

The \ac{FTN} environment generalizes this idea to a tree-structured setting. Starting from the root, the agent makes sequential decisions until it reaches a leaf node, where it receives a multi-dimensional reward corresponding to the chosen fruit. The tree depth controls task complexity. At depth five, the agent makes five sequential decisions; at depth seven, the number of possible outcomes grows exponentially, producing a much larger and more diverse Pareto front. This makes \ac{FTN} with higher depth a significantly more challenging benchmark, particularly for algorithms that must adapt to unseen preferences or maintain wide coverage.

In all experiments, algorithms are trained across a range of randomly sampled linear preference vectors, following the setup in \citet{YSN:19}. At test time, additional preference vectors are sampled to assess generalization. All results are averaged over multiple random seeds to account for variance.

\subsubsection{Metrics and Results}
We evaluate performance using three widely adopted metrics in \ac{MORL}: 
\begin{enumerate}
    \item \textbf{Coverage Ratio F1 (CRF1):} A harmonic mean of precision and recall that captures both accuracy and coverage of the Pareto front. 
    \item \textbf{Hypervolume:} The volume dominated by the obtained solutions with respect to a reference point, reflecting both the quality and diversity of the Pareto front.  
    \item \textbf{Sparsity:} The average distance between neighboring solutions, indicating how uniformly the Pareto front is covered. 
\end{enumerate}

\begin{table}[!thb] 
\centering 
\caption{We compare the performance of the proposed distributed Q-learning algorithm with the original Q-learning approach~\citet{YSN:19} and~\citet{Basaklar:23} in terms of CFR1, hypervolume and sparsity, showing superior performance in more complex scenarios.} 
\label{my-label} 
\small % or \footnotesize, or \scriptsize
\begin{tabular}{@{}lS[table-format=1.3]S[table-format=3.2]S[table-format=1.2]S[table-format=1.3]S[table-format=5.2]S[table-format=1.3]S[table-format=4.2]@{}} 
\toprule 
  & \multicolumn{3}{c}{Deep Sea Treasure} & \multicolumn{2}{c}{Fruit Tree Nav. (d=5)} & \multicolumn{2}{c}{Fruit Tree Nav. (d=7)} \\ 
\cmidrule(r){2-4} \cmidrule(l){5-6} \cmidrule(l){7-8} 
  & {CRF1} & {Hyperv.} & {Sparsity} & {CRF1} & {Hyperv.} & {CRF1} & {Hyperv.} \\ 
\midrule 
\cite{YSN:19} \hspace{-3.5mm} & 0.994 & 227.39 & 2.62 & 1.0 & 6920.58 & 0.819 & 6395.27 \\  
\cite{Basaklar:23} \hspace{-3.5mm} & 1.0 & 241.73 & 1.14 & 1.0 & 6920.58 & 0.920 & 11419.58 \\  
\ac{D-EQL} (ours) \hspace{-3.5mm} & \hspace{-3.5mm}\textbf{1.0} & \textbf{241.73} & \textbf{1.14} & \hspace{-3.5mm}\textbf{1.0} & \hspace{2.5mm}\textbf{6920.58} & \hspace{-3.5mm}\textbf{1.0} & \hspace{-3.5mm}\textbf{12110.74} \\ 
\bottomrule 
\end{tabular} 
\end{table}

\Cref{my-label} compares \ac{D-EQL} with prior approaches. On the simpler \ac{DST} domain, all methods achieve near-perfect coverage. Our method, on pair with \cite{Basaklar:23}, attains $\text{CRF1}=1.0$ and simultaneously achieves the highest hypervolume and lowest sparsity, indicating comprehensive and well-distributed solutions along the Pareto front.

On \ac{FTN} with depth five, performance saturates across all methods with perfect coverage and identical hypervolume, reflecting the relative simplicity of this setting.

The advantage of \ac{D-EQL} becomes evident in the more complex \ac{FTN} with depth seven. Prior methods show a notable drop in coverage ($\text{CRF1}=0.819$ for~\cite{YSN:19} and $0.92$ for~\cite{Basaklar:23}), whereas \ac{D-EQL} maintains perfect coverage ($\text{CRF1}=1.0$). Moreover, \ac{D-EQL} achieves the highest hypervolume ($12110.74$), showing $22.1\%$ and $8.69\%$ improvements over~\cite{YSN:19} and~\cite{Basaklar:23}, respectively, and  demonstrating broader Pareto front coverage and superior solution diversity. These results show that our distributed training scheme preserves accuracy while scaling effectively to environments with exponentially growing outcome spaces.

Overall, \ac{D-EQL} achieves state-of-the-art performance: it matches existing methods on simpler tasks and clearly outperforms them in complex scenarios, highlighting the benefits of distributed training for multi-objective reinforcement learning. This makes \ac{D-EQL} a more suitable for handling the vast dimensions of state-action spaces in \ac{RAN} control functions.

\subsubsection{Hyperparameters}  
The hyperparameters used in \ac{D-EQL} training for \ac{DST} and \ac{FTN} environment are listed in \Cref{tab:hyperparams-DEQL-both}.

\begin{table}[!htb]
\centering
\caption{Hyperparameters used for \ac{D-EQL} on deep sea treasure and fruit tree navigation.}
\label{tab:hyperparams-DEQL-both}
\begin{tabular}{lll}
\toprule
\textbf{Hyperparameter} & \textbf{DST} & \textbf{FTN} \\
\midrule
\multicolumn{3}{l}{\textbf{Model}} \\
\midrule
Hidden feature (units per layer) & 256 & 512 \\
Activation & SiLU & SiLU~\cite{EUD:18} \\
Number of layers & 3 & 3 \\
\midrule
\multicolumn{3}{l}{\textbf{Actor}} \\
\midrule
Number of actors & 10 & 10 \\
$\epsilon$-greedy (linear, start $\rightarrow$ final) & 0.8 $\rightarrow$ 0.1 & 0.8 $\rightarrow$ 0.1 \\
Anneal timesteps & $1 \times 10^{6}$ & $1 \times 10^{6}$ \\
Local buffer capacity & 125 & 125 \\
Max environment step & $1 \times 10^{6}$ & $1 \times 10^{6}$ \\
\midrule
\multicolumn{3}{l}{\textbf{Learner}} \\
\midrule
Learning rate & $3.5 \times 10^{-4}$ & $3.5 \times 10^{-4}$ \\
Target update period (gradient step) & 1 & 1 \\
Model sync period (gradient step) & 250 & 250 \\
Discount factor $\gamma$ & 0.99 & 0.99 \\
Prefetched batches & 16 & 16 \\
Transition batch size & 128 & 128 \\
Preference batch size & 128 & 128 \\
\midrule
\multicolumn{3}{l}{\textbf{Replay Memory (PER)}} \\
\midrule
Number of shards & 1 & 1 \\
Capacity & $5 \times 10^{5}$ & $5 \times 10^{5}$ \\
Priority exponent $\alpha$ & 0.7 & 0.7 \\
Importance sampling $\beta$ (linear, start $\rightarrow$ final) & 0.4 $\rightarrow$ 1.0 & 0.4 $\rightarrow$ 1.0 \\
$\beta$ anneal timesteps & $2 \times 10^{4}$ & $2 \times 10^{4}$ \\
\bottomrule
\end{tabular}
\end{table}

% ===================================================================
% SECTION 15: Case Study
% ===================================================================
\clearpage
\section{Case Study}
\label{appendix:study_case}

We design a controller agent for \ac{LA}, a crucial functionality of modern wireless communication systems that employs adaptive coding and modulation to optimize the spectral efficiency of the radio link between transmitter and receiver. By adopting a \ac{MORL} approach, the \ac{LA} controller agent can adjust transmission parameters to meet connectivity service intents expressed in terms of data rate, reliability, and latency requirements for individual users.

\subsection{Link Adaptation}

\ac{LA} adapts the modulation order and code rate of individual packet transmissions to match the capacity of the radio link capacity, given the radio link state. The \ac{LA} parameters are encoded into a unique value, referred to as \ac{MCS} index in~\cite{3GPP38214}, that is provided to the receiver for packet decoding. 

The 3GPP \ac{5G} \ac{NR} system, rely on an \ac{OLLA} approach inspired to~\cite{PMK+:07} to maximize the link spectral efficiency while adhering to a predefined \ac{BLER} target using receiver-side \ac{CSI}, such as \ac{CQI}~\cite{3GPP38214}, and \ac{HARQ} feedback--a 1-bit information indicating whether a prior packet transmission was successful or not. While this approach suits best-effort traffic, its reliance on long communication sessions to converge makes is suboptimal to address connectivity service intents under more general conditions, such as short bursty traffic, fast channel aging, medium-high user mobility, etc. 

A \ac{MORL} approach instead enables to dynamic \ac{LA} toward selection transmission parameters that best align with different service intents. For example, selecting \ac{MCS} conservatively—e.g., lower modulation orders such as \ac{QPSK} or reduced code rates—favor robustness by lowering the probability of decoding errors. This enables to achieve highly reliable transmissions at the cost of throughput, since more time-frequency \ac{RE} are required per information bit. Conversely, an aggressive \ac{MCS} selection can push spectral efficiency closer to or even beyond the instantaneous link capacity, exploiting retransmissions to increase data rate and reduce latency for best-effort traffic. However, overly aggressive choices may lead to excessive retransmissions and throughput degradation. By explicitly balancing these conflicting objectives, MORL allows \ac{LA} to adapt beyond fixed \ac{BLER}-driven policies, supporting a wider range of connectivity intents.

\subsection{MOMDP Design for Link Adaptation}

Our goal is to train a single \ac{PEUMO} for \ac{LA} to learn the Pareto frontier outlining the optimal trade-off between the utilization of radio resources and the amount of information bits delivered by a packet transmission.
 
As \ac{LA} and \ac{HARQ} operate on a per-\ac{UE} and per-packet transmission basis, we formulate this problem as an episodic \ac{MOMDP} $\mathcal{M} = \left\langle\mathcal{S}, \mathcal{A}, p, \vr, \Omega, \gamma, \rho_{0} \right\rangle$, where $\mathcal{S}$ denotes the state space, $\mathcal{A}$ the action space, $p(s^{\prime} \mid s, a)$ the transition dynamics, $\vr:\mathcal{S} \times \mathcal{A} \rightarrow \mathbb{R}^K$ a multi-dimensional reward vector, $\Omega\subseteq \mathbb{R}^K$ the reward preference space, $\gamma \in [0,1)$ a discount factor, and $\rho_{0}$ the initial state distribution.

An episode models the lifespan of a \ac{UE} packet in the \ac{HARQ} process—from its first transmission to either a successful reception or the packet being dropped upon $N$ transmission attempts, as illustrated in~\Cref{fig:episode}. This enables us to train a single \ac{RL} policy from the collective experience generated by any \acp{UE} across the network. A transition in the episode represents the duration of a packet transmission in the \ac{HARQ} process, from the selection of LA parameters (i.e., the action) to the reception of the associated \ac{HARQ} feedback, i.e., an \ac{ACK} or \ac{NACK} for successful or failed transmission, respectively. For instance, the 3GPP \ac{5G} \ac{NR} system, used in our evaluations, supports at most four packet retransmissions. Hence, the episode length $N$ may range from one to five~steps. Each step is characterized by a state, an action, and an associated reward and preference vectors, as presented next. 

\begin{figure}[t]
    \centering
    \includegraphics[width=1\columnwidth]{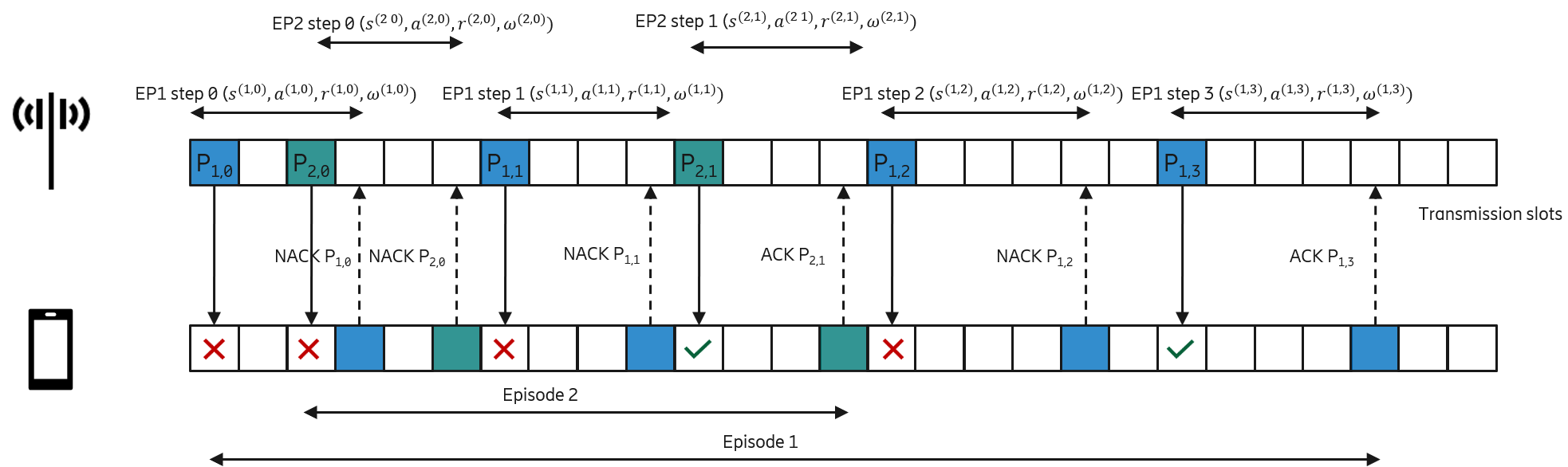}
    \caption{Example of \ac{MOMDP} episodes modelling the downlink \ac{LA} and \ac{HARQ} process. Here, the notation $P_{i, n}$ denotes the $n$-th transmission attempt of the $i$-th data packet.}
    \label{fig:episode}
\end{figure}

\subsection{Action Space}

The action space consists of the set of \ac{MCS} index values supported by a communication standard, i.e., $\mathcal{A}=\{a_m \mid a_m = m, m= 0, \dots, M-1\}$. Therefore, an action $a_m \in \mathcal{A}$ implicitly provides a combination of modulation order, code rate, and spectral efficiency to be used to transmit a packet. The \ac{5G} \ac{NR} system used in our evaluations in \Cref{sec:experiment} supports $M=28$ \ac{MCS} index values as specified in Table 5.1.3.1-1 or Table 5.1.3.1-2 in~\cite{3GPP38214}, corresponding to modulation orders up to 64QAM and 256QAM, respectively. 

For a new packet transmission, the selection of an \ac{MCS} index, combined with the time-frequency resources allocated by a scheduler, determines the amount of information bits, i.e., the \ac{TBS}, to be transmitted. Packet re-transmissions, however, reuse the \ac{TBS} value of the original transmission, as no new information is transmitted. A packet re-transmission, however, may occur with a different \ac{MCS} index therefore resulting in  possibly a different amount of radio resources.

\subsection{Reward Vector and Preference Space}\label{appendixE:subsec:morl_reward}

We design a two-dimensional reward function $\vr = [r_1, r_2]^{\top} \in \mathcal{S}\times \mathcal{A} \in \mathbb{R}^{2} $ with two competing components: $r_1$ representing the amount of information bits successfully carried by a packet; and $r_2$ denoting the cost, in terms of time-frequency resource, incurred in each individual transmission of the packet. Specifically, for each transmission attempt $n$ of a packet, we define the reward function as
\begin{align}
    \vr^{(n)}(s,a) = \left\{
    \begin{array}{cl}
       \left[\begin{array}{c}
            0\\
            -\frac{N_{RE}^{(n)}}{N_{RE}^{\max}}
       \end{array}
       \right]  &  \mathrm{if\;transmision\;fails\;at}\; n\mathrm{-th}\;\mathrm{attempt},\\
       & \\
         \left[\begin{array}{c}
            \frac{TBS}{N_{RE}^{\max}}\\
            -\frac{N_{RE}^{(n)}}{N_{RE}^{\max}}
       \end{array}
       \right]  &  \mathrm{if\;transmision\;succeeds\;at}\; n\mathrm{-th}\;\mathrm{attempt},
    \end{array}
    \right.\label{eqn:2D_reward}
\end{align}
where the \ac{TBS} and $N_{RE}^{(n)}$ denote the number of information bits carried by the packet and the number of \ac{RE} used for the $n$-th transmission attempt, respectively, and $N_{RE}^{\max}$ is the maximum number of \acp{RE} available, given the system bandwidth. Scaling the reward components by $N_{RE}^{\max}$ has a twofold purpose: Firstly, it keeps each component within similar range of values, while preserving the functional relation between the \ac{MCS} index selected to transmit \ac{TBS} information bits and the required number of time-frequency \ac{RE}. This relation is specified by communication standards, as in the \ac{3GPP} \ac{TS} 38.211~\cite{3GPP38211}. Secondly, it makes the reward design agnostic to the system bandwidth, with $\frac{TBS}{N_{RE}^{\max}}$ representing the spectral efficiency for transmitting \ac{TBS} bits using the entire system bandwidth. This allows us to employ domain randomization in training (cf.~\cref{appendix:F1_network_simulator}) to improve model generalization over the \ac{RAN} environment.

Therefore, for each packet transmission attempt $n$, the first reward component takes value $r_1^{(n)}=\nicefrac{TBS}{N_{RE}^{\max}}$ is the transmission is successful or $r_1^{(n)}=0$ otherwise. The second reward component, on the other hand, always indicates the resource cost incurred at the current transmission attempt $n$, i.e., $r_2^{(n)}= - \nicefrac{N_{RE}^{(n)}}{N_{RE}^{\max}}$, regardless of whether the transmission attempt succeeds or fails.

\subsection{State Design}

A key goal of our design is to achieve model generalization across diverse \ac{RAN} environments, enabling a single \ac{MORL} model to operate reliably under different deployments and radio conditions. To this end, we construct a rich state space $\mathcal{S}\subseteq \mathbb{R}^{K}$ and apply domain randomization in training~\cite{Igl:19domain_rand}. To model the state space $\mathcal{S}$ for link adaptation, we follow~\cite{demirel:2025} which considers a \ac{DQN} approach for \ac{LA} with a single, fixed reward design based on the link spectral efficiency. In particular, we model $\mathcal{S}$ using two types of features: (a) semi-static information characterizing the network deployment surrounding the \ac{UE}; (b) and information describing observable link dynamics relevant to infer \ac{LA} parameters. 

Semi-static information characterizing the network deployment may include, for instance,  deployment type (e.g., rural, urban, dense urban, etc.), location, orientation, relationships among network sites or radio cells, as well as technology configurations, such as whether the system operated in time-duplex or full duplex mode, carrier frequency, system bandwidth, transmit power, antenna array type, etc. On the other hand, information characterizing the dynamics of \ac{LA} consist of real-time observation (measured in a milliseconds timescale), such as channel state information, \ac{HARQ} feedback, measurement of path loss, data buffer state, historical actions, and more. We refer to~\cite{demirel:2025} for a complete description of the state features.

%-------------------------------------------------------------------------
\clearpage
\section{Extended Experimental Evaluation}
\label{appendix:Experiments}
%-------------------------------------------------------------------------

This Section extends the discussion and empirical evaluation presented in~\Cref{sec:experiment} with additional results. We organized the material as follows:~\Cref{appendix:F1_network_simulator,appendix:F2_training_setup} describe the network simulator environment and training setup for the \ac{MORL} controller agent;~\Cref{appendix:F3_MORLagent_eval} extend the analysis of the controller agent presented in~\Cref{subsec:MORL_controller_LA}, including an additional scenario with two communication services. Finally,~\Cref{appendix:F4_BO_preference_steering} extends our analysis of the intent fulfillment loop.

%-------------------------------------------------------------------------
\subsection{Network Simulator Environment}
\label{appendix:F1_network_simulator}
%-------------------------------------------------------------------------

We train and evaluate the \ac{MORL} controller agent using a high-fidelity, event-driven system-level simulator compliant with the 3GPP 5G \ac{NR} specifications. Each rollout simulation models a heterogeneous multi-cell \ac{RAN} operating in \ac{TDD} mode with \ac{SU-MIMO} transmission. The carrier frequency is set to 3.5~GHz, and the physical layer follows the \ac{OFDM} numerology $\mu=0$ specified in 3GPP TS 38.211 (cf.~Table~4.2-1~\cite{3GPP38211}).

To improve model generalization across diverse \ac{RAN}deployments and radio enviroments, we apply domain randomization across multiple network characteristics, summarized in~\Cref{table:sim_params}. Each simulation consists of three tri-sector radio sites, randomly configured as either conventional \ac{MIMO} or \ac{mMIMO}, with antenna attributes defined in~\Cref{table:sim_params}. Site-level parameters such as location, cell radius, system bandwidth, and downlink transmit power are also randomized by sampling values from the same parameter set.

The training scenario is further diversified by randomizing cell load, traffic type, \acp{UE}, and receiver configuration. \acp{UE} are generated with a mixture of \ac{FB} and \ac{eMBB} traffic, randomly placed in the simulated area according to one of the indoor/outdoor probability distributions in~\Cref{table:sim_params}. Each \ac{eMBB} \ac{UE} generates traffic with variable packet size and inter-arrival times, modeled using empirical distributions derived from field measurement campaigns.  

Finally, individual \acp{UE} are randomized in terms of antenna configuration, mobility (speed), and receiver implementation. The latter accounts for manufacturer-specific differences in hardware (e.g., antenna arrays and chipsets) and internal algorithms (e.g., \ac{CSI} estimation), which influence perceived radio conditions. This randomized environment ensures that the \ac{MORL} controller agent is trained under various realistic network conditions, thus improving its ability to generalize to unseen scenarios.

\begin{table}[!thb]
	\caption{\ac{RAN} environment simulation parameters for domain randomization during training.} % title name of the table
	\centering % centering table
	\begin{tabular}{l l l} % creating 10 columns
		\toprule[1pt]\midrule[0.3pt]
		\textbf{Parameter} & \textbf{Value range} & \textbf{Description} \\ [0.5ex]
		\midrule
		% Entering 1st row
		Duplexing type & TDD & Fixed\\
		Carrier frequency & 3.5 GHz & Fixed \\
		Deployment type & 3-site 9-sector & \\
		Site type & \{\ac{MIMO}, \ac{mMIMO}\} & Randomized\\
		Antenna array & 1x2x2 \ac{MIMO} (4) & Fixed  \\
		&8x4x2 \ac{mMIMO} (64) & Fixed \\
		Cell radius & \{166, 300, 600, 900, 1200\} m & Randomized \\
		Bandwidth & \{20, 40, 50, 80, 100\} MHz & Randomized\\
		Number of sub-bands & \{20, 106, 133, 217, 273\} & Randomized \\
		DL TX power & \{20, 40, 50, 80, 100\} W  & Randomized \\
		\ac{UE} antennas & \{2, 4\} & Randomized \\
		Maximum TX rank & \{2, 4\} & As per \ac{UE} ant. \\
		Maximum DL TX & 5 & Fixed \\
		\ac{UE} traffic type & \{\ac{FB}, \ac{eMBB}\} & Randomized \\
		Number \ac{FB} \acp{UE}  & \{1, 5, 10\} & Randomized \\
		Number \ac{eMBB} \acp{UE} & \{0, 10, 25, 50, 100, 200, 300\} & Randomized\\
		Speed \ac{UE} \ac{FB} & \{0.67, 10, 15, 30\} m/s & Randomized \\
		Speed \ac{UE} \ac{eMBB} & \{0.67, 1.5, 3\} m/s & Randomized\\
		\ac{UE} receiver types & \{type0, type1, type2, type3\} & Randomized \\
		Indoor probability & \{0.2, 0.4, 0.8\} & Randomized\\
		\midrule[0.3pt]%\bottomrule[1pt]
	\end{tabular}\label{table:sim_params}
\end{table}

%-------------------------------------------------------------------------
\subsection{Training Setup}
\label{appendix:F2_training_setup}
%-------------------------------------------------------------------------
We train the \ac{MORL} \ac{LA} controller agent using our \ac{D-EQL} algorithm, described in~\Cref{appendix:distributed_morl}, with a single GPU and 560 CPU cores. The learner uses Adam optimizer~\citep{KiB:17} with a learning rate of $5\times10^{-5}$, weight decay of $\nicefrac{0.02}{512}$, and default momentum terms $(\beta_{1},\beta_{2})=(0.9,0.999)$, and \ac{MSE} loss. He initialization is used for all network parameters. A soft target update policy is applied with an update factor of $0.001$ and a period of one timestep. The model synchronization period is 200 gradient iterations, and training begins after 50,000 timesteps. To reduce communication overhead between the learner and replay memory, 16 batches are prefetched per cycle. Experience and preference batches contain 512 and 128 samples, respectively.

The actor subsystem consists of 40 CPU-based rollout workers, each interacting with 14 parallel simulations (one CPU core per simulation), resulting in efficient experience generation. Each actor collects about 112 samples per second, for a total of roughly 279 million over the training horizon. The learner processes about 27,500 samples per second for gradient updates. Each actor maintains a local buffer of 2,500 samples and follows a linear epsilon-greedy strategy, decaying $\epsilon$ from 0.8 to 0.05 over 5.5 million timesteps. The agent operates with a discount factor of 1.0. Training throughput is about 53.8 batches per second, with each batch containing 65,536 samples.

Replay memory is organized as a single module with four independent shards, each capable of storing four million samples. Each shard has a fixed communication path to a designated learner shard, minimizing cross-shard delays. Prioritized experience replay is employed with parameters $\alpha=0.6$ and $\beta=0.4$ to improve sample efficiency. In total, the system runs 11,200 simulations under different random seeds to ensure reproducibility across diverse network conditions. Communication details of the distributed system are further discussed in~\Cref{appendix:hyperparams}.

Furthermore, we explore a preference space $\Omega=\Delta^{1}\triangleq \{\boldsymbol{\omega} \mid \boldsymbol{\omega} = [\omega, 1-\omega]^{\top}, \omega\in[0, 1]\}$ defined for the two-dimensional reward in~\Cref{eqn:2D_reward}. The preference space is partitioned into strata, and each actor is assigned to explore a different stratum. Preferences are then sampled from the corresponding strata following the procedure in~\Cref{alg:sls-build} (stratum construction) and~\Cref{alg:sls-sample} (stratum-based sampling). Further details on all hyperparameters used in training are summarized in~\Cref{appendix:hyperparams}.

%-------------------------------------------------------------------------
\subsection{Testing the MORL LA Controller Agent} \label{appendix:F3_MORLagent_eval}
%-------------------------------------------------------------------------
\subsubsection{Single Connectivity Service}\label{appendix:F3_single_service}

We extend the analysis presented in~\Cref{subsec:MORL_controller_LA} by further evaluating the \ac{MORL}-based \ac{LA} controller for a single connectivity service: video streaming users. Focusing on a single user class simplifies the analysis of the Pareto front achievable by the \ac{MORL} controller agent.

\Cref{fig:pareto-frontier} shows the Pareto front defined by the two-dimensional reward function in~\Cref{eqn:2D_reward} for a 3-cell deployment with 10 streaming users. Each point on the frontier is obtained from 480 independent simulations, where each radio cell employs the \ac{MORL}-based \ac{LA} controller with a fixed preference value $\omega \in [0,1]$. The parameter $\omega$ determines the trade-off between minimizing radio resources required for packet transmission and maximizing the transmitted payload size. Moving along the frontier results in different performance trade-offs in network \acp{KPI}, as detailed in~\Cref{fig:spec-eff-vs-throughput}.

At the top-right corner of~\Cref{fig:pareto-frontier,fig:spec-eff-vs-throughput}, large values of $\omega$ prioritize payload maximization (i.e., high \ac{TBS}), but at the cost of excessive radio resource consumption. In this case, the controller~agent selects overly aggressive \ac{MCS} values relative to the channel state (see~\Cref{fig:controller_MCS_distr}), targeting spectral efficiencies beyond channel capacity. This leads to frequent transmission failures (\ac{BLER} $\approx 60\%$, see~\Cref{fig:controller_BLER_distr}) and numerous retransmissions, yielding suboptimal throughput and spectral efficiency.

In contrast, when $\omega \approx 0$ (bottom-left corner of~\Cref{fig:pareto-frontier,fig:spec-eff-vs-throughput}), the controller favors conservative \ac{MCS} choices (see~\Cref{fig:controller_MCS_distr}), targeting spectral efficiencies well below channel capacity. Although this results in low resource utilization and highly reliable transmissions (\ac{BLER} $\approx 0\%$, see~\Cref{fig:controller_BLER_distr}), it under-utilizes favorable channel conditions by employing low modulation orders and code rates. Thus, the system fails to deliver higher payloads, limiting throughput and spectral efficiency.

Overall, throughput and spectral efficiency peak at $\omega \approx 0.34$ and $\omega \approx 0.5$, respectively. Beyond these values, throughput declines more rapidly than spectral efficiency due to the rising \ac{BLER}, which reduces transmission reliability. These results highlight the intrinsic tension between maximizing data rate and maintaining reliability in link adaptation, highlighting how a \ac{MORL} controller agent can be deployed to provide differentiated connectivity services. 

\begin{figure}[t!]
    \begin{subfigure}[h]{1.0\textwidth}
        \centering
        \includegraphics[width=0.6\textwidth]{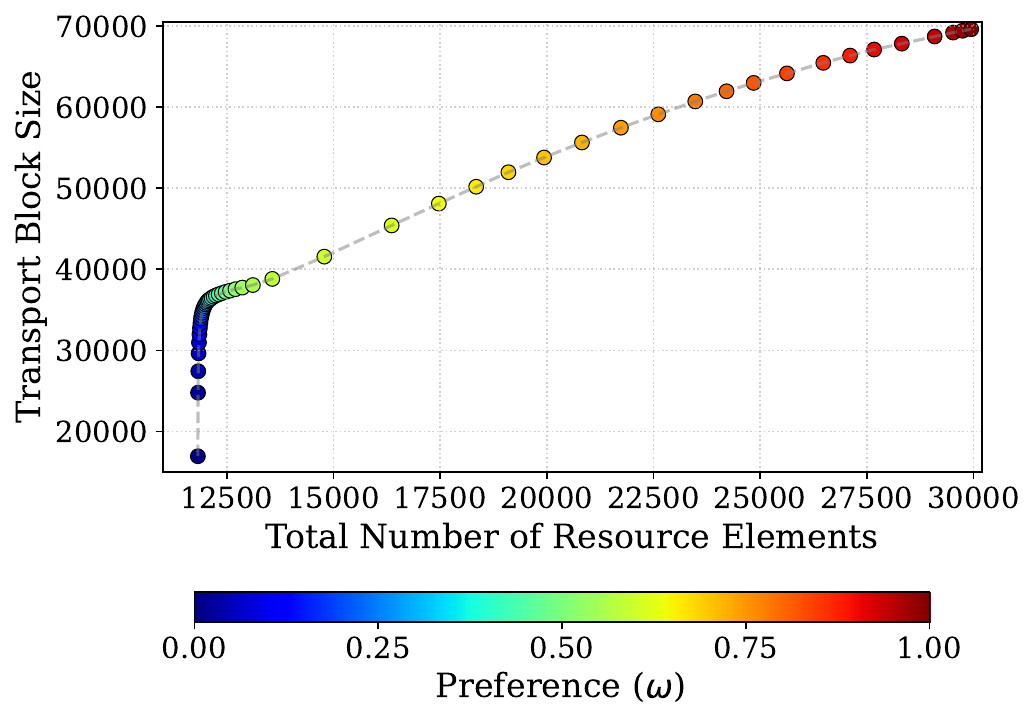}
    \end{subfigure}
    \caption{\textbf{Pareto front illustrating the trade-off between transport block size and resource utilization.} The Pareto front captures the relationship between the transport block size (vertical axis) and the total number of resource elements (horizontal axis) across a range of system configurations. Each point represents an outcome from $480$ independent simulations, computed using distinct preference vectors $\omega\in [0, 1]$, and is color-coded by the corresponding preference weight.}
    \label{fig:pareto-frontier}

    \vspace{1em} % spacing between groups of subfigures

    \begin{subfigure}[h]{1.0\textwidth}
        \centering
        \includegraphics[width=0.6\textwidth]{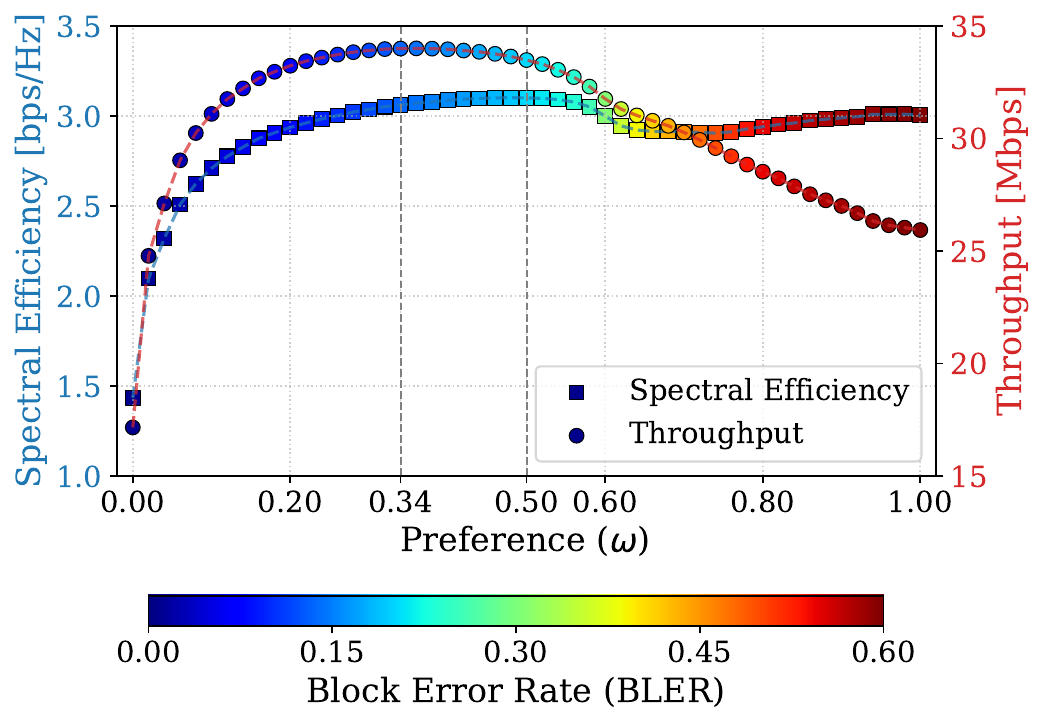}
    \end{subfigure}
    \caption{\textbf{Joint characterization of spectral efficiency and throughput under varying preference weights $\omega$, with \ac{BLER}-encoded performance.} The figure presents the trade-off between spectral efficiency (squares, left axis) and throughput (circles, right axis) as a function of the preference weight $\omega$, which governs the optimization objective. Data points are color-coded based on \ac{BLER}, with cooler hues indicating lower error rates. Dashed vertical lines denote peaks in the performance~trends.}
    \label{fig:spec-eff-vs-throughput}
\end{figure}

\begin{figure*}[t!]
    \centering
    \begin{subfigure}{0.49\textwidth}
        \centering
        \includegraphics[width=\linewidth]{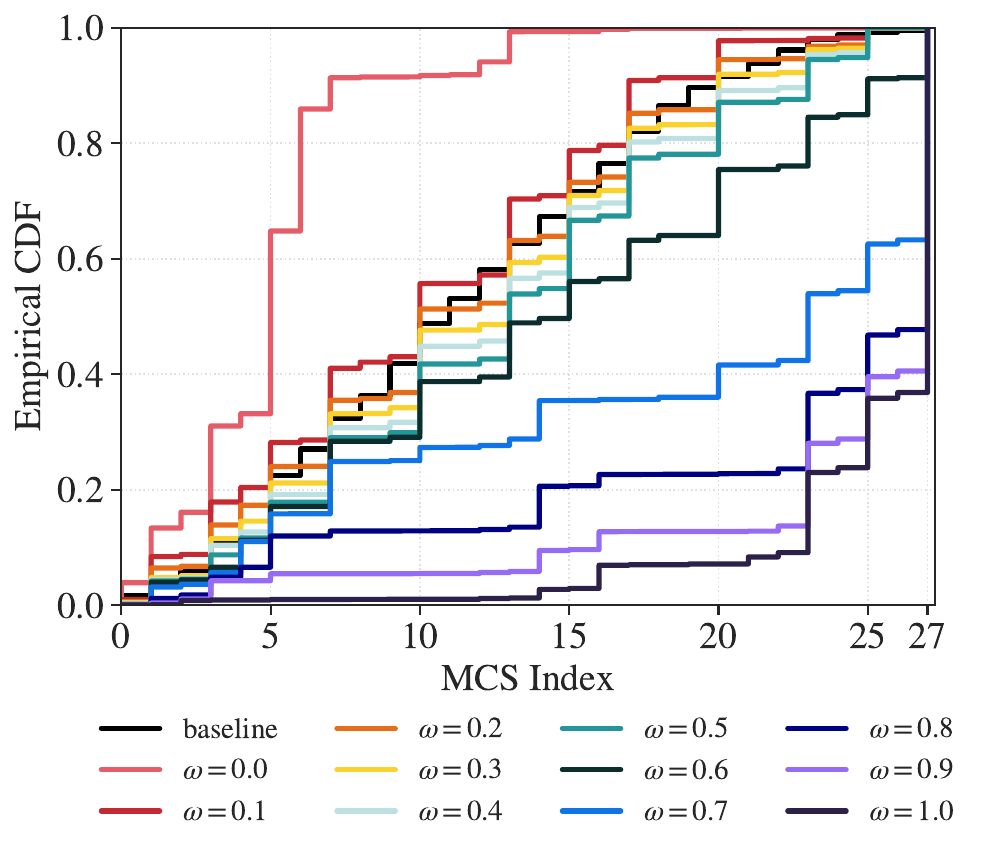}
        \caption{Action (\ac{MCS} index) distributions for various $\omega$.}
        \label{fig:controller_MCS_distr}
    \end{subfigure}
    \hfill
    \begin{subfigure}{0.49\textwidth}
        \centering
        \includegraphics[width=\linewidth]{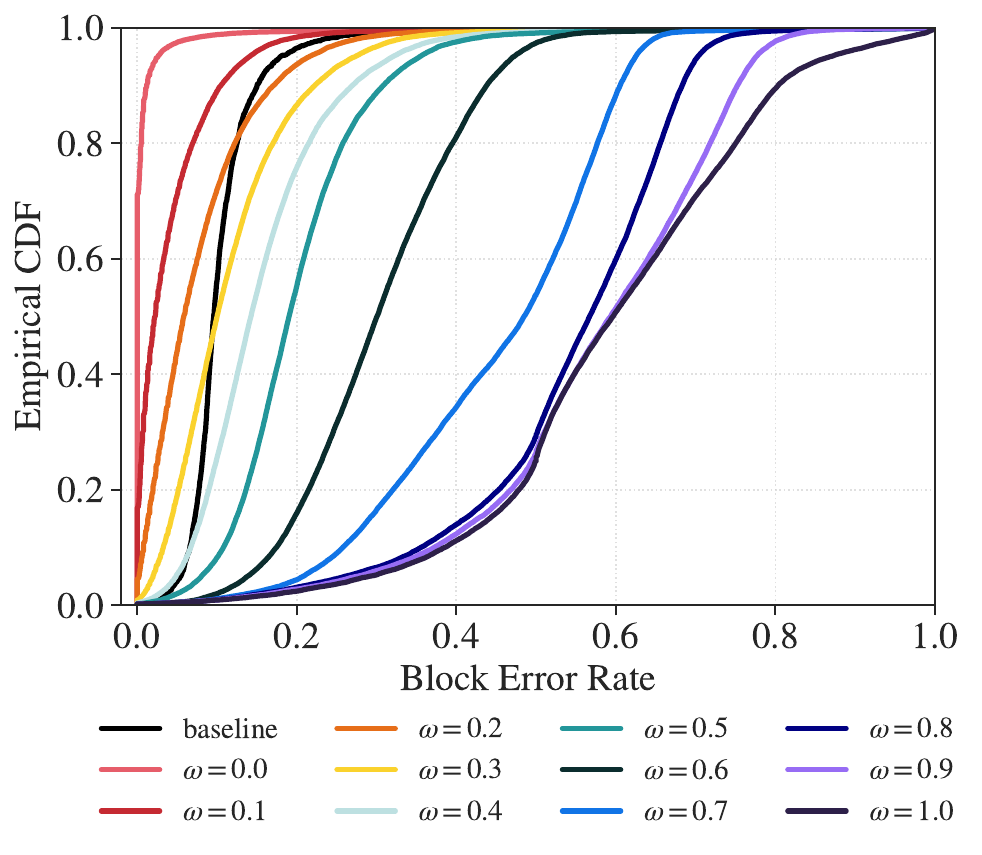}
        \caption{Block error rate distributions for various $\omega$.}
        \label{fig:controller_BLER_distr}
    \end{subfigure}
    \caption{Controller agent behavior for different preference values $\omega$. }
    \label{fig:controller_policies}
\end{figure*}

\Cref{fig:controller_policies} further illustrates how the controller policy changes by selecting different preference values. In particular,~\Cref{fig:controller_MCS_distr} illustrate the action (i.e., \ac{MCS} index) distribution induced by different preference values $\omega$. For example, it clearly highlights how small values of $\omega$ induce a link adaptation policy that selects overly conservative \ac{MCS} index values relative to the channel state, thus aiming for transmissions with low spectral efficiency (i.e, characterized by low modulation order and code rate). Although this makes the transmission very robust, as demonstrated by the corresponding \ac{BLER} distribution in~\Cref{fig:controller_BLER_distr}, such policy leads to low data throughput.

Conversely,~\Cref{fig:controller_MCS_distr} also shows how large preference values $\omega\approx 1$ induce a link adaptation policy that selects overly aggressive \ac{MCS} index values relative to the channel state, thus aiming for transmissions with too high spectral efficiency (i.e, characterized by high modulation order and code rates). This makes the transmissions over-the-air unreliable, as demonstrated by the corresponding \ac{BLER} distribution in~\Cref{fig:controller_BLER_distr}.

%-------------------------------------------------------------------------
\subsubsection{Multi Connectivity Services with QoS Differentiation}
%-------------------------------------------------------------------------

We consider a practical scenario with two service applications with distinct \ac{QoS} profiles concurrently sharing the resources of a radio cell: \emph{real-time gaming} and \emph{web browsing} users. \Cref{tab:QoS_profile_comparison} characterizes their \ac{QoS} profile in terms of purpose, service type, \ac{DSCP} value, \ac{5QI} value and \ac{QoS} features.

\begin{table}[h!]
\centering
\renewcommand{\arraystretch}{1.5}
\setlength{\tabcolsep}{7.5pt}
\footnotesize
\caption{\ac{QoS} Profile Comparison: Real-time gaming vs Web Browsing}
\begin{tabular}{p{2cm} p{5cm} p{5cm}}
\toprule[1pt]\midrule[0.3pt]
\textbf{Aspect} & \textbf{Real-time gaming} & \textbf{Web browsing} \\
\midrule
\textbf{Purpose} & Real-time, delay-sensitive traffic & Delay-tolerant, no bandwidth guarantees \\
\textbf{Service type} & \leavevmode\Ac{EF} & \leavevmode\Ac{BE} \\
\textbf{\ac{DSCP} value} & \leavevmode\ac{EF} (46) & \leavevmode\ac{BE} (0) \iffalse or CS0\fi \\
\textbf{\ac{5QI}} & 3 & 9 \\
\textbf{\ac{QoS} features} &  \leavevmode\Ac{GBR} \newline Ultra-low latency \newline Low jitter \newline \Ac{PDB} $ \approx 50\text{ms}$ \newline \ac{PER} $\approx 1 \times 10^{-3}$ & \leavevmode\Ac{non-GBR} \newline No strict latency  \newline No strict jitter \newline \ac{PDB} $\approx 300\text{ms}$ \\
\bottomrule
\end{tabular} \label{tab:QoS_profile_comparison}
\end{table}

\textbf{Real-time gaming} traffic consists of continuous, high-frequency bidirectional streams, often transmitted over \ac{UDP} based protocols to support real-time video rendering and user input feedback. This type of traffic demands substantially higher bitrates (ranging from 5 to 25 Mbps), ultra-low latency, and minimal jitter to maintain responsive and seamless game-play. As such, real-time gaming is classified as \ac{GBR} traffic and expedited forwarding service, necessitating stringent \ac{QoS} settings, including \ac{5QI} = 3 and \ac{DSCP} values like \ac{EF} (46), corresponding to a \ac{PDB} of $\approx 50$ ms and \ac{PER} $\approx 10^{-3}$ (cf.~Table~5.7.4-1, \cite{3GPP23501}).

\textbf{Web browsing} traffic is instead elastic, delay-tolerant, and bursty, following a request-response model (like the \ac{HTTP}) over reliable \ac{TCP} connections. It generally demands low to moderate bitrates (typically below 1 Mbps) and is relatively insensitive to latency and jitter, making it tolerant of network delays. Web browsing is typically classified as  \ac{non-GBR}  traffic, associated with \ac{5QI} = 9 and \ac{DSCP} values such as \ac{BE} (0), corresponding to a \ac{PDB} of $\approx 300$ ms and \ac{PER} $\approx 10^{-6}$.

\textbf{The evaluation scenario} consists of a dense-urban deployment comprising three 3-sector sites operating at 3.5 GHz with a 100 MHz bandwidth with inter-site distance of 167 meters to ensure full uplink coverage across the simulation area. Traffic is predominantly downlink-oriented, with minimal uplink activity. Real-time gaming users remain active throughout the simulation duration, whereas web browsing users follow a Poisson arrival process with a distribution modeled to fit realistic field data patterns and depart the simulation upon completing their downloads (e.g., webpage, email, etc.). 

Unlike the single-service application considered in~\Cref{appendix:F3_single_service}, the controller agent here applies a different preference vector to each service application:
$\boldsymbol{\omega}_{\mathrm{g}} = [\omega_{\mathrm{g}}, 1-\omega_{\mathrm{g}}]^T$ and $\boldsymbol{\omega}_{\mathrm{w}} = [\omega_{\mathrm{w}}, 1-\omega_{\mathrm{w}}]^T$. Like before, we analyze how shifting the \ac{MORL} controller policy along Pareto front defined by the two reward components in~\Cref{eqn:2D_reward}, by tuning $\boldsymbol{\omega}_{\mathrm{g}}$ or $\boldsymbol{\omega}_{\mathrm{w}}$, produces different trade-offs in various performance KPIs. Under these settings, the values achievable for a performance \ac{KPI} $g(\cdot)$ of each service application depends on both preference vectors, i.e., $g_g = g_g(\boldsymbol{\omega}_{\mathrm{g}}, \boldsymbol{\omega}_{\mathrm{w}})$ $g_w = g_w(\boldsymbol{\omega}_{\mathrm{g}}, \boldsymbol{\omega}_{\mathrm{w}})$.  

\Cref{fig:LowLoad_KPIs} to \Cref{fig:VeryHighLoad_KPIs} present the achievable user experience for real-time gaming and web browsing services in terms of three \acp{KPI} that closely relate to their \ac{QoS} profile: user throughput, latency, and \ac{BLER}. Each figure refers to one of the four traffic load scenarios, with a mixture of indoor and outdoor users, summarized in~\Cref{tab:sim_params_1}. Each figure also depicts the average \ac{MCS} value selected by the \ac{MORL} controller, showing how the controller agent applies a different policy to each service application for different combinations of preference vectors $\boldsymbol{\omega}_{\mathrm{g}}$ or $\boldsymbol{\omega}_{\mathrm{w}}$ and network load conditions. 

\begin{table}[t]
    \small
	\caption{\ac{RAN} environment simulation parameters.} % title name of the table
	\centering % centering table
    \renewcommand{\arraystretch}{1.3}
    \begin{tabular}{lccccc}
		\toprule[1pt]\midrule[0.3pt]
		\textbf{Load scenario} & \multicolumn{2}{c}{\textbf{Number of gaming users}} & \multicolumn{2}{c}{\textbf{Web users arrival  rate}} & \textbf{Performance \acp{KPI}} \\
         & Indoor & Outdoor & Indoor & Outdoor & \\[0.5ex]
		\midrule
        \textbf{Low} & 12 & 6 & 3.15 & 1.35 & \Cref{fig:LowLoad_KPIs}  \\
        \textbf{Medium} & 24 & 12 & 6.3 & 2.7 & \Cref{fig:MediumLoad_KPIs}  \\
        \textbf{High} & 48 & 24 & 12.6 & 5.4 & \Cref{fig:HighLoad_KPIs}  \\
        \textbf{Very high} & 72 & 36 & 18.9 & 8.1 &  \Cref{fig:VeryHighLoad_KPIs} \\
		\midrule[0.3pt]%\bottomrule[1pt]
	\end{tabular}\label{tab:sim_params_1}
\end{table}

For example, let us analyze the throughput distributions for real-time gaming uses (\Cref{subfig:LowLoad_a} to \Cref{subfig:VeryHighLoad_a}) and web browsing user (\Cref{subfig:LowLoad_b} to \Cref{subfig:VeryHighLoad_b}) for the various scenarios. For low and medium low load conditions, cf. \Cref{subfig:LowLoad_a}-\ref{subfig:LowLoad_b} and \Cref{subfig:MediumLoad_a}-\ref{subfig:MediumLoad_b}, respectively the mean throughput distribution of the two services shows similarities due to the abundance of radio resources compared to traffic load. The difference in mean throughput magnitude between the two type of services (i.e., Mbps vs Kbps) can be explained by the difference in traffic: continuous video streaming vs sporadic downloads of small packets.

At high and very high load conditions, cf.~\Crefrange{subfig:HighLoad_a}{subfig:HighLoad_b} and \Crefrange{subfig:VeryHighLoad_a}{subfig:VeryHighLoad_b}, respectively, the two distributions of throughput start showing significant differences, clearly revealing how each service achieves the best mean throughput with different combinations of preference values $(\boldsymbol{\omega}_{\mathrm{g}}, \boldsymbol{\omega}_{\mathrm{w}})$. As the traffic load becomes very high, the region of preference values $(\boldsymbol{\omega}_{\mathrm{g}}, \boldsymbol{\omega}_{\mathrm{w}})$ that optimizes the throughput of each service shrinks into a smaller and well defined area. Furthermore, since in these scenarios more users share the same amount of radio resources, both services achieve lower throughput.

Similar trends can be observed for latency (expressed as \ac{RTT} for real-time gaming users and as webpage load time for web browsing users, respectively), and block error rate. The conditions observed in different service \ac{KPI} in~\Cref{fig:LowLoad_KPIs} to \Cref{fig:VeryHighLoad_KPIs} can be related to the constraints $g_i(\boldsymbol{\omega})\leq b_i$ that can be required to be fulfilled by a service intent in the optimization problem~(\ref{eq:optimization}) solved by the optimizer agent to dynamically adapt the preference vectors for each service applications. For instance, in~\Cref{subsec:triadic_agents_workflow_validation} we presented an example with a video streaming service requiring a minimum of 7 Mbps per active user (i.e., $g_{i, thr}(\boldsymbol{\omega})\geq 7$). This threshold is rated as good for most real-time gaming
applications at 720p and 1080p resolutions, and excellent for video streaming, given that typical requirements range from 5 Mbps for HD to 15 Mbps for 4K content. 

% ----------------------------------------------------------------------------
% Low load
% ----------------------------------------------------------------------------
\begin{figure}[p]
    \centering
    
    % Row 1
    \begin{subfigure}[t]{0.425\textwidth}
        \centering
        \includegraphics[width=\linewidth]{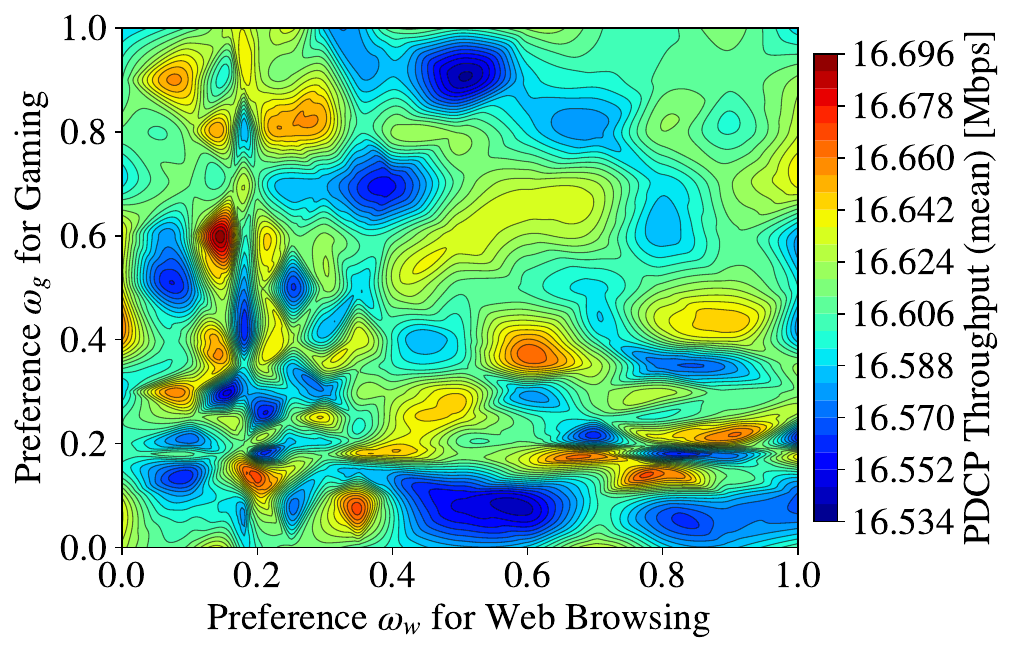}
        \caption{Real-time gaming: Mean PDCP Throughput}
        \label{subfig:LowLoad_a}
    \end{subfigure}
    \hfill
    \begin{subfigure}[t]{0.425\textwidth}
        \centering
        \includegraphics[width=\linewidth]{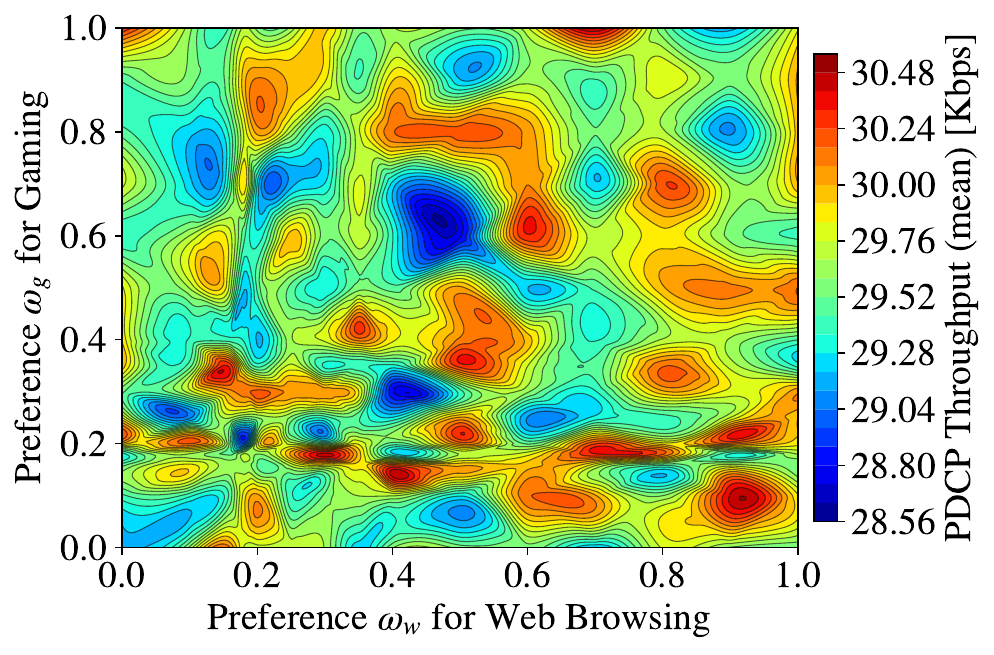}
        \caption{Web browsing: Mean PDCP Throughput}
        \label{subfig:LowLoad_b}
    \end{subfigure}

    % Row 2
    \begin{subfigure}[t]{0.425\textwidth}
        \centering
        \includegraphics[width=\linewidth]{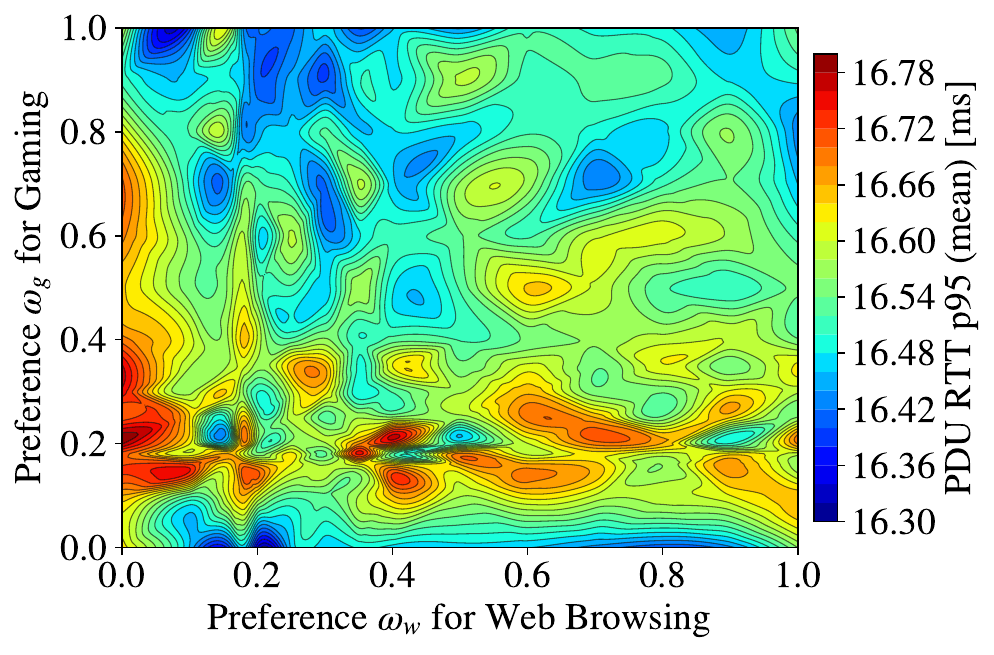}
        \caption{Real-time gaming: 95th Percentile PDU RTT}
        \label{subfig:LowLoad_c}
    \end{subfigure}
    \hfill
    \begin{subfigure}[t]{0.425\textwidth}
        \centering
        \includegraphics[width=\linewidth]{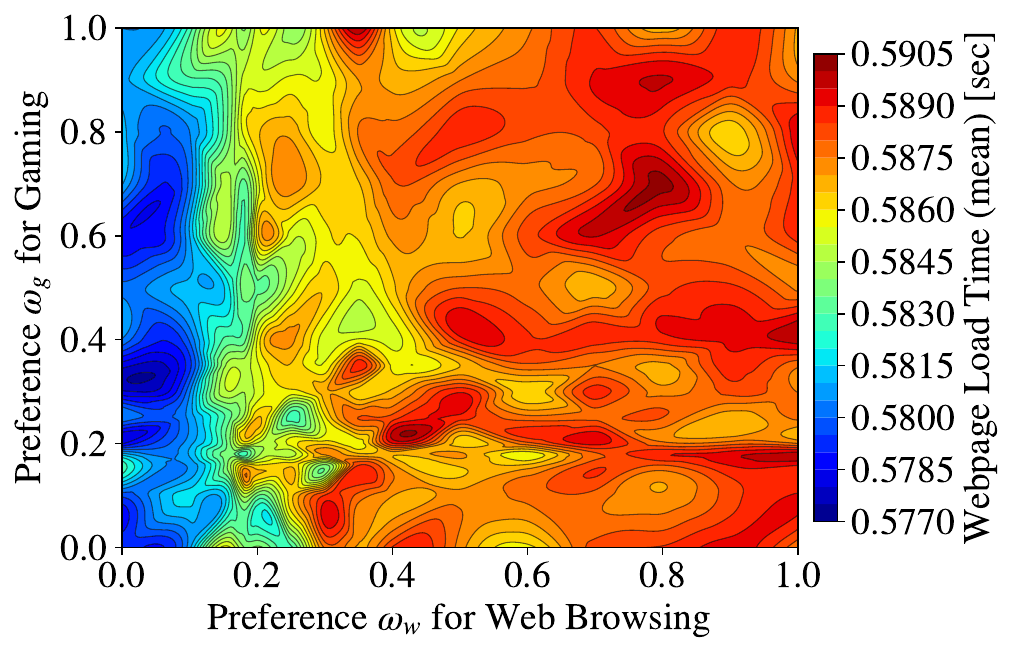}
        \caption{Web Browsing: Loading time}
        \label{subfig:LowLoad_d}
    \end{subfigure}

    % Row 3
    \begin{subfigure}[t]{0.425\textwidth}
        \centering
        \includegraphics[width=\linewidth]{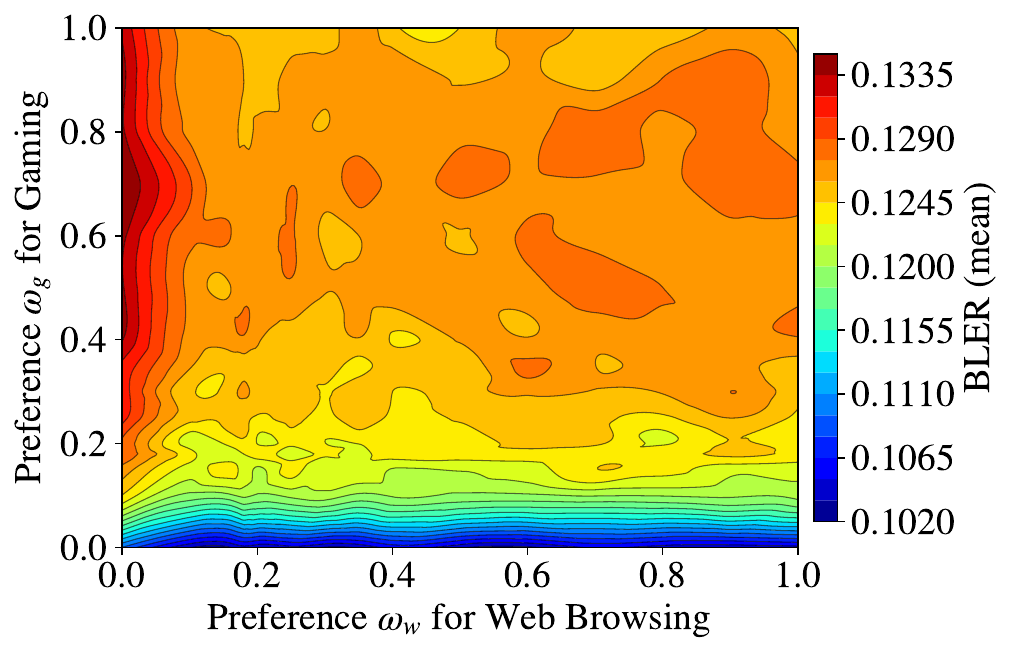}
        \caption{Real-time gaming: Mean \ac{BLER}.}
        \label{subfig:LowLoad_e}
    \end{subfigure}
    \hfill
    \begin{subfigure}[t]{0.425\textwidth}
        \centering
        \includegraphics[width=\linewidth]{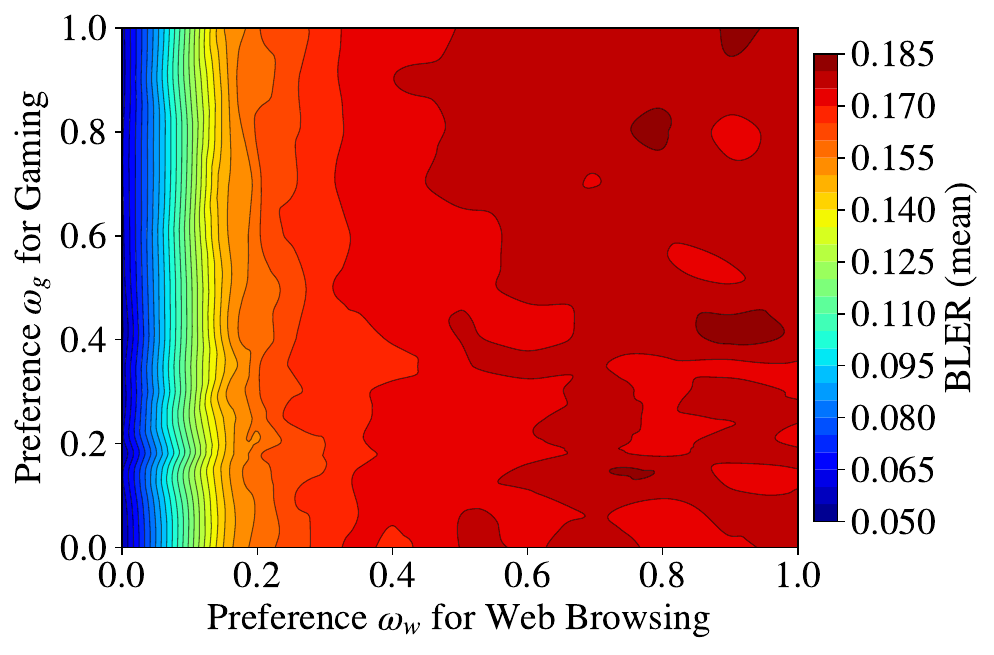}
        \caption{Web browsing: Mean \ac{BLER}}
        \label{subfig:LowLoad_f}
    \end{subfigure}

    % Row 4
    \begin{subfigure}[t]{0.425\textwidth}
        \centering
        \includegraphics[width=\linewidth]{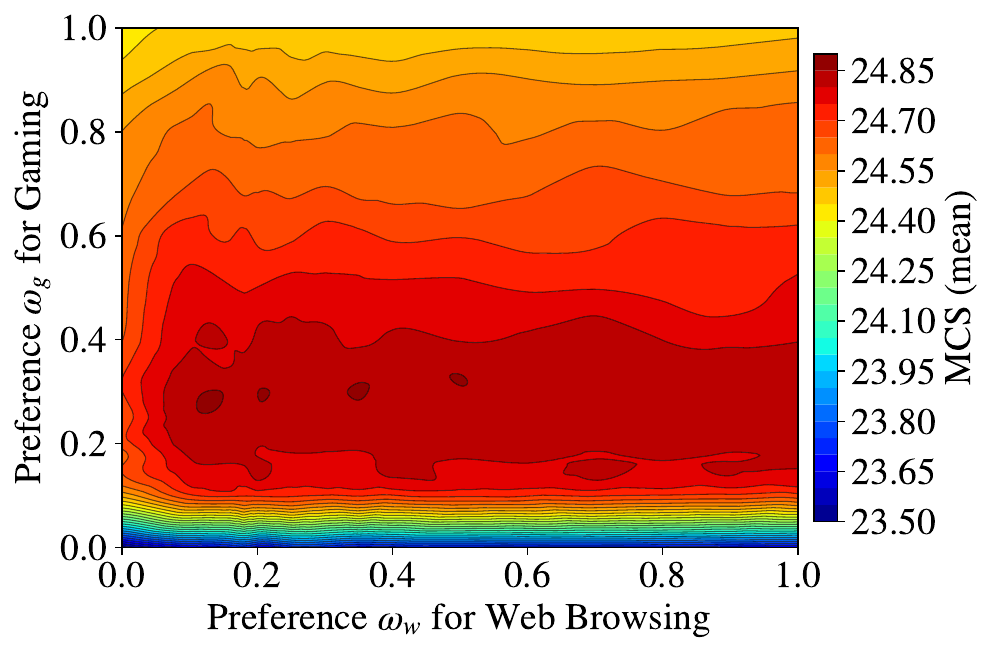}
        \caption{Real-time gaming: \ac{MCS} selection}
        \label{subfig:LowLoad_g}
    \end{subfigure}
    \hfill
    \begin{subfigure}[t]{0.425\textwidth}
        \centering
        \includegraphics[width=\linewidth]{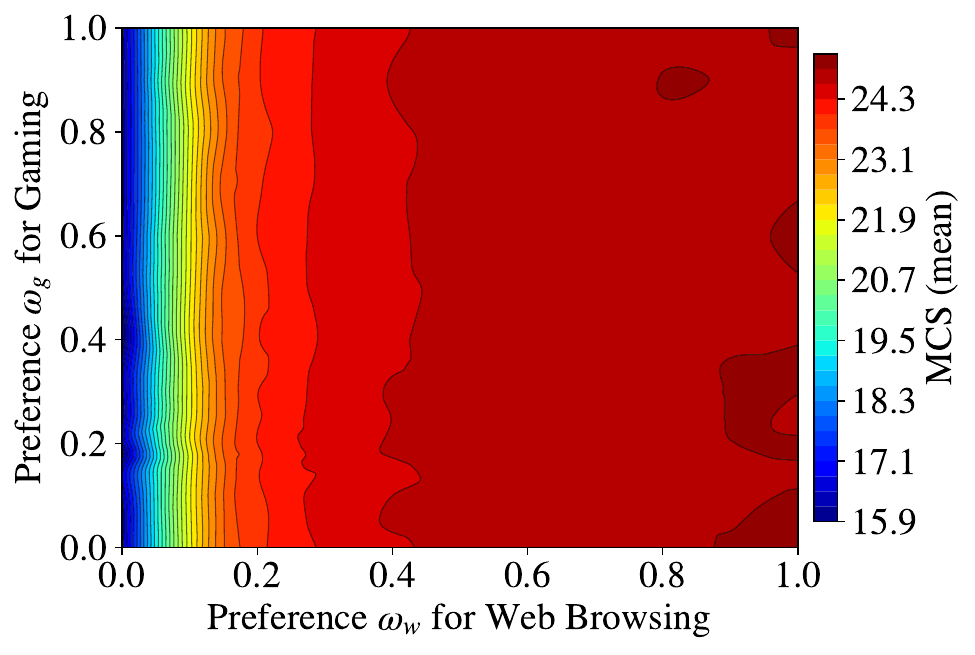}
        \caption{Web browsing: \ac{MCS} selection}
        \label{subfig:LowLoad_h}
    \end{subfigure}

    \caption{\textbf{Impact of user preference weights on the performance of real-time gaming and browsing users in low network load conditions.}
    Each subplot shows a distinct \ac{QoS} metric for real-time gaming users (left column) and for web browsing users (right column) under varying preference weights ($\omega_g$, $\omega_w$) reflecting resource allocation priorities for the two connectivity services. Metrics include: (a)-(b) mean user throughput, (c)-(d) mean latency (defined according to the service), (e)-(f) mean \ac{BLER}. Furthermore, (g)-(h) show the action (mean \ac{MCS}) distribution under ($\omega_g$, $\omega_w$).}
    \label{fig:LowLoad_KPIs}
\end{figure}

% ----------------------------------------------------------------------------
% Medium load
% ----------------------------------------------------------------------------

\begin{figure}[p]
    \centering
    
    % Row 1
    \begin{subfigure}[t]{0.425\textwidth}
        \centering
        \includegraphics[width=\linewidth]{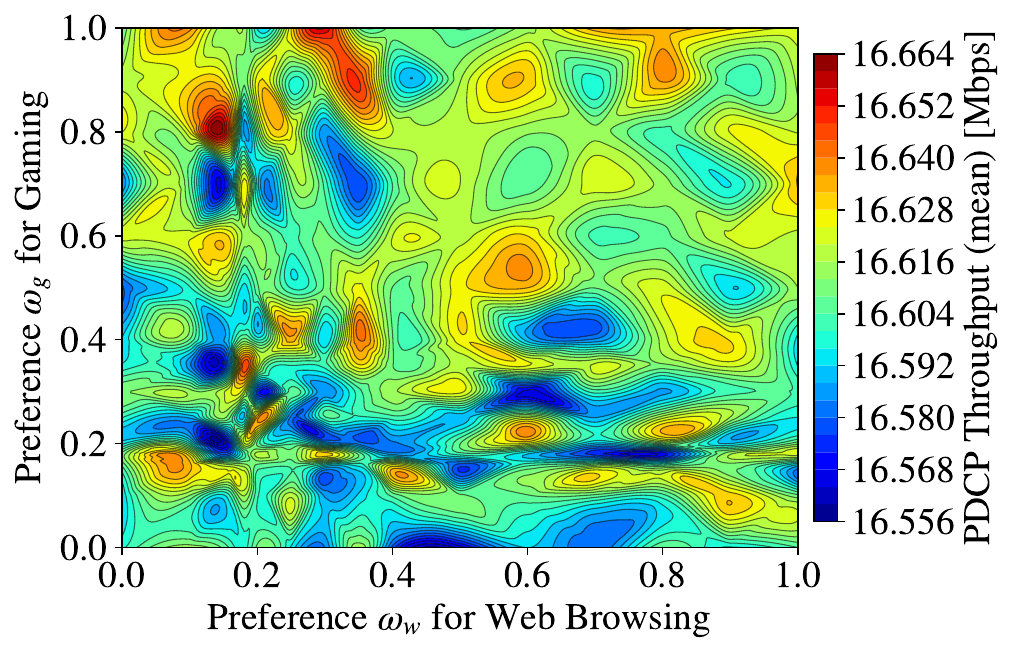}
        \caption{Real-time gaming: Mean PDCP Throughput}
        \label{subfig:MediumLoad_a}
    \end{subfigure}
    \hfill
    \begin{subfigure}[t]{0.425\textwidth}
        \centering
        \includegraphics[width=\linewidth]{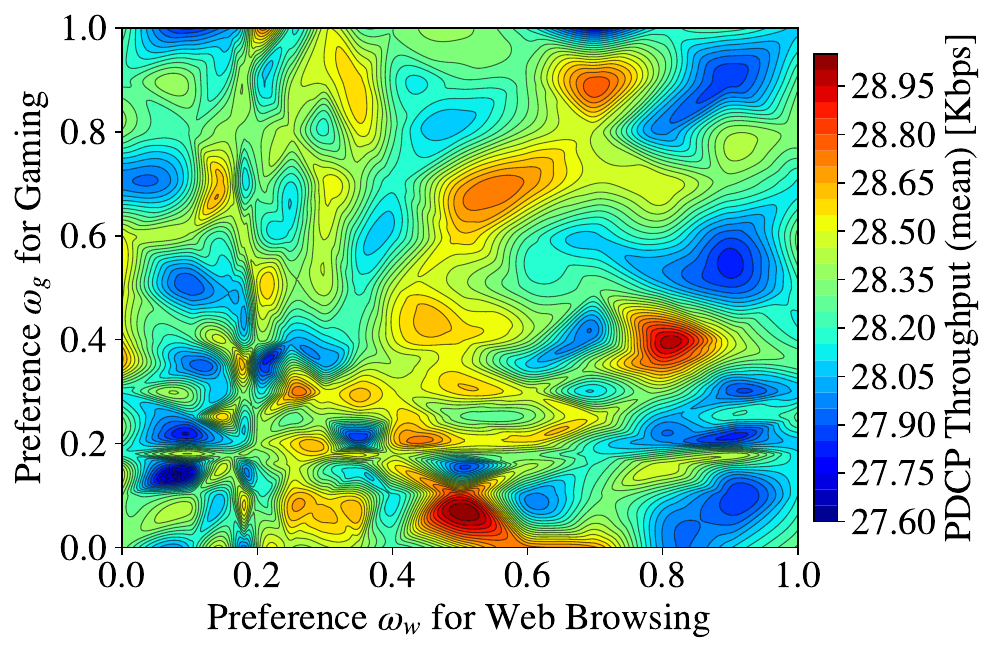}
        \caption{Web browsing: Mean PDCP Throughput}
        \label{subfig:MediumLoad_b}
    \end{subfigure}

    % Row 2
    \begin{subfigure}[t]{0.425\textwidth}
        \centering
        \includegraphics[width=\linewidth]{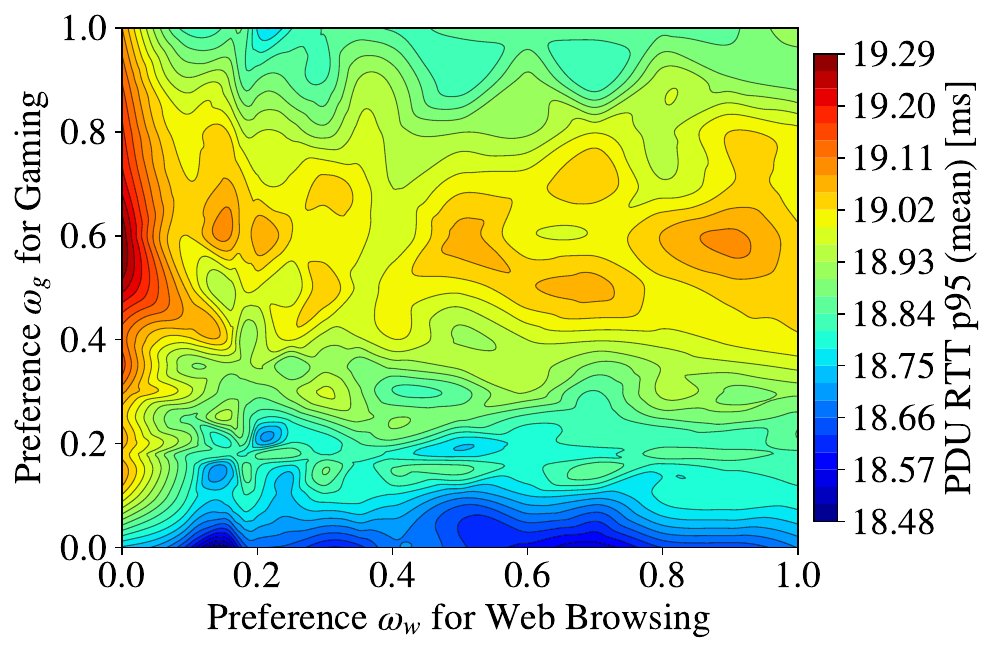}
        \caption{Real-time gaming: 95th Percentile PDU RTT}
        \label{subfig:MediumLoad_c}
    \end{subfigure}
    \hfill
    \begin{subfigure}[t]{0.425\textwidth}
        \centering
        \includegraphics[width=\linewidth]{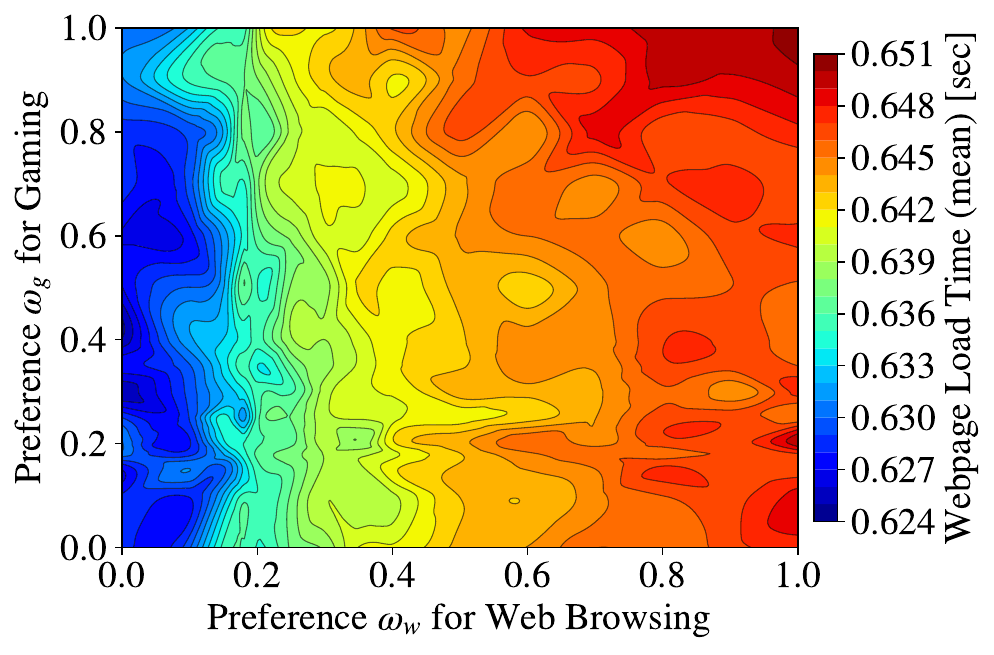}
        \caption{Web Browsing: Loading time}
        \label{subfig:MediumLoad_d}
    \end{subfigure}

    % Row 3
    \begin{subfigure}[t]{0.425\textwidth}
        \centering
        \includegraphics[width=\linewidth]{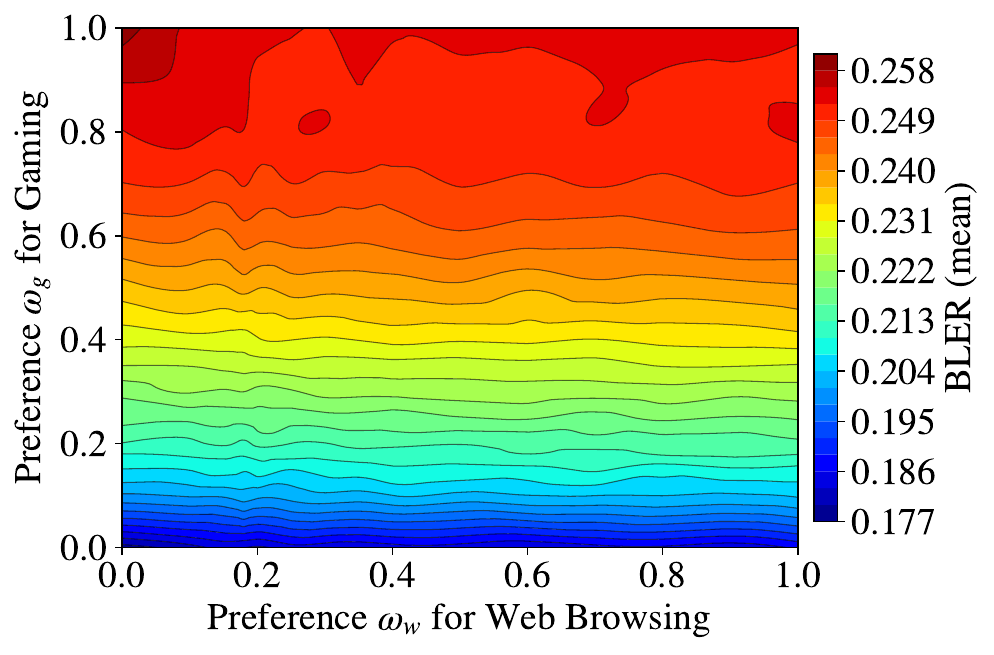}
        \caption{Real-time gaming: Mean \ac{BLER}.}
        \label{subfig:MediumLoad_e}
    \end{subfigure}
    \hfill
    \begin{subfigure}[t]{0.425\textwidth}
        \centering
        \includegraphics[width=\linewidth]{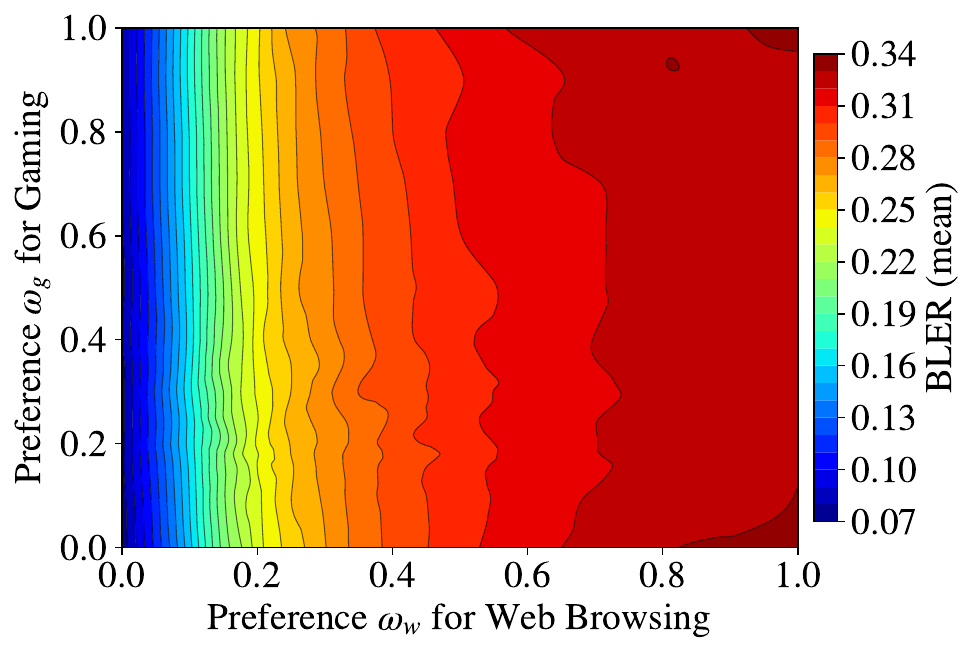}
        \caption{Web browsing: Mean \ac{BLER}}
        \label{subfig:MediumLoad_f}
    \end{subfigure}

    % Row 4
    \begin{subfigure}[t]{0.425\textwidth}
        \centering
        \includegraphics[width=\linewidth]{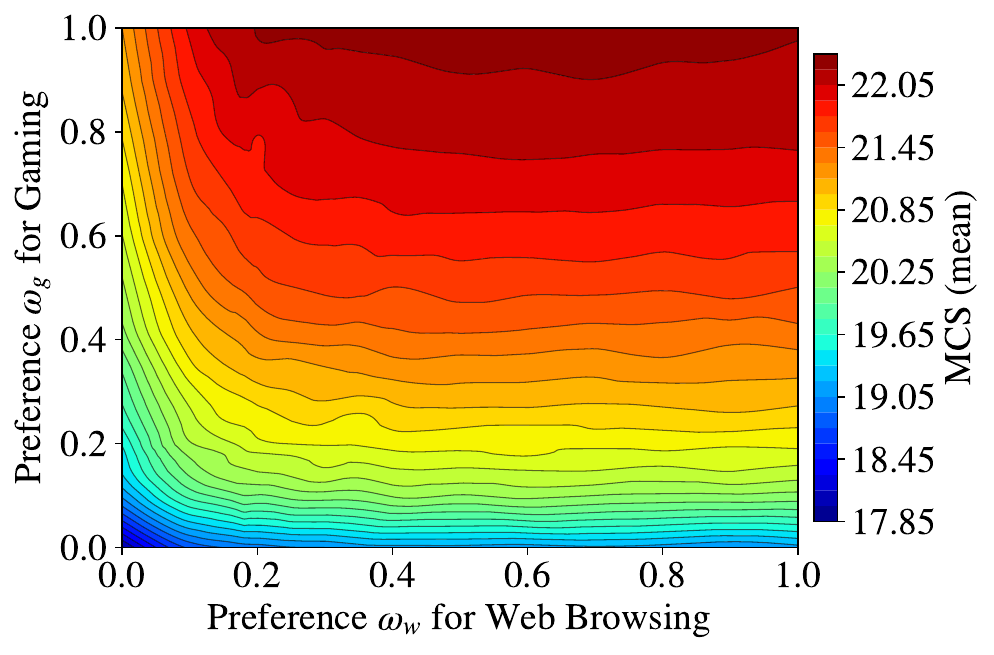}
        \caption{Real-time gaming: \ac{MCS} selection}
        \label{subfig:MediumLoad_g}
    \end{subfigure}
    \hfill
    \begin{subfigure}[t]{0.425\textwidth}
        \centering
        \includegraphics[width=\linewidth]{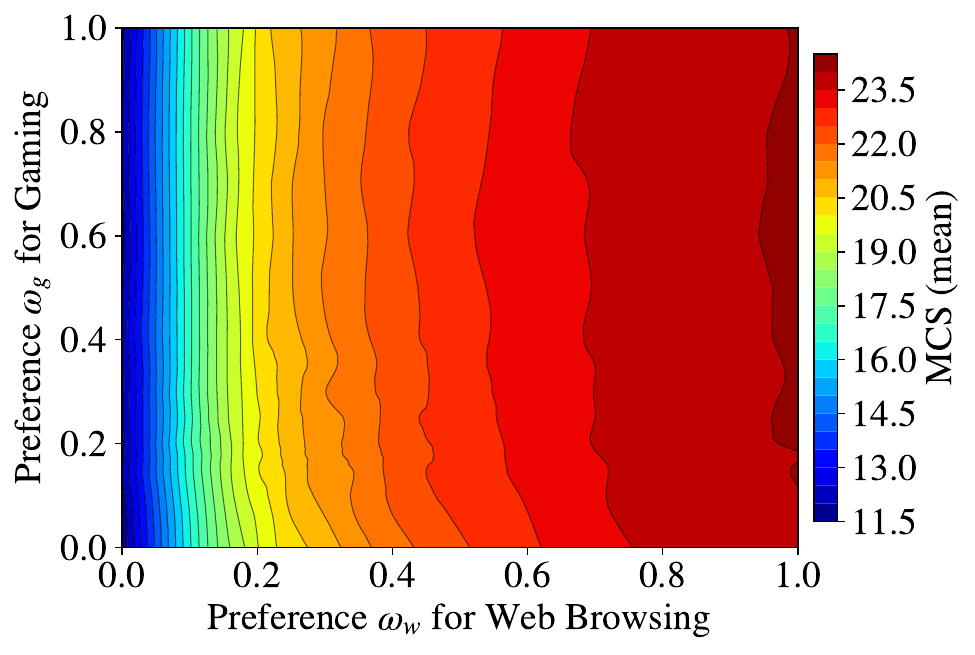}
        \caption{Web browsing: \ac{MCS} selection}
        \label{subfig:MediumLoad_h}
    \end{subfigure}
    \caption{\textbf{Impact of user preference weights on the performance of real-time gaming and browsing users in medium network load conditions.}
    Each subplot shows a distinct \ac{QoS} metric for real-time gaming users (left column) and for web browsing users (right column) under varying preference weights ($\omega_g$, $\omega_w$) reflecting resource allocation priorities for the two connectivity services. Metrics include: (a)-(b) mean user throughput, (c)-(d) mean latency (defined according to the service), (e)-(f) mean \ac{BLER}. Furthermore, (g)-(h) show the action (mean \ac{MCS}) distribution under ($\omega_g$, $\omega_w$).}
    \label{fig:MediumLoad_KPIs}
\end{figure}

% ----------------------------------------------------------------------------
% High load
% ----------------------------------------------------------------------------

\begin{figure}[p]
    \centering
    
    % Row 1
    \begin{subfigure}[t]{0.425\textwidth}
        \centering
        \includegraphics[width=\linewidth]{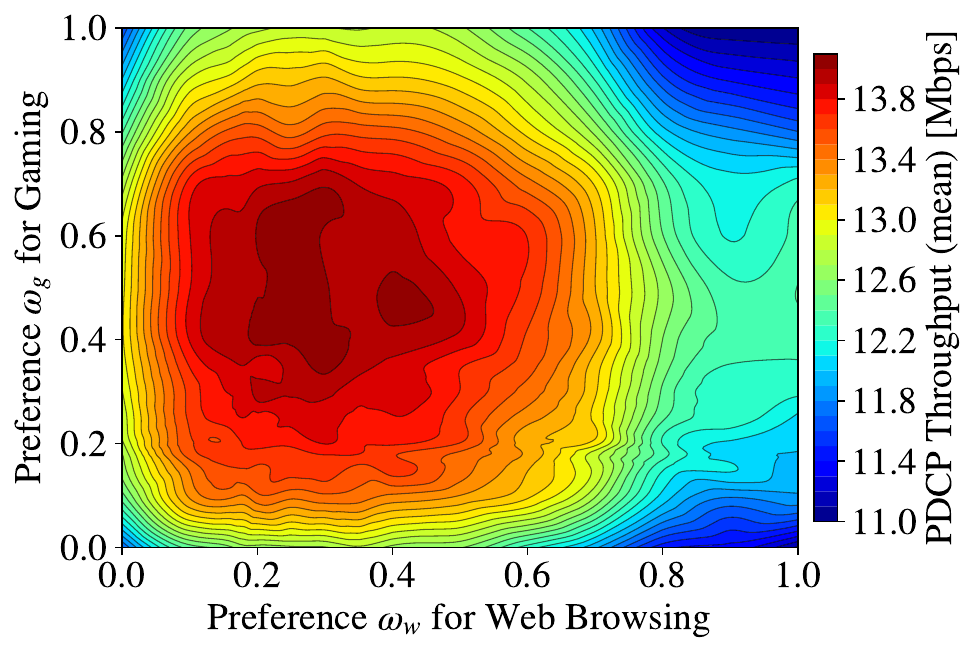}
        \caption{Real-time gaming: Mean PDCP Throughput}
        \label{subfig:HighLoad_a}
    \end{subfigure}
    \hfill
    \begin{subfigure}[t]{0.425\textwidth}
        \centering
        \includegraphics[width=\linewidth]{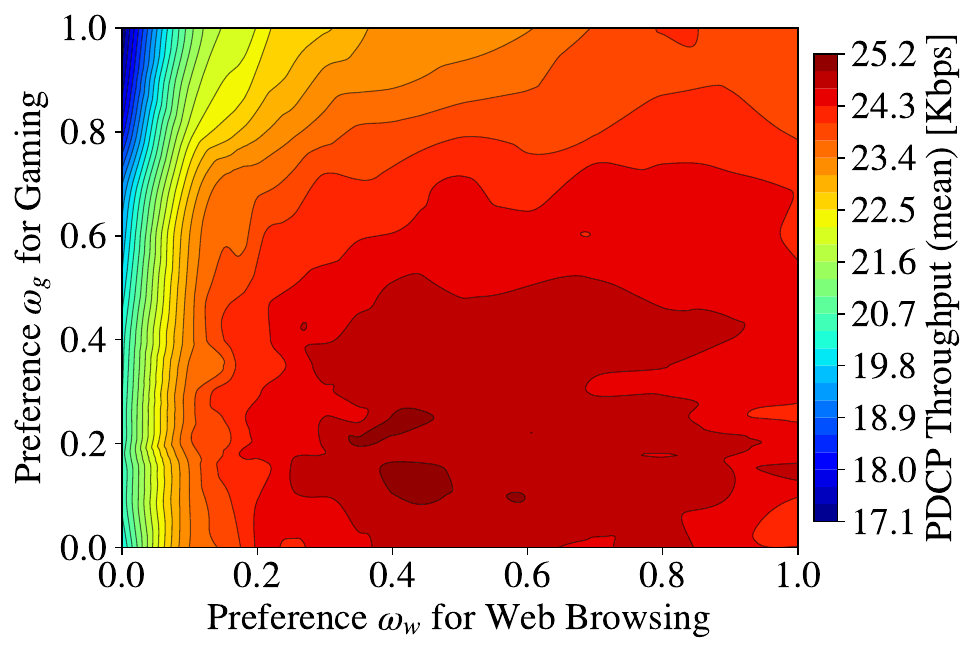}
        \caption{Web browsing: Mean PDCP Throughput}
        \label{subfig:HighLoad_b}
    \end{subfigure}

    % Row 2
    \begin{subfigure}[t]{0.425\textwidth}
        \centering
        \includegraphics[width=\linewidth]{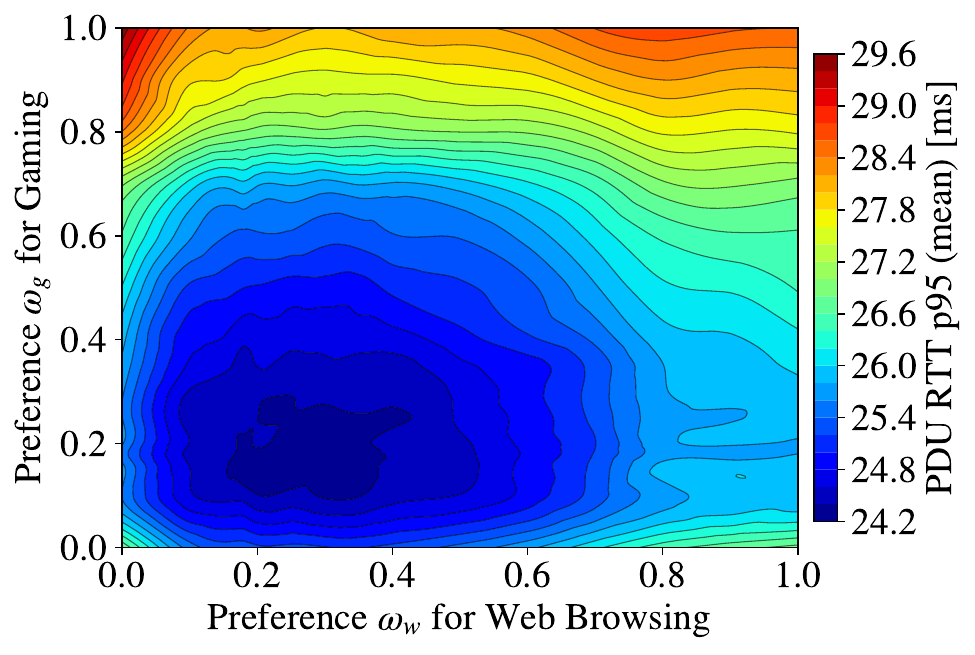}
        \caption{Real-time gaming: 95th Percentile PDU RTT}
        \label{subfig:HighLoad_c}
    \end{subfigure}
    \hfill
    \begin{subfigure}[t]{0.425\textwidth}
        \centering
        \includegraphics[width=\linewidth]{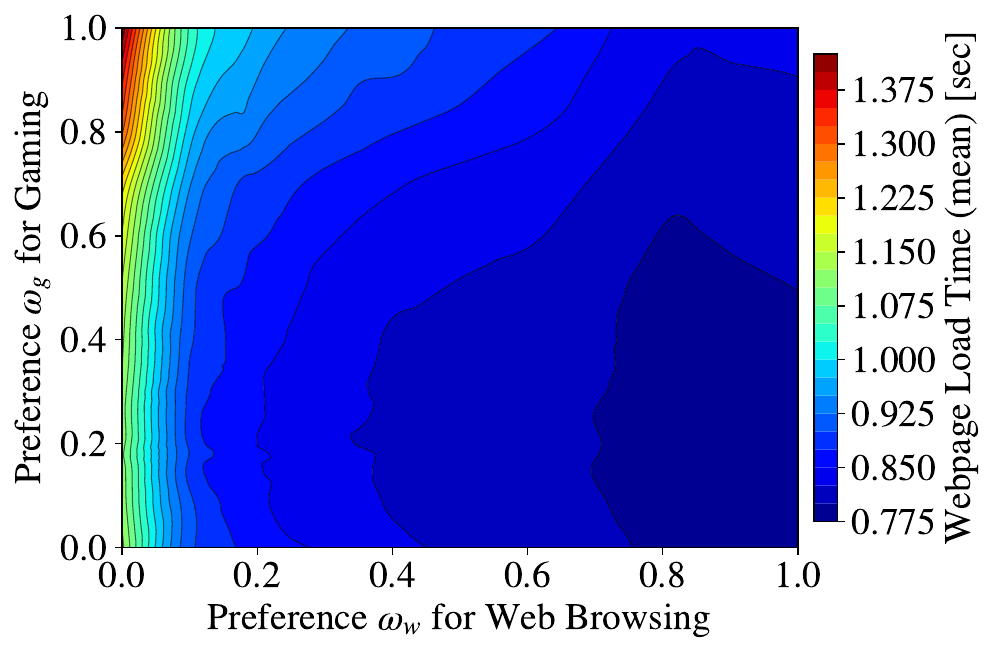}
        \caption{Web Browsing: Loading time}
        \label{subfig:HighLoad_d}
    \end{subfigure}

    % Row 3
    \begin{subfigure}[t]{0.425\textwidth}
        \centering
        \includegraphics[width=\linewidth]{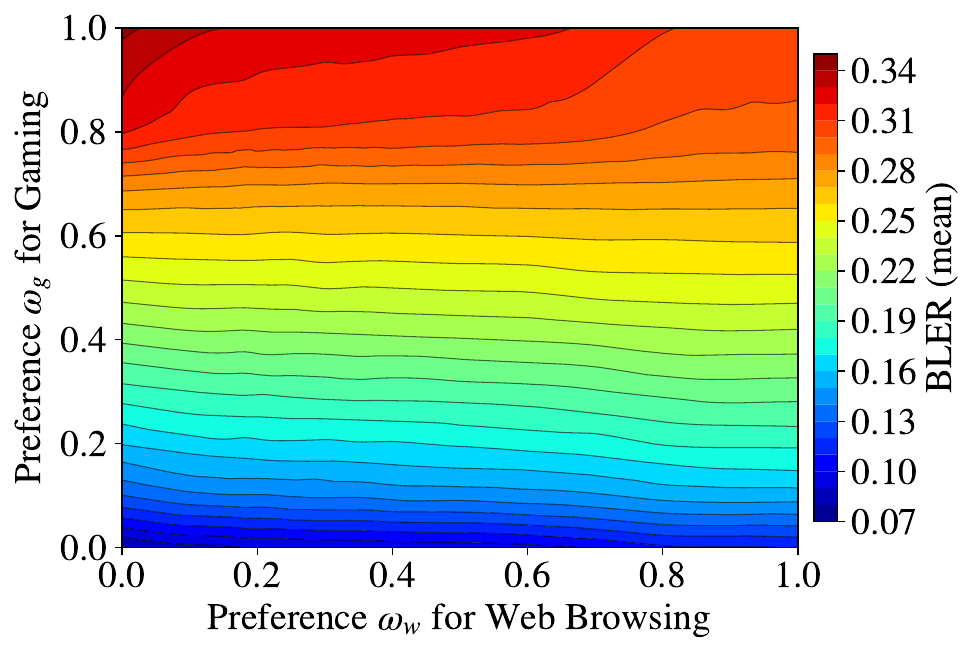}
        \caption{Real-time gaming: Mean \ac{BLER}.}
        \label{subfig:HighLoad_e}
    \end{subfigure}
    \hfill
    \begin{subfigure}[t]{0.425\textwidth}
        \centering
        \includegraphics[width=\linewidth]{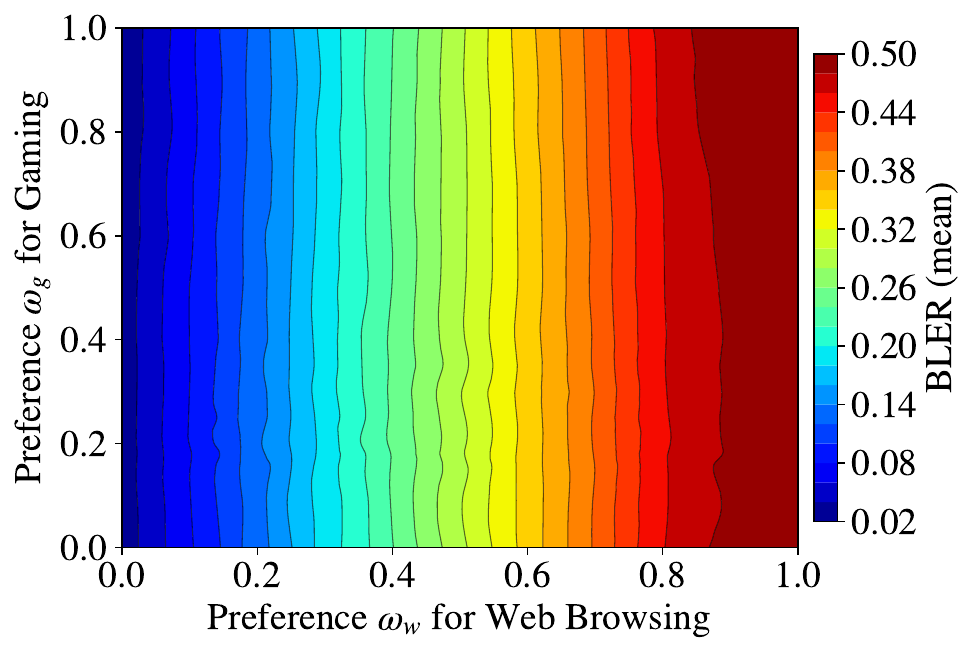}
        \caption{Web browsing: Mean \ac{BLER}}
        \label{subfig:HighLoad_f}
    \end{subfigure}

    % Row 4
    \begin{subfigure}[t]{0.425\textwidth}
        \centering
        \includegraphics[width=\linewidth]{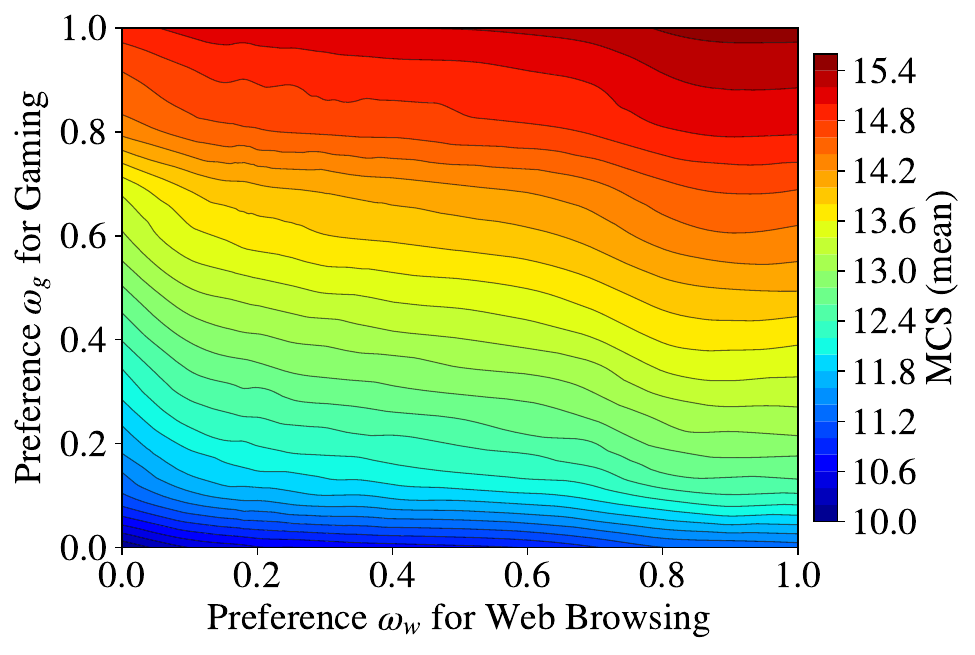}
        \caption{Real-time gaming: \ac{MCS} selection}
        \label{subfig:HighLoad_g}
    \end{subfigure}
    \hfill
    \begin{subfigure}[t]{0.425\textwidth}
        \centering
        \includegraphics[width=\linewidth]{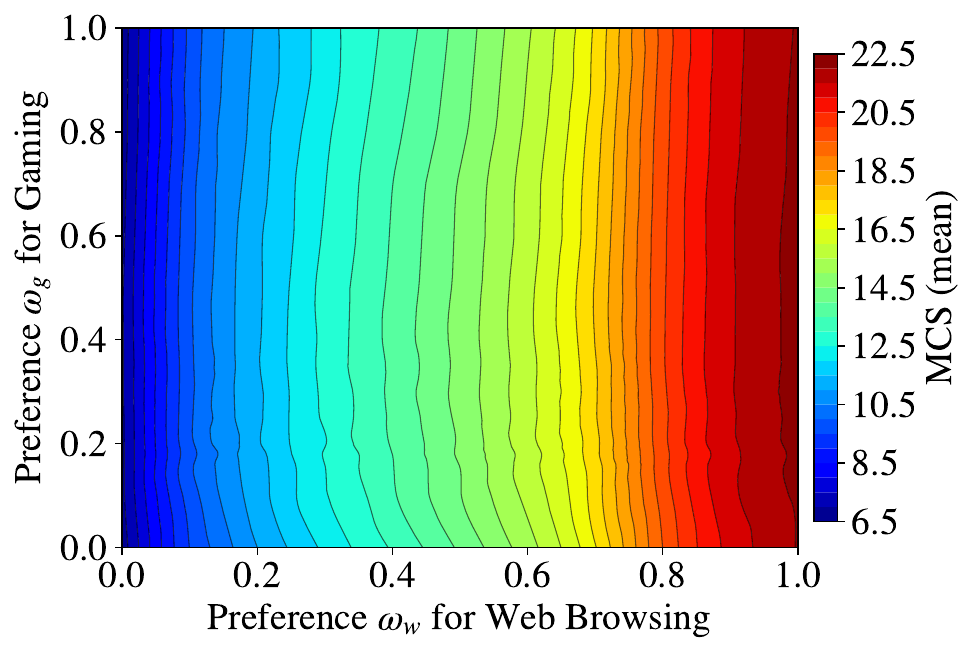}
        \caption{Web browsing: \ac{MCS} selection}
        \label{subfig:HighLoad_h}
    \end{subfigure}

    \caption{\textbf{Impact of user preference weights on the performance of real-time gaming and browsing users in high network load conditions.}
    Each subplot shows a distinct \ac{QoS} metric for real-time gaming users (left column) and for web browsing users (right column) under varying preference weights ($\omega_g$, $\omega_w$) reflecting resource allocation priorities for the two connectivity services. Metrics include: (a)-(b) mean user throughput, (c)-(d) mean latency (defined according to the service), (e)-(f) mean \ac{BLER}. Furthermore, (g)-(h) show the action (mean \ac{MCS}) distribution under ($\omega_g$, $\omega_w$).}
    \label{fig:HighLoad_KPIs}
\end{figure}

% ----------------------------------------------------------------------------
% Very high load
% ----------------------------------------------------------------------------
\begin{figure}[p]
    \centering
    
    % Row 1
    \begin{subfigure}[t]{0.425\textwidth}
        \centering
        \includegraphics[width=\linewidth]{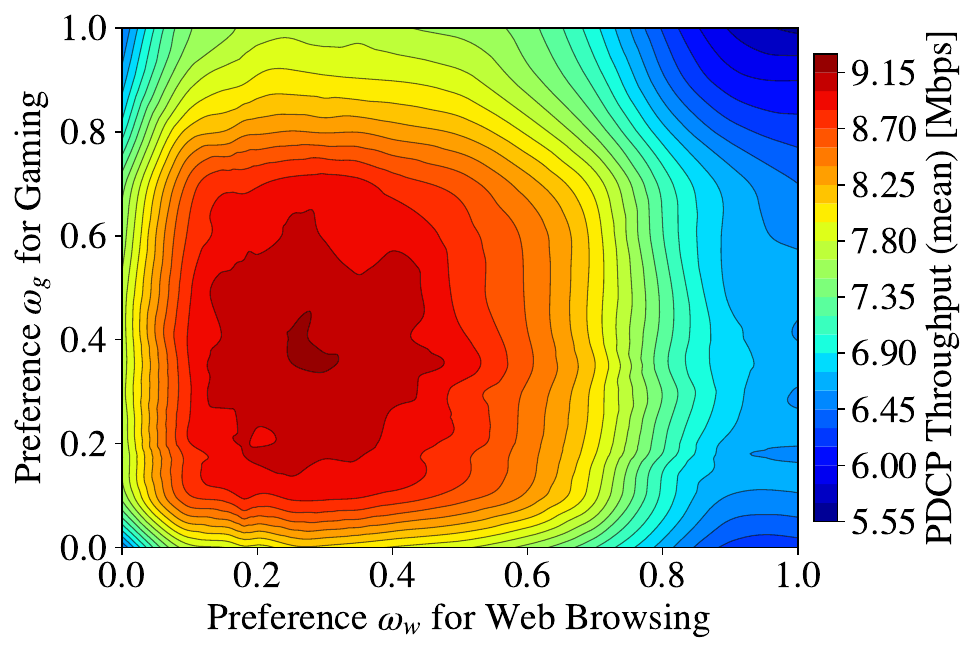}
        \caption{Real-time gaming: Mean PDCP Throughput}
        \label{subfig:VeryHighLoad_a}
    \end{subfigure}
    \hfill
    \begin{subfigure}[t]{0.425\textwidth}
        \centering
        \includegraphics[width=\linewidth]{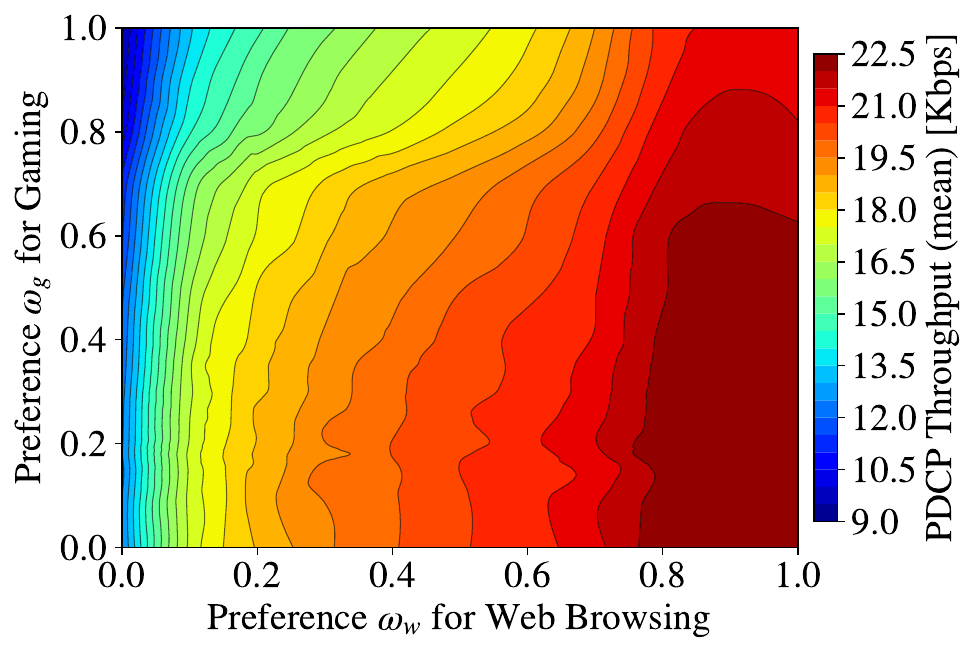}
        \caption{Web browsing: Mean PDCP Throughput}
        \label{subfig:VeryHighLoad_b}
    \end{subfigure}

    % Row 2
    \begin{subfigure}[t]{0.425\textwidth}
        \centering
        \includegraphics[width=\linewidth]{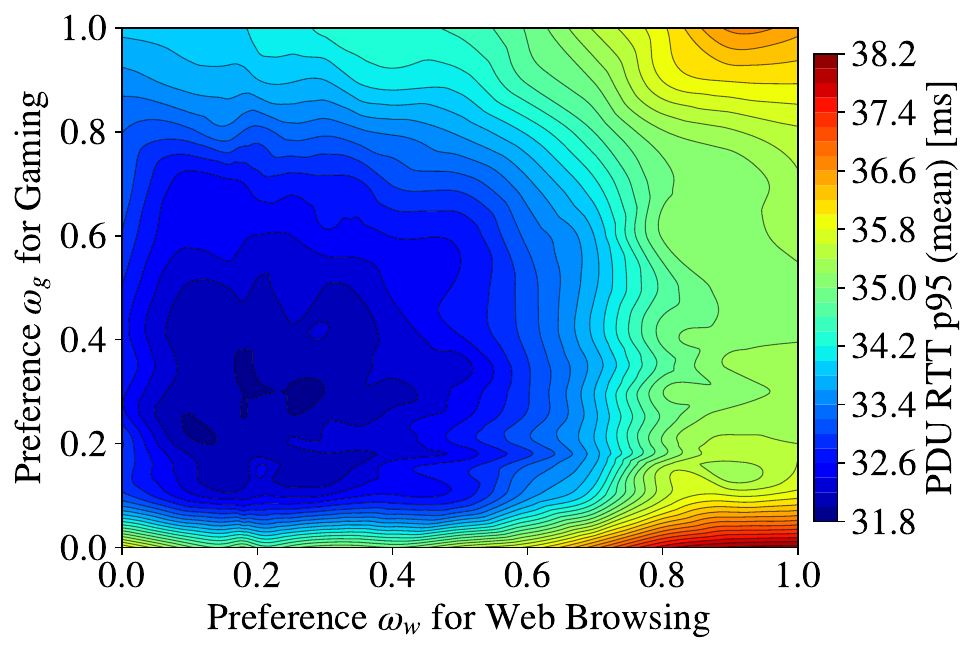}
        \caption{Real-time gaming: 95th Percentile PDU RTT}
        \label{subfig:VeryHighLoad_c}
    \end{subfigure}
    \hfill
    \begin{subfigure}[t]{0.425\textwidth}
        \centering
        \includegraphics[width=\linewidth]{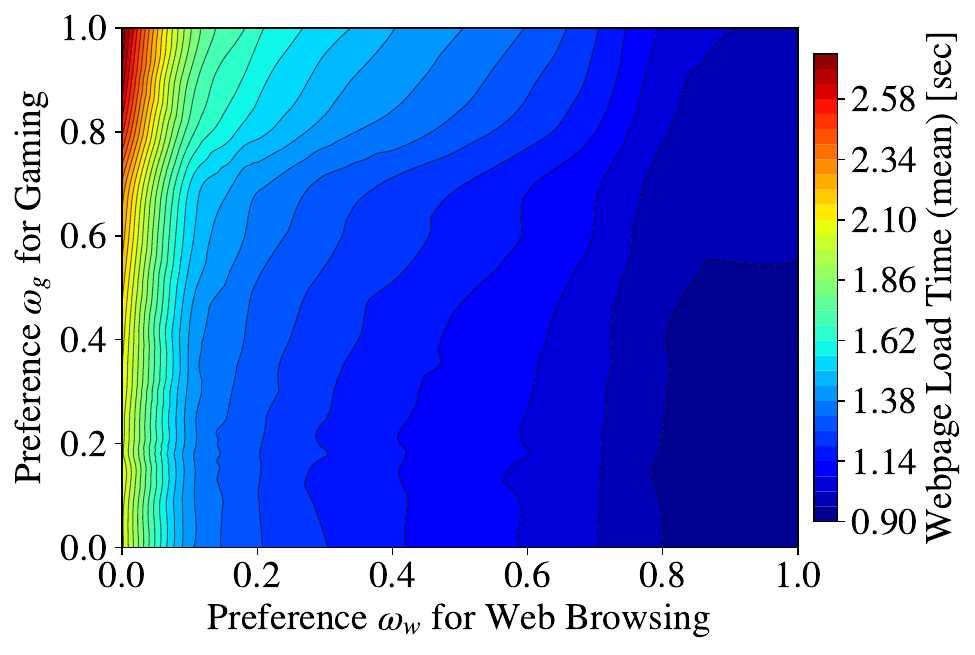}
        \caption{Web Browsing: Loading time}
        \label{subfig:VeryHighLoad_d}
    \end{subfigure}

    % Row 3
    \begin{subfigure}[t]{0.425\textwidth}
        \centering
        \includegraphics[width=\linewidth]{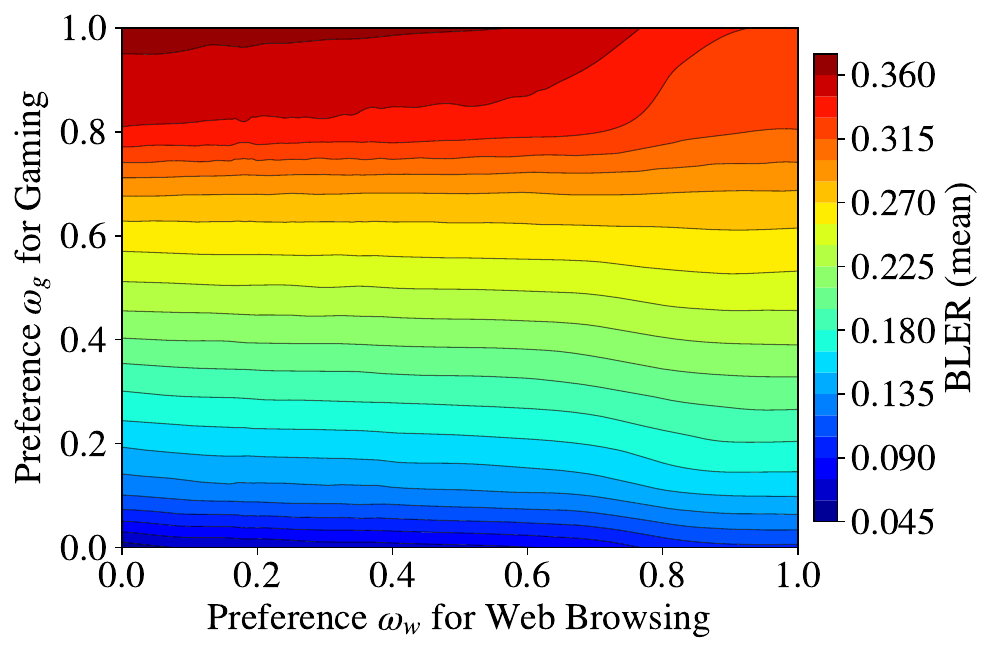}
        \caption{Real-time gaming: Mean \ac{BLER}.}
        \label{subfig:VeryHighLoad_e}
    \end{subfigure}
    \hfill
    \begin{subfigure}[t]{0.425\textwidth}
        \centering
        \includegraphics[width=\linewidth]{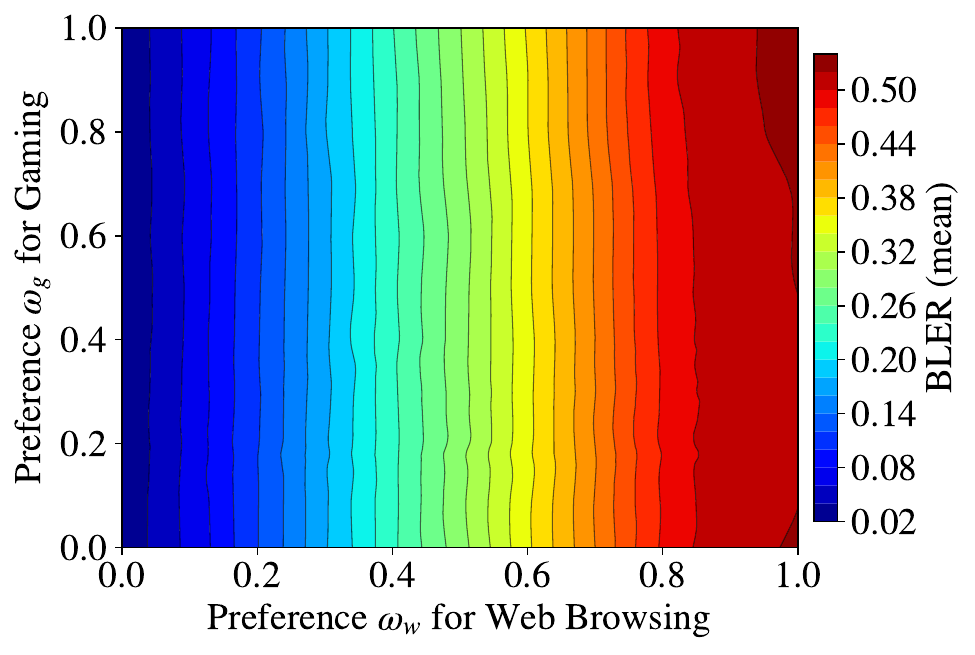}
        \caption{Web browsing: Mean \ac{BLER}}
        \label{subfig:VeryHighLoad_f}
    \end{subfigure}

    % Row 4
    \begin{subfigure}[t]{0.425\textwidth}
        \centering
        \includegraphics[width=\linewidth]{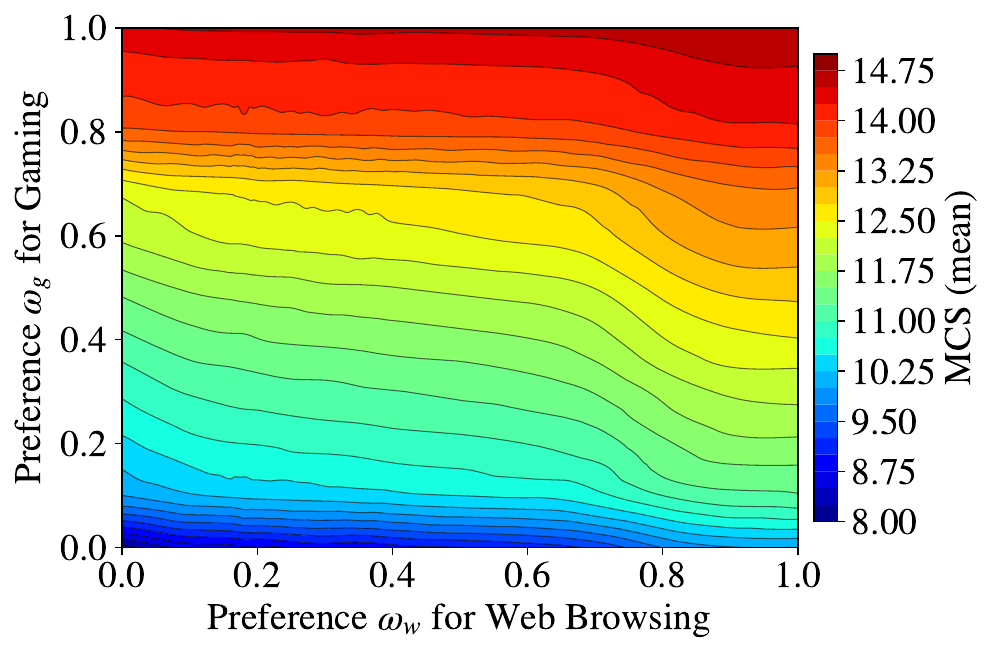}
        \caption{Real-time gaming: \ac{MCS} selection}
        \label{subfig:VeryHighLoad_g}
    \end{subfigure}
    \hfill
    \begin{subfigure}[t]{0.425\textwidth}
        \centering
        \includegraphics[width=\linewidth]{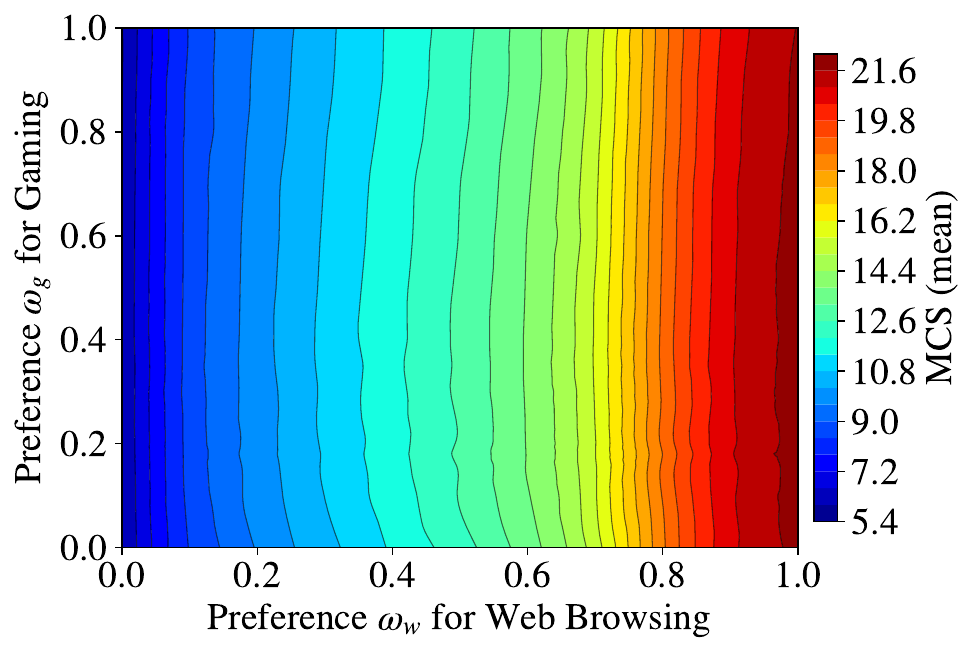}
        \caption{Web browsing: \ac{MCS} selection}
        \label{subfig:VeryHighLoad_h}
    \end{subfigure}

    \caption{\textbf{Impact of user preference weights on the performance of real-time gaming and browsing users in very high network load conditions.}
    Each subplot shows a distinct \ac{QoS} metric for real-time gaming users (left column) and for web browsing users (right column) under varying preference weights ($\omega_g$, $\omega_w$) reflecting resource allocation priorities for the two connectivity services. Metrics include: (a)-(b) mean user throughput, (c)-(d) mean latency (defined according to the service), (e)-(f) mean \ac{BLER}. Furthermore, (g)-(h) show the action (mean \ac{MCS}) distribution under ($\omega_g$, $\omega_w$).}
    \label{fig:VeryHighLoad_KPIs}
\end{figure}

\subsection{Online Preference Optimization} \label{appendix:F4_BO_preference_steering}

\Cref{appendix:Bayesian_optim} presented the Preference-Aligned eXploration Bayesian Optimization (PAX-BO) algorithm for the optimizer agent (cf.~\Cref{alg:paxbo}). A key design feature of PAX-BO is the integration of trust regions to stabilize the selection of the preference values $\omega$ for the downstream controller agent. 

\Cref{fig:BO_algorithms} compares the performance of PAX-BO considering an intent definition that requires to maximize the aggregate system throughput while keeping the \ac{BLER} of each user below $10\%$ -- which corresponds to the typical configuration of link adaptation in 4G/5G \ac{RAN} systems.~\Cref{fig:BO_wo_trust_region} shows the performance of PAX-BO without trust region, whereas in~\Cref{fig:BO_w_trust_region} we enabled the trust region. The top panel of both figures shows the aggregate system throughput $f(\omega)$, the middle panel shows \ac{BLER} constraint evaluations $g(\omega)$ with threshold $g(\omega) \leq 0.1$, and the bottom panel shows the evolution of the preference parameter $\omega$. Compared to~\Cref{fig:BO_wo_trust_region}, the trust region stabilizes the optimization, reducing constraint violations and leading to smoother preference adaptation and improved system throughput (with smoother degradations).

\begin{figure*}[t!]
    \centering
    \begin{subfigure}{0.49\textwidth}
        \centering
        \includegraphics[width=\linewidth]{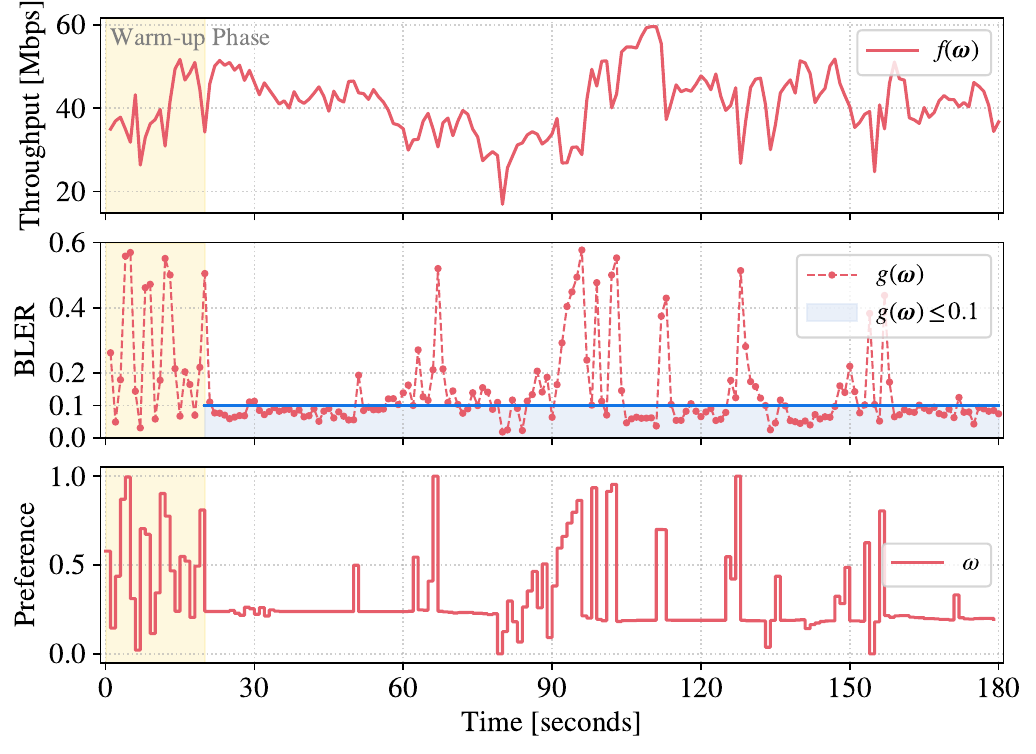}
        \caption{Performance of \ac{BO} without trust region}
        \label{fig:BO_wo_trust_region}
    \end{subfigure}
    \hfill
    \begin{subfigure}{0.49\textwidth}
        \centering
        \includegraphics[width=\linewidth]{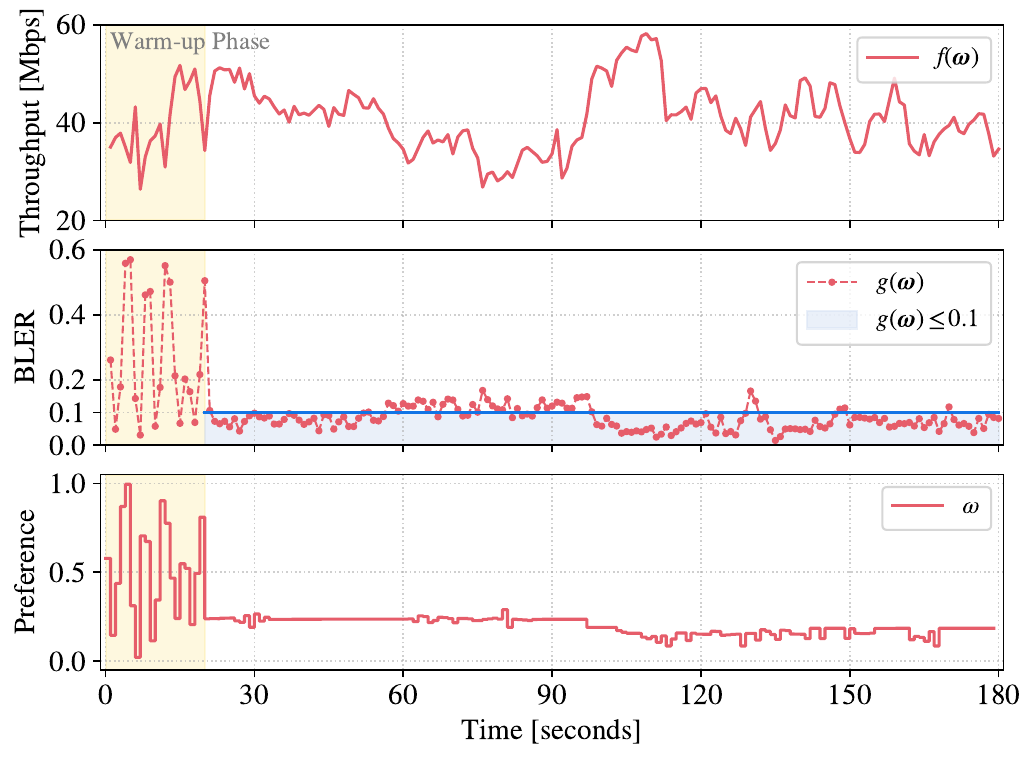}
        \caption{Performance of \ac{BO} with trust region}
        \label{fig:BO_w_trust_region}
    \end{subfigure}
    \caption{Performance comparison of optimizer agent when using the PAX-BO algorithm (a) without trust region and (b) with trust region enabled.}
    \label{fig:BO_algorithms}
\end{figure*}

% ===================================================================
% SECTION 17: Compute Resources and Hyperparameters
% ===================================================================
\clearpage
\section{Compute Resources and Hyperparameters} 
\label{appendix:hyperparams}

All MORL training runs were performed on a high-performance computing (HPC) cluster. The main training node was equipped with an NVIDIA A100-PCIE-40GB GPU and 48 CPU cores, which hosted the learner, actors, and replay memory. The replay buffer was partitioned into four independently prioritized shards, each pinned to a dedicated CPU core to support parallelized access. Co-locating the learner, actors, and replay shards on the same node minimized intra-node communication latency.

We used 40 actors for each experiment, and each actor launched two threads that interacted with 14 simulator instances in parallel. The simulators were distributed across multiple compute nodes, totaling 560 CPU cores. Each simulator ran in a separate process and communicated with its assigned actor via ZeroMQ, enabling scalable multi-node environment interaction.

Cluster job scheduling and resource management were handled by the Load Sharing Facility (LSF), which managed job queueing, monitoring, and node allocation according to the experiments’ resource specifications.

\begin{table}[htbp]
\centering
\caption{Hyperparameters Used for Adaptor in Interpreter.}
\label{tab:adaptor-hyperparams}
\begin{tabular}{ll}
\toprule
\multicolumn{2}{c}{\textbf{Monitor}} \\
\midrule
Window size (W)       & 12 \\
Alert-on ratio ($\rho_\text{on}$)        & 0.55 \\
Alert-off ratio ($\rho_\text{off}$)      & 0.45 \\
\midrule
\multicolumn{2}{c}{\textbf{Adjust}} \\
\midrule
Step (Mbps)       & 0.08 \\
Lifetime (Mbps)   & 0.40 \\
Floor             & 5.00 \\
Ceiling           & 9.00 \\
Cooldown steps    & 2 \\
Deadband          & 0.05 \\
Gain (up)         & 1.0 \\
Gain (down)       & 1.0 \\
\bottomrule
\end{tabular}
\end{table}

\begin{table}[htbp]
\centering
\caption{Hyperparameters Used for Supervised Fine-Tuning of the Intent-to-\ac{OTM} Translator.}
\label{tab:sft-hyperparams}
\renewcommand{\arraystretch}{1.2}
\begin{tabular}{ll}
\toprule
\textbf{Component} & \textbf{Setting} \\
\midrule
Base model & Qwen2.5-7B-Instruct \\
Parameter-efficient tuning & LoRA (rank 64, $\alpha$\,=\,16, dropout\,=\,0.05) \\
LoRA target modules & \texttt{q\_proj}, \texttt{k\_proj}, \texttt{v\_proj}, \texttt{o\_proj} \\
Precision & bfloat16 \\
Epochs & 2 \\
Batch size (per device) & 2 \\
Gradient accumulation steps & 8 \\
Optimizer & AdamW (Torch fused) \\
Learning rate & $2\times10^{-4}$ \\
Scheduler & Cosine decay \\
Warmup ratio & 0.03 \\
Weight decay & 0.01 \\
Gradient clipping & 1.0 \\
Gradient checkpointing & Enabled \\
Max sequence format & Qwen chat template (intent → OTM pair) \\
Evaluation frequency & Every 200 steps \\
Checkpoint frequency & Every 200 steps (max 5 checkpoints) \\
\bottomrule
\end{tabular}
\end{table}

\begin{table}[t]
\centering
\caption{Bayesian Optimization Hyperparameters Used in the Actor.}
\label{tab:bo_hyperparams}
\renewcommand{\arraystretch}{1.2}
\begin{tabular}{ll}
\toprule
\textbf{Hyperparameter} & \textbf{Value / Description} \\
\midrule
Acquisition function             & qLogEI \\
MC samples for acquisition       & 256 \\
Raw samples for optimization     & 512 \\
Number of restarts               & 10 \\
Batch size ($q$)                 & 1 \\
GP refit frequency               & Every 1 observation \\
Training window size             & 60 most recent samples \\
Input scaling (normalize)        & Yes (Normalize transform) \\
Output scaling                   & Standardize outcomes \\
\midrule
Trust region initial radius      & 0.15 \\
Minimum trust region radius      & 0.05 \\
Trust region shrink factor       & 0.7 \\
Infeasible patience              & 2 consecutive infeasible samples \\
No-improvement patience          & 5 evaluations \\
\midrule
Initial preference samples        & 20 Sobol samples (fixed list) \\
Preference domain                & $[0, 1]^2$ \\
\bottomrule
\end{tabular}
\end{table}

\begin{table}[htbp]
\centering
\caption{Hyperparameters used for Multi-Objective Reinforcement Learning (MORL).}
\label{tab:hyperparams-morl}
\begin{tabular}{l l}
\toprule
\multicolumn{2}{c}{\textbf{Learner}} \\
\midrule
Optimizer & Adam~\cite{KiB:17} \\
Learning rate & $5 \times 10^{-5}$ \\
$\beta_{1}$ (Adam momentum term) & 0.9 \\
$\beta_{2}$ (Adam second moment term) & 0.999 \\
$\epsilon$ (Adam numerical stability) & $1.5 \times 10^{-4}$  \\
Weight decay & $\nicefrac{0.02}{512}$ \\
Gradient norm & 20 \\
Target update period & Every 1 gradient updates \\
Target update policy & Soft \\
Target update factor & $1.0 \times 10^{-3}$ \\
Model update interval & Every 200 gradient updates \\
Prefetched batches & 16 \\
Batch size (experience) & 512 \\
Batch size (preference) & 128 \\
Warm-up phase & 50,000 samples \\
Loss function & MSE \\
\midrule
\multicolumn{2}{c}{\textbf{Actor}} \\
\midrule
Number of actors & 40 \\
Local buffer capacity & 2,500 \\
Discount factor ($\gamma$) & 1.0 \\
$\epsilon$-greedy (linear decay) & 0.8 $\rightarrow$ 0.05 \\
Timesteps & 5,500,000 \\
\midrule
\multicolumn{2}{c}{\textbf{Replay Memory}} \\
\midrule
Number of shards & 4 \\
Capacity of each shard & 4,000,000 \\
Prioritization exponent ($\alpha$) & 0.6 \\
Importance sampling exponent ($\beta$) & 0.4 \\
\midrule
\multicolumn{2}{c}{\textbf{Model}} \\
\midrule
Activation function & ReLU \\
Number of blocks & 6 \\
Number of layers per block & 2 \\
units per layer & 128 \\
Dropout probability & 0.1 \\
Layer normalization & True \\
\bottomrule
\end{tabular}
\end{table}

\end{document}